\documentclass{article}
\usepackage{jld_preprint,times}
\usepackage{amsmath,amsfonts,bm}

\def\eqref#1{equation~\ref{#1}}
\def\1{\bm{1}}

\newcommand{\test}{\mathcal{D_{\mathrm{test}}}}

\DeclareMathAlphabet{\mathsfit}{\encodingdefault}{\sfdefault}{m}{sl}
\SetMathAlphabet{\mathsfit}{bold}{\encodingdefault}{\sfdefault}{bx}{n}

\usepackage{graphicx}
\usepackage{tikz}
\usepackage{float}
\usepackage{booktabs}
\usepackage{amsmath,amssymb,amsthm}
\usepackage{xcolor}
\usepackage{colortbl}
\usepackage{subcaption}
\usepackage{multirow}
\usepackage{wrapfig}
\usepackage{needspace}
\usepackage{placeins}
\usepackage{algorithm}
\usepackage{algpseudocode}
\usepackage{url}
\definecolor{paperlink}{HTML}{285A78}
\definecolor{jldcode}{HTML}{7A8388}
\definecolor{citeblue}{HTML}{1F3FBF}% citations: standard blue
\definecolor{refred}{HTML}{C0141C}% figures, tables, sections, equations, appendix
\definecolor{urlorange}{HTML}{E06A00}% external links
\usepackage[colorlinks=true,linkcolor=refred,citecolor=citeblue,urlcolor=urlorange]{hyperref}
\graphicspath{{figures/}}
\providecommand{\name}{}\renewcommand{\name}{JLD}
\providecommand{\namefast}{}\renewcommand{\namefast}{JLD-fast}   % block 1 only, 2.5 ms
\providecommand{\namelin}{}\renewcommand{\namelin}{JLD-lin}      % patch embedding only
\providecommand{\namevideo}{}\renewcommand{\namevideo}{JLD-video}
\providecommand{\M}{}\renewcommand{\M}{\mathbf{M}}            % sensitivity matrix
\providecommand{\U}{}\renewcommand{\U}{\mathbf{U}}            % Jacobian lens (top-k eigenvectors)
\providecommand{\tok}{}\renewcommand{\tok}[1]{\mathbf{h}_{#1}} % patch token t
\providecommand{\out}{}\renewcommand{\out}{\mathbf{z}}        % output embedding
\providecommand{\dist}{}\renewcommand{\dist}{D}               % the distance
\providecommand{\pre}{}\renewcommand{\pre}{R}                 % chroma pre-filter
\providecommand{\srcc}{}\renewcommand{\srcc}{SRCC}
\providecommand{\plcc}{}\renewcommand{\plcc}{PLCC}
\providecommand{\krcc}{}\renewcommand{\krcc}{KRCC}

\theoremstyle{plain}
\newtheorem{theorem}{Theorem}
\newtheorem{proposition}{Proposition}
\newtheorem{lemma}{Lemma}
\theoremstyle{definition}

\usepackage{listings}
\usepackage{xcolor}

\definecolor{codeaccent}{HTML}{465A9C}

\lstdefinestyle{jldcode}{
    language=Python,
    basicstyle=\ttfamily\footnotesize\color{black},
    keywordstyle=\color{codeaccent}\bfseries,
    commentstyle=\color{black!55},
    stringstyle=\color{black},
    backgroundcolor=\color{black!3},
    frame=single,
    rulecolor=\color{black!15},
    framerule=0.4pt,
    xleftmargin=0.4em,
    xrightmargin=0.4em,
    framexleftmargin=0.4em,
    framexrightmargin=0.4em,
    breaklines=true,
    breakatwhitespace=false,
    showstringspaces=false,
    columns=fullflexible,
    keepspaces=true,
    aboveskip=5pt,
    belowskip=5pt
}

\newenvironment{tightlist}{\begin{itemize}\setlength{\itemsep}{0pt}\setlength{\parskip}{0pt}\setlength{\topsep}{2pt}\setlength{\leftmargin}{1.2em}}{\end{itemize}}

\renewcommand{\topfraction}{0.85}
\renewcommand{\bottomfraction}{0.7}
\renewcommand{\textfraction}{0.15}
\renewcommand{\floatpagefraction}{0.8}
\addtocontents{toc}{\protect\setcounter{tocdepth}{-10}}

\title{\name{}: Perceptual Distance through a Jacobian Lens}

\author{
Shreshth Saini$^{1,2}$ \quad
Balu Adsumilli$^{2}$ \quad
Alan C. Bovik$^{1,3}$ \\
$^{1}$The University of Texas at Austin, Austin, TX, USA \\
$^{2}$Google, USA \\
$^{3}$University of Colorado Boulder, Boulder, CO, USA \\
\texttt{saini.2@utexas.edu}
}

\iclrfinalcopy

\begin{document}
\maketitle
\begin{abstract}
Image compression, restoration, and generation all require a way to measure how different two images look to a person. Pixel error ignores how people see, while the most accurate perceptual distances are typically fitted to human judgments, tying them to a fixed data and resolution. For example, when image resolution is doubled, the correlation of DISTS with human scores on TID2013 drops from 0.815 to 0.717.
We introduce the Jacobian Lens Distance (\name{}), which derives its perceptual geometry from a frozen vision encoder rather than from human labels. \name{} combines the locality of early patch features with the perceptual sensitivity captured by later encoder representations. Specifically, we use the encoder Jacobian to identify directions in the early feature space that most strongly affect the encoder output, producing a fixed metric tensor, $\mathbb{E}[J^\top J]$, which we call the Jacobian lens. The lens is fitted only once from 100 unlabeled images, taking about 35 seconds.
Locally, this construction defines a pullback metric in pixel space, giving \name{} a clear geometric interpretation that can be directly analyzed on real images. Across four standard perceptual databases, \name{} achieves state-of-the-art performance and consistently outperforms LPIPS, DISTS, PieAPP, and DreamSim. \name{} is also robust to changes in image resolution, on TID2013, its lens-term correlation remains nearly unchanged when the resolution is doubled, decreasing only from 0.850 to 0.845. We further introduce \namefast{}, which is $4\times$ faster than LPIPS-VGG while achieving a mean correlation of 0.911. Finally, \name{} naturally extends to video, reaching a correlation of 0.786 on Waterloo IVC 4K compared with 0.611 for VMAF. \href{https://shreshthsaini.github.io/jld/}{Code}.
\end{abstract}

% Overview + motivation figure. v23.5 (worker M): the method pipeline of the old Fig 1 (fig:hero, now commented out
% in main.tex) is merged in as the top row; figures/v235/M_pipeline.pdf (scripts/v235/M_overview_pipeline.py).
% Bottom row unchanged from v23.4 (worker A2). Sources: figures/v234/A2_motivation_numbers.json
% (selection scripts/v234/A2_motivation_select.py over the released per-pair scores; maps scripts/v234/A2_motivation_maps.py;
% plot scripts/v234/A2_motivation_fig.py). Matched-PSNR agreement: figures/v233/B_numbers.json (pairs.matched_psnr_agreement.tid2013).
\begin{figure}[H]
\centering
\includegraphics[width=0.86\linewidth]{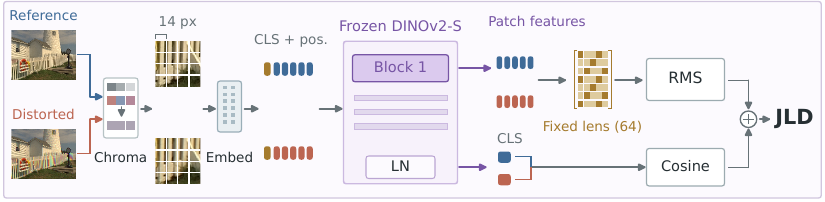}\\[2pt]% arXiv layout: was 0.92
\includegraphics[width=0.86\linewidth]{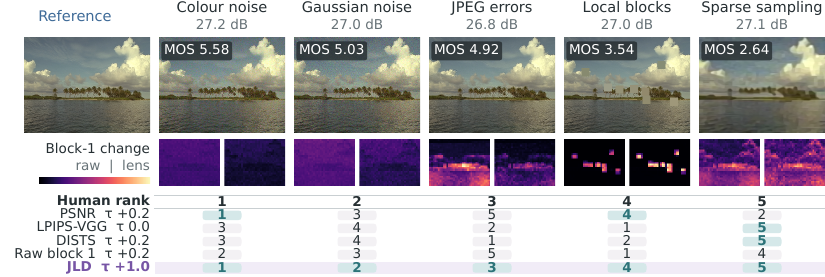}% arXiv layout: was 0.92
\vspace{-4pt}
\caption{Top: block-1 patch features of the images are projected onto a fixed lens (\name{}), fitted from the output's sensitivity without labels; Bottom: five TID2013 distortions at equal PSNR (26.8 to 27.2\,dB), ordered by human score, with the per-patch change, and rank and Kendall $\tau$ against the human ranks.}
\label{fig:motivation}\label{fig:hero}% fig:hero kept so old references to the Fig 1 pipeline resolve
\vspace{-6pt}
\end{figure}

\section{Introduction}
\label{sec:introduction}
Measuring how different two images look to a person is a basic problem in imaging.
Lossy codecs are tuned using perceptual measures, while restoration and generative models are trained and evaluated with them \citep{wang2004ssim,zhang2018lpips}. Pixel error, measured using mean squared error or PSNR, is simple and differentiable, but it treats every pixel change equally, whereas human vision does not \citep{wang2004ssim,ponomarenko2015tid2013}. Figure~\ref{fig:motivation} shows five distortions from TID2013 \citep{ponomarenko2015tid2013} whose PSNR values differ by less than 0.4\,dB. Despite their similar pixel error, their perceptual severity is very different. Across all TID2013 pairs with equal PSNR, it selects the image preferred by human observers in only 51.1\% of cases, whereas \name{} does so in 79.9\%. A useful perceptual distance should therefore assign small distances to changes that people barely notice and large distances to changes they find visually important. At the same time, if it is to be used as a loss or an evaluation measure, its geometry should remain well defined and interpretable.

Existing approaches largely fall into two families. 
Classical full-reference measures explicitly model properties associated with early vision.
SSIM and MS-SSIM compare local luminance, contrast, and structure \citep{wang2004ssim,wang2003msssim};
FSIM, GMSD, and VSI compare phase-congruency, gradient, and saliency maps \citep{zhang2011fsim,xue2014gmsd,zhang2014vsi}.
Deep feature distances instead compare representations from pretrained networks. LPIPS, DISTS, PieAPP, and DreamSim, etc., are additionally fitted to human judgements \citep{zhang2018lpips,ding2022dists,prashnani2018pieapp,fu2023dreamsim}. 
This fitting makes such metrics poor at generalizability. When TID2013 is rescored at twice the image resolution, from 256 to 512 pixels, the Spearman correlation (SRCC) of DISTS with ground truth drops from 0.815 to 0.717, while LPIPS-VGG \citep{zhang2018lpips,simonyan2015vgg} drops from 0.757 to 0.661. 
A separate line of work derives perceptual geometry from models of unlabelled data.
IEM~\citep{ohayon2026iem} integrates denoising errors under a learned density, while LASI~\citep{severo2024lasi} fits linear embeddings at inference time. These methods require no human labels and provide a principled geometry, but they are considerably more expensive to evaluate. For example, IEM requires 725.7\,ms per image pair at $256\times256$ pixels, compared with 12.2\,ms for \name{} at $512\times384$ (Table~\ref{tab:main}).

In this work, we ask a simple question: \emph{does a network that has only learned to see already contain a useful perceptual distance?} We find that it does, but the relevant information does not appear at any single network depth (Appendix Figure~\ref{fig:A2-depth}).
In a frozen DINOv2-S encoder \citep{oquab2024dinov2}, early patch tokens preserve local image structure.
However, treating all 384 feature directions equally gives too much weight to directions that have little effect on the network output, reaching a SRCC of only 0.784 on all full TID2013 dataset.
Late representations capture information that is more relevant to the encoder output, but much of the spatial detail has already been pooled away; the output tokens reach a correlation of only 0.742. The key observation is that these two properties can be combined. We retain the spatial detail of early features, but weight their directions according to how strongly they affect the network output.
This raises the correlation to 0.850 and, unlike the unweighted alternatives, remains stable as image resolution increases. The same effect is visible in Figure~\ref{fig:motivation}, raw features rank the five distorted images almost at random relative to human observers ($\tau=0.2$), whereas the weighted features reproduce the human ordering exactly. %We summarise this principle as measuring early and weighing by late. 

The weights come from the Jacobian of the output with respect to an early patch token. This Jacobian describes how a small change in the token affects the encoder output.
Averaging its squared response over natural images gives a single fixed metric tensor in feature space, $\mathbb{E}\!\left[J^{\top}J\right]$. For block 1 of DINOv2-S, this metric is strongly anisotropic, only 64 of its 384 eigendirections account for 76.5\% of its trace. Based on this observation, we introduce the Jacobian Lens Distance (\name{}).
\name{} retains the 64 leading eigendirections of $\M$ as a fixed projection, which we call the Jacobian lens.
%It then measures the root-mean-square displacement between the lensed block-1 patch tokens of two images, together with a small output-level term (Figure~\ref{fig:motivation}). The lens is estimated once, in 35 seconds, from 100 unlabelled images using random output probes, and 
It never uses human judgements (Appendix Figure~\ref{fig:A-lens}). Because the lens is a fixed linear map, the geometry induced by \name{} can be computed and visualised directly on real images, independently of any benchmark.
To our knowledge, \name{} is the first perceptual distance to transfer a network's output sensitivity back to its early features through a single fixed metric. Contributions:
\begin{tightlist}\interlinepenalty=10000
\item We introduce a novel perceptual distance derived entirely from the output sensitivity of a frozen vision encoder, without human labels. Its lens term defines a pseudometric whose local form is a pullback metric in pixel space (Section~\ref{sec:method}).

\item We analyse the induced geometry directly on real images. The iso-distance contours of \name{} agree with the ellipses predicted by the encoder, its sensitivity decreases with local contrast (Section~\ref{sec:geometry}).

\item We evaluate \name{} against 15 full-reference distances. \name{} achieves the SOTA performance, and beat every learned deep distance, and remains stable as image resolution doubles. \namefast{} is four times faster than LPIPS-VGG, and a single fitted lens transfers across 14 encoders and extends to video (Section~\ref{sec:analysis}).
\end{tightlist}

\section{Background}\label{sec:related}\label{sec:background}
A perceptual distance should assign small distances to images that look the same.
A feature distance $\|\Phi(x)-\Phi(y)\|_{\mathbf Q}$, for a feature map $\Phi$ and positive semidefinite matrix $\mathbf Q$, satisfies the metric axioms by construction, while its local geometry around an image is determined by the pullback $A^{\top}\mathbf QA$, where $A=\partial\Phi/\partial x$.
Existing feature distances differ mainly in how $\mathbf Q$ is chosen, raw feature distance, DeepWSD and DeepDC weight all channels equally \citep{liao2022deepwsd,zhu2025deepdc}, LPIPS learns a diagonal $\mathbf Q$ from human judgements \citep{zhang2018lpips}, DISTS fits weights on feature statistics \citep{ding2022dists}, and DreamSim fine-tunes the encoder on human triplets \citep{fu2023dreamsim}. Classical metrics derive the tensor from models of early vision, either through the perceptual metric of a divisive-normalisation transform \citep{epifanio2003metric,malo2006nonlinear,laparra2010divisive,laparra2016nlpd} or through the Fisher information of a response model, with eigen-distortions predicting the most and least visible directions \citep{berardino2017eigen,feather2025principal,vacher2024fisher}; both produce an image-dependent matrix in pixel space.
IEM integrates differences between denoising errors under a learned density \citep{ohayon2026iem}, while LASI fits linear embeddings at inference time \citep{severo2024lasi}.

Pullback metrics through neural networks are established in other settings, active subspaces based on $\mathbb E[\nabla f\nabla f^{\top}]$ \citep{constantine2014active}, decoder-induced metrics on latent spaces \citep{arvanitidis2018oddity}, the pullback Fisher metric for steering language models \citep{wang2026fishback}, and the Jacobian lens of \citet{gurnee2026workspace}, which decodes averaged intermediate-to-final Jacobians through the vocabulary. \name{} differs from these approaches by deriving the tensor from a frozen encoder's own output, rather than from labels, a hand-designed vision model, or a learned density, and by fixing that tensor on early features so that it is shared across images; to our knowledge, no perceptual distance has been constructed this way.

\section{The Jacobian lens distance}\label{sec:method}
% arXiv layout: fig:geo-contours is input in sections/geometry.tex, next to its first reference.
% \input{sections/figs/fig_geo_contours}

\subsection{Encoder and features}\label{sec:setup}\label{sec:model}
A perceptual distance needs two properties that a single layer of a vision encoder does not provide together. Early patch features are local and respond to every change in the image, but they weight all changes alike. Late representations separate the changes that matter from those that do not, but they have lost the spatial detail on which low-level distortions act. \name{} takes the features from an early block and the weighting from the encoder output, and the Jacobian is the object that connects the two.

\name{} compares two aligned images using a single frozen encoder, DINOv2-S (ViT-S/14), pretrained without human judgements \citep{oquab2024dinov2}.
An RGB image $x\in[0,1]^{3\times H\times W}$ with sides divisible by $p=14$ produces $N=HW/p^{2}$ patch positions.
At position $t$, we take the patch token $\tok{t}(x)\in\mathbb{R}^{d}$, with $d=384$, after block 1 of 12.
The final layer-normalised class token $\out(x)\in\mathbb{R}^{m}$ is the output whose sensitivity defines the metric.
Images are center cropped to a multiple of 14 pixels, so the same lens applies at every image size. A lens can be fitted after any block; on the TID2013 pairs, it improves agreement with human scores at every depth, with the largest gain at block 1: 0.850, compared with 0.803 on the patch embedding, 0.816 after block 2, and at most 0.780 from block 3 onward (Appendix Figure~\ref{fig:A2-depth}). After one attention block, each patch has interacted once with its neighbours, so the tokens retain local structure while already incorporating context. Block 1 is therefore the natural place to measure distortions, and the lens supplies the weighting that these early tokens lack.

\subsection{Sensitivity as a metric tensor}\label{sec:lens}
Raw feature distance weights all $d$ directions equally, even though many have little effect on what the encoder computes. We replace this uniform weighting with the encoder's own sensitivity. The Jacobian $J_t(x)=\partial\out(x)/\partial\tok{t}(x)\in\mathbb{R}^{m\times d}$ maps a small change $\delta$ at patch $t$ to the first-order output change $J_t(x)\delta$. Averaging its squared response over images and positions gives a single fixed matrix,
\begin{equation}
\M=\mathbb E_x\Big[\frac1N\sum_{t=1}^N J_t(x)^{\top}J_t(x)\Big]=\sum_{i=1}^d\mu_iu_iu_i^{\top},
\qquad \U_k=[u_1\,\cdots\,u_k],
\label{eq:sensitivity}
\end{equation}
\label{def:sensitivity}%
where $\mu_1\geq\dots\geq\mu_d\geq0$. The matrix $\M$ is positive semidefinite, and $\delta^{\top}\M\delta$ gives the expected squared output change caused by moving one patch by $\delta$.
Its eigenvectors order the feature directions by how strongly the encoder output responds to them: a displacement along $u_1$ changes the output the most, and a displacement along $u_d$ leaves it almost unchanged.
We call the orthonormal matrix $\U_k$, with $k=64$, the Jacobian lens.
Among all $k$-dimensional subspaces, it retains the most sensitivity, $\operatorname{tr}(\U_k^{\top}\M\U_k)$, while any displacement orthogonal to it changes the output by at most $\sqrt{\mu_{k+1}}\,\|\delta\|$ in RMS (Proposition~\ref{prop:optimality}).
The 64 retained directions account for 76.5\% of the trace (Appendix~\ref{app:method-details}). The lens is thus the optimal rank-$k$ summary of what the encoder attends to, and it is a property of the encoder alone: one matrix, shared by every image, position and resolution.

% Worker A2 (v23.4): compact main-text algorithm; the appendix copy (alg:jld) is commented out (v23.5).
\begin{wrapfigure}[16]{r}{0.47\linewidth}
\vspace{-12pt}
\begin{minipage}{\linewidth}
\hrule height .8pt\vspace{2pt}
\captionof{algorithm}{\name{}}\label{alg:A2-jld}
\vspace{-4pt}\hrule\vspace{1pt}
\small
\begin{algorithmic}[1]
\Statex \textbf{Fit} (unlabelled crops $\mathcal X$, $P$ probes, rank $k$)
\State $\hat\M\gets 0$
\For{$x\in\mathcal X$, $v\sim\mathcal N(0,I_m)$, $P$ times}
\State $g_t\gets\nabla_{\tok{t}}\,v^{\top}\out(x)$ for all $t$
\State $\hat\M\gets\hat\M+\frac{1}{N|\mathcal X|P}\sum_t g_tg_t^{\top}$
\EndFor
\State $\U_k\gets$ top-$k$ eigenvectors of $\hat\M$
\Statex \textbf{Score} (images $x$, $y$)
\State $\tilde x,\tilde y\gets\pre x,\pre y$ \Comment{Eq.~\ref{eq:A2-prefilter}}
\State $\Delta_t\gets\U_k^{\top}\big(\tok{t}(\tilde x)-\tok{t}(\tilde y)\big)$ for all $t$
\State $c\gets 1-\cos\big(\out(\tilde x),\out(\tilde y)\big)$
\State \Return $\big(\frac1N\sum_t\|\Delta_t\|^2\big)^{1/2}+\frac12 c$
\end{algorithmic}
\vspace{1pt}\hrule height .8pt
\end{minipage}
\vspace{-8pt}
\end{wrapfigure}

Fitting the lens requires no explicit Jacobian matrices (Appendix Figure~\ref{fig:A-lens}; Algorithm~\ref{alg:A2-jld}).
A Gaussian probe $v\sim\mathcal N(0,I_m)$ on the output, backpropagated once, gives $g_t=\nabla_{\tok{t}}(v^{\top}\out)=J_t^{\top}v$ at every patch, and $\mathbb E_v[g_tg_t^{\top}]=J_t^{\top}\mathbb E_v[vv^{\top}]J_t=J_t^{\top}J_t$; therefore, averaging $N^{-1}\sum_t g_tg_t^{\top}$ over crops and probes gives an unbiased estimate of $\M$ (Proposition~\ref{prop:estimation}).
Using four $224\times224$ crops from each of 100 unlabelled DIV2K images \citep{agustsson2017div2k} and eight probes per crop, the fit takes 35\,s on one A100 GPU and uses no image pairs, distortions, or human labels (Appendix~\ref{app:fit}). The fit is stable: independent refits with fresh crops and probes recover the same subspace, with overlap 0.962 (Appendix~\ref{app:md-stability}).
Human data are used only to select among 245 configurations on CSIQ and KADID-10k development splits (Appendix~\ref{app:sweep}).
The resulting lens displacement also tracks the encoder output far better than raw block-1 features (Appendix Figure~\ref{fig:A-lens}).

\subsection{The distance}\label{sec:distance}\label{sec:prefilter}\label{sec:csf}\label{sec:properties}\label{sec:card}
Human contrast sensitivity falls faster with spatial frequency for colour than for luminance \citep{campbell1968csf,mullen1985}. Following S-CIELAB \citep{zhang1996scielab}, we use a fixed pre-filter that blurs chroma while preserving luminance, 
\begin{equation}
\pre x=T^{-1}\big(Y,\ G_\sigma*C_b,\ G_\sigma*C_r\big),\qquad (Y,C_b,C_r)=T(x),\qquad \sigma=2\text{ pixels},
\label{eq:A2-prefilter}
\end{equation}
where $T$ maps RGB to YCbCr and $G_\sigma$ is a Gaussian. Both filtered images are passed through the encoder on the same patch grid, and
\begin{equation}
\dist(x,y)=
\underbrace{\Big[\frac1N\sum_{t=1}^N
\big\|\U_k^{\top}\big(\tok{t}(\pre x)-\tok{t}(\pre y)\big)\big\|^2\Big]^{1/2}}_{\text{lens term } \dist_{\rm lens}(x,y)}
+\underbrace{\tfrac12\big[1-\cos\big(\out(\pre x),\out(\pre y)\big)\big]}_{\text{output term}}.
\label{eq:lens-distance}
\end{equation}
\label{eq:final-jld}\label{def:lens-distance}\label{sec:output}%
The two terms are complementary. The lens term measures where and how strongly the image changed, patch by patch, along the directions the encoder is sensitive to, and its per-patch values produce the lens maps in Figure~\ref{fig:motivation}. The output term captures whole-image changes, such as global colour shifts, that are spread too thinly over patches to dominate the lens term. Development data fixed its weight at $1/2$ (Appendix~\ref{app:output-weight}).
The lens term is euclidean distance after a fixed feature map, $\dist_{\rm lens}(x,y)=\|\Phi(x)-\Phi(y)\|_2$, where $\Phi$ stacks the lensed tokens $N^{-1/2}\U_k^{\top}\tok{t}(\pre x)$. We call it pseudometric: it is non-negative, symmetric, zero on identical images, and obeys the triangle inequality (Proposition~\ref{prop:metric}).
The output term need not satisfy the triangle inequality, and we empirically verify that violations by the full distance are rare (Section~\ref{sec:geometry}).
\namefast{} keeps the pre-filter and lens term but stops the encoder after block 1, reducing the cost per pair from 118.5 to 10.7 GFLOPs; \namelin{} applies the lens to the linear patch embedding (Appendix~\ref{app:visualisation}).
Scoring requires two forward passes and no learned parameters beyond the frozen encoder and the fixed lens, and every operation is differentiable, so \name{} can also be used as a loss function. 

Near an image, the lens term becomes a quadratic form in pixel space: with $A_t(x)=\partial\tok{t}(\pre x)/\partial x$,
\begin{equation}
\dist_{\rm lens}(x,x+\varepsilon)^2=\varepsilon^{\top}G(x)\,\varepsilon+o(\|\varepsilon\|^2),
\qquad
G(x)=\frac1N\sum_{t=1}^N A_t(x)^{\top}\,\U_k\U_k^{\top}A_t(x)
\label{eq:A2-pullback}
\end{equation}
(Proposition~\ref{prop:local}). $G(x)$ is the pullback of the fixed lens through the encoder and therefore varies with the image. The output term contributes only $O(\|\varepsilon\|^2)$, so Equation~\ref{eq:A2-pullback} also gives the leading local behaviour of $\dist$.

\IfFileExists{sections/geometry.tex}{\section{Geometry of the distance}\label{sec:geometry}
% arXiv layout: contours first (Figure 2), slices before Section 4.2 (Figure 3).
% Worker A2 (v23.4). Numbers: figures/v234/A2_geo_contours_numbers.json; script scripts/v234/A2_geo_contours.py
% (cached grids and rings of scripts/v23/geo_contours.py and scripts/v233/D1_geo_contours2.py).
\begin{figure}[t]
\centering
\includegraphics[width=\linewidth]{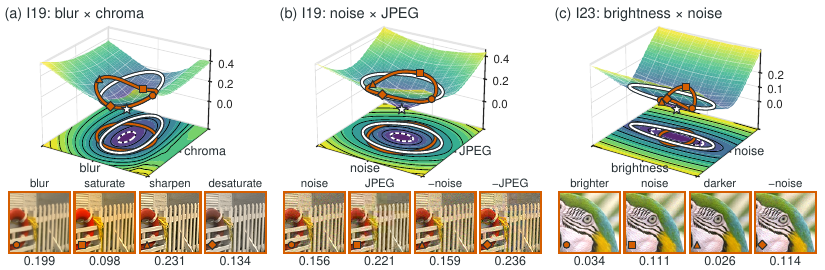}
\vspace{-16pt}
\caption{\name{} on three planes spanned by two distortion directions of equal pixel energy around a reference (star), with the ellipse predicted by the pullback $G$ (dashed). White ring: constant \name{}; orange ring: constant MSE (25.1, 24.0 and 29.7\,dB), whose four marked images are shown below with their \name{}.}
\label{fig:geo-contours}
% \vspace{-6pt}% arXiv layout: dropped, body text follows the caption directly
\end{figure}

In this section, we examine the geometry that the \name{} induces on image space. Because \name{} is an explicit function of a frozen encoder and a fixed projection, we can study its geometry directly on real images, before comparing it with human judgements.
\label{sec:geo-axioms}
As guaranteed by Section~\ref{sec:distance}, the lens term satisfies the triangle inequality on all 171.4 million same-reference triplets from four databases; the full distance violates it on at most 107 triplets per million, with a maximum violation of 3.6\%.
\label{sec:geo-spectrum}%
A single lens also generalises across images, a lens fitted to one image overlaps the shipped lens by a median 0.72, comparable to the 0.69 overlap between two fits of the same image using different methods (Appendix~\ref{app:geo-local-lens}).

\subsection{The local shape of the distance}\label{sec:geo-local}
Near an image $x$, Equation~\ref{eq:A2-pullback} gives $\dist(x,x+\varepsilon)\approx(\varepsilon^{\top}G(x)\,\varepsilon)^{1/2}$.
With the eigen decomposition $G(x)=\sum_j\lambda_j e_je_j^{\top}$, $\lambda_1\geq\lambda_2\geq\dots\geq0$, perturbations at distance $r$ form, to first order, $\{\varepsilon:\varepsilon^{\top}G(x)\varepsilon=r^2\}$, whose half-axis along $e_j$ has length $r/\sqrt{\lambda_j}$.
Directions with large $\lambda_j$ are fixed, small pixel changes along them are registered strongly. Directions with small $\lambda_j$ are soft and are tolerated more. For MSE, $G\propto I$, so the balls are spheres and all changes with equal energy are treated equally. Writing $\varepsilon^{\top}G(x)\varepsilon=\frac1N\sum_t\|\U_k^{\top}A_t(x)\varepsilon\|^2$ makes the two steps explicit, the lienarized first block $A_t(x)$ maps a pixel change to a token change that depends on local image content, and the fixed lens keeps only the component of that change that matters to the output.
A pixel change can therefore be nearly invisible to \name{} in two ways: it may barely move the tokens, or it may move them mostly along directions that the output ignores. Based on our observation, the second case is more common.
Each term in $G(x)$ has rank at most $k$, so $\operatorname{rank}G(x)\leq kN=64N$, while the image contains $3p^2N=588N$ pixel values; therefore, at least 89\% of pixel directions leave the lens term unchanged to first order at every image. Thus, \name{} concentrates its sensitivity on an encoder-selected subspace that changes with the image.

Figure~\ref{fig:geo-contours} restricts the distance to planes $\varepsilon=s\,a+t\,b$ spanned by two distortion directions $a,b$ with equal pixel energy.
On such a plane, $G$ reduces to the $2\times2$ matrix $[a\;b]^{\top}G(x)[a\;b]$, whose contours are ellipses with half-axis ratio $(\lambda_{\max}/\lambda_{\min})^{1/2}$, which measures their anisotropy.
Around image I19 (TID2013) (Figure~\ref{fig:geo-contours}a), MSE produces nearly circular contours (radius ratio 1.06), whereas \name{} produces ellipses 2.33 times longer along chroma than blur; LPIPS-VGG and DISTS give ratios of 1.61 and 1.64. The encoder predicts this shape well: the restricted pullback has anisotropy 2.07, its fixed axis lies 3$^\circ$ from blur, and its ellipse matches the innermost contour within 9.2\% in radius. Blur is therefore a fixed direction, while an equal-energy chroma change is soft, consistent with the lower spatial acuity of human vision for colour \citep{mullen1985}. The same prediction holds on the other planes, JPEG is stiff relative to noise (anisotropy 1.62, within 10.8\%), while on I23 (TID2013) brightness is soft relative to noise (3.14).
Across the three rings of equal MSE, \name{} varies by factors of 2.4, 1.5, and 4.4, showing that equal pixel error can hide large differences in perceptual severity. When evaluated in small cells across a photograph, $G$ becomes smaller in textured regions, where people also detect faint patterns less reliably (masking; SRCC $-0.86$ against local contrast, Appendix Figure~\ref{fig:geo-field}).

% Owner: worker A (v23.3). Numbers: figures/v233/A_geo_slices_numbers.json; script scripts/v233/A_geo_slices.py
% (the v23 figure, scripts/v23/geo_slices.py, with a level-5 table of PSNR and MOS from results/baselines_fr_tid2013_official.csv).
\begin{figure}[t]
\centering
\includegraphics[width=0.9\linewidth]{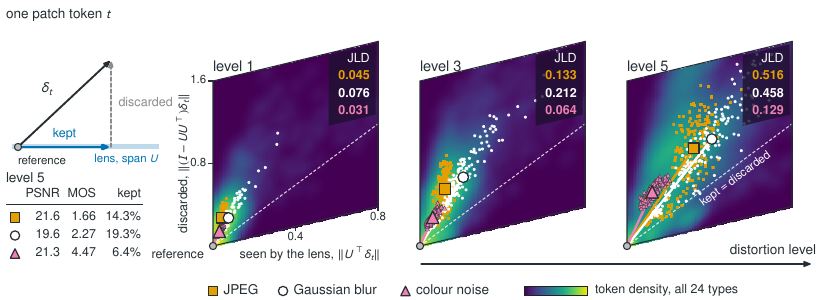}% arXiv layout: was 0.8
\vspace{-4pt}
\caption{Three distortions of I19 (TID2013) at levels 1, 3 and 5, per patch. Block-1 displacement the lens keeps (horizontal) against the part it discards (vertical). Markers sit at each family's RMS coordinates; their horizontal position is the lens term.}
\label{fig:geo-slices}
% \vspace{-6pt}% arXiv layout: dropped, body text follows the caption directly
\end{figure}

\subsection{Distortions through the lens}\label{sec:geo-paths}
For a finite distortion, each block-1 token moves by $\Delta_t=\tok{t}(\pre y)-\tok{t}(\pre x)$, which splits orthogonally into a kept part $\U_k\U_k^{\top}\Delta_t$ and a discarded part $(I-\U_k\U_k^{\top})\Delta_t$, with $\|\Delta_t\|^2=\|\U_k^{\top}\Delta_t\|^2+\|(I-\U_k\U_k^{\top})\Delta_t\|^2$.
The lens term measures the RMS of the kept component, whereas raw feature distance weighs both components equally.
A displacement with no preferred direction would retain $k/d=64/384$, or 16.7\%, of its squared norm; a smaller fraction means that the distortion moves the tokens mainly along directions ignored by the output.
Figure~\ref{fig:geo-slices} shows the kept and discarded components for every patch.
At level 5 on I19 (TID2013), JPEG and colour noise have nearly identical PSNR (21.6 and 21.3\,dB), but very different MOS values, 1.66 and 4.47. The lens retains 14.3\% of the JPEG displacement but only 6.4\% of the noise displacement, making \name{} four times larger for JPEG (0.516 versus 0.129), in agreement with human judgements.
Colour noise moves the tokens substantially, but mostly along discarded directions; JPEG block edges move them along directions used by the output. Across all 120 distortions of I19 (TID2013), the lens retains 9.6\% of the displacement energy; across all 24 distortion types, the median is 10.6\% (Appendix~\ref{app:geo-fingerprint}), well below the isotropic share. The lens therefore removes most block-1 feature motion; raw block-1 features count this discarded motion in full and rank the five images in Figure~\ref{fig:motivation} almost at random.

}{}
\section{Predicting human judgements}\label{sec:experiments}\label{sec:analysis}
In this section, we evaluate \name{} against 15 full-reference distances on human judgements of images and videos. Because no human label enters the lens fit, every result below is a direct test of the geometry developed in Sections~\ref{sec:method} and~\ref{sec:geometry}, not of a fit to the benchmarks.

\subsection{Image similarity benchmarks}\label{sec:fr}\label{sec:cost}\label{sec:selection}\label{sec:study}\label{sec:qual}
% v24: ranking marks and the IEM dagger are in the caption; cost protocol is in Section sec:fr.
% v23.3 (worker B): Mean column from figures/v233/B_numbers.json (table_main_mean; results/table_main_v2.csv); ourMetrics darkened to E4DAF2.
% Generated by scripts/70_tables_v5.py; v22 hand edits (W7): caption, correlation heading, first column heading.
% Cost protocol: results/notes/cost.md; inputs=512x384, IEM=256x256, A100, batch=1.
% Costs count visible tensor operations; values below 0.05 GFLOPs print as <0.1.
% Cell subset n=2032; exact=2032.0; source: results/table_main_v2.csv:24, row=iem_n, column=kadid10k_test.
% Cell subset n=250; exact=250.0; source: results/table_main_v2.csv:24, row=iem_n, column=pipal_val.
\begin{table}[t]
\centering\small
\setlength{\tabcolsep}{3pt}
\renewcommand{\arraystretch}{1.03}
\caption{Image distances compared on agreement with human ratings and on cost per pair. \name{} uses no human label, has the best mean over the datasets of every distance. \textbf{Bold} marks the best distance in a column; \underline{underlining} marks the second. $\dagger$ marks IEM results from subsets of $2{,}032$ KADID and $250$ PIPAL pairs.}
\label{tab:main}
\colorlet{humanJudgements}{gray!7}
\definecolor{ourMetrics}{HTML}{E4DAF2}
\begin{tabular*}{\linewidth}{@{\extracolsep{\fill}}lrrrrrrrr@{}}
\toprule
& \multicolumn{2}{c}{Cost per pair $\downarrow$} & \multicolumn{6}{c}{Spearman correlation with human ratings $\uparrow$} \\
\cmidrule(lr){2-3}\cmidrule(l){4-9}
Method & GFLOPs & ms & \shortstack{TID2013\\(diag.)} & \shortstack{CSIQ\\(dev.)} & \shortstack{LIVE\\test} & \shortstack{KADID-10k\\test} & \shortstack{\textbf{Mean}\\first four} & \shortstack{PIPAL\\val} \\
\midrule
\multicolumn{9}{@{}l}{\textbf{Learned deep distances, fitted to human judgements (similarity choices or quality scores)}} \\
% Cell GFLOPs=240.6; exact=240.582918144; source: results/cost.csv:13, metric=lpips_vgg, column=gflops_per_pair.
% Cell ms=10.9; exact=10.914095502812415; source: results/cost.csv:13, metric=lpips_vgg, column=gpu_ms_per_pair.
% Cell tid2013=0.670; exact=0.6695960547701812; source: results/table_main_v2.csv:13, row=lpips_vgg, column=tid2013.
% Cell csiq=0.883; exact=0.8830028532982714; source: results/table_main_v2.csv:13, row=lpips_vgg, column=csiq.
% Cell live=0.932; exact=0.9318494120383872; source: results/table_main_v2.csv:13, row=lpips_vgg, column=live.
% Cell kadid10k_test=0.724; exact=0.7238198259209079; source: results/table_main_v2.csv:13, row=lpips_vgg, column=kadid10k_test.
% Cell pipal_val=0.612; exact=0.6123784083784084; source: results/table_main_v2.csv:13, row=lpips_vgg, column=pipal_val.
\rowcolor{humanJudgements}LPIPS-VGG & 240.6 & 10.9 & 0.670 & 0.883 & 0.932 & 0.724 & 0.802 & 0.612 \\
% Cell GFLOPs=240.7; exact=240.747282432; source: results/cost.csv:15, metric=dists, column=gflops_per_pair.
% Cell ms=9.9; exact=9.936597023624929; source: results/cost.csv:15, metric=dists, column=gpu_ms_per_pair.
% Cell tid2013=0.818; exact=0.8178180388400073; source: results/table_main_v2.csv:15, row=dists, column=tid2013.
% Cell csiq=0.943; exact=0.9427767318284221; source: results/table_main_v2.csv:15, row=dists, column=csiq.
% Cell live=0.954; exact=0.9535764376250466; source: results/table_main_v2.csv:15, row=dists, column=live.
% Cell kadid10k_test=0.885; exact=0.8853784949063407; source: results/table_main_v2.csv:15, row=dists, column=kadid10k_test.
% Cell pipal_val=0.704; exact=0.70394017994018; source: results/table_main_v2.csv:15, row=dists, column=pipal_val.
\rowcolor{humanJudgements}DISTS & 240.7 & 9.9 & 0.818 & 0.943 & 0.954 & 0.885 & 0.900 & 0.704 \\
% Cell GFLOPs=385.7; exact=385.679253504; source: results/cost.csv:16, metric=pieapp, column=gflops_per_pair.
% Cell ms=16.5; exact=16.537297051399946; source: results/cost.csv:16, metric=pieapp, column=gpu_ms_per_pair.
% Cell tid2013=0.844; exact=0.8437660024619354; source: results/table_main_v2.csv:16, row=pieapp, column=tid2013.
% Cell csiq=0.897; exact=0.8968451251988768; source: results/table_main_v2.csv:16, row=pieapp, column=csiq.
% Cell live=0.918; exact=0.9182120141542843; source: results/table_main_v2.csv:16, row=pieapp, column=live.
% Cell kadid10k_test=0.864; exact=0.8643840757697688; source: results/table_main_v2.csv:16, row=pieapp, column=kadid10k_test.
% Cell pipal_val=0.706; exact=0.7062682782682783; source: results/table_main_v2.csv:16, row=pieapp, column=pipal_val.
\rowcolor{humanJudgements}PieAPP & 385.7 & 16.5 & 0.844 & 0.897 & 0.918 & 0.864 & 0.881 & 0.706 \\
% Cell GFLOPs=212.2; exact=212.154212352; source: results/cost.csv:18, metric=dreamsim, column=gflops_per_pair.
% Cell ms=51.3; exact=51.26894399290904; source: results/cost.csv:18, metric=dreamsim, column=gpu_ms_per_pair.
% Cell tid2013=0.812; exact=0.8119637299627569; source: results/table_main_v2.csv:17, row=dreamsim, column=tid2013.
% Cell csiq=0.911; exact=0.9114138879209581; source: results/table_main_v2.csv:17, row=dreamsim, column=csiq.
% Cell live=0.910; exact=0.9099194368953291; source: results/table_main_v2.csv:17, row=dreamsim, column=live.
% Cell kadid10k_test=0.851; exact=0.8505368950786634; source: results/table_main_v2.csv:17, row=dreamsim, column=kadid10k_test.
% Cell pipal_val=\textbf{0.759}; exact=0.7589268029268029; source: results/table_main_v2.csv:17, row=dreamsim, column=pipal_val.
\rowcolor{humanJudgements}DreamSim & 212.2 & 51.3 & 0.812 & 0.911 & 0.910 & 0.851 & 0.871 & \textbf{0.759} \\
\midrule
\multicolumn{9}{@{}l}{\textbf{Distances not fitted to human judgements}} \\
% Cell GFLOPs=\textbf{$<$0.1}; exact=0.0; source: results/cost.csv:4, metric=psnr, column=gflops_per_pair.
% Cell ms=\textbf{1.4}; exact=1.4140065177343786; source: results/cost.csv:4, metric=psnr, column=gpu_ms_per_pair.
% Cell tid2013=0.687; exact=0.6869111186179578; source: results/table_main_v2.csv:2, row=psnr, column=tid2013.
% Cell csiq=0.809; exact=0.8086909105297525; source: results/table_main_v2.csv:2, row=psnr, column=csiq.
% Cell live=0.873; exact=0.872970661502979; source: results/table_main_v2.csv:2, row=psnr, column=live.
% Cell kadid10k_test=0.673; exact=0.6730671670055804; source: results/table_main_v2.csv:2, row=psnr, column=kadid10k_test.
% Cell pipal_val=0.255; exact=0.2547583467583468; source: results/table_main_v2.csv:2, row=psnr, column=pipal_val.
PSNR & \textbf{$<$0.1} & \textbf{1.4} & 0.687 & 0.809 & 0.873 & 0.673 & 0.760 & 0.255 \\
% Cell GFLOPs=0.2; exact=0.234252968; source: results/cost.csv:5, metric=ssim, column=gflops_per_pair.
% Cell ms=2.7; exact=2.678551943972707; source: results/cost.csv:5, metric=ssim, column=gpu_ms_per_pair.
% Cell tid2013=0.627; exact=0.6269215690600412; source: results/table_main_v2.csv:3, row=ssim, column=tid2013.
% Cell csiq=0.837; exact=0.8366574573995449; source: results/table_main_v2.csv:3, row=ssim, column=csiq.
% Cell live=0.910; exact=0.9102790599554298; source: results/table_main_v2.csv:3, row=ssim, column=live.
% Cell kadid10k_test=0.621; exact=0.6212832943376392; source: results/table_main_v2.csv:3, row=ssim, column=kadid10k_test.
% Cell pipal_val=0.363; exact=0.3634878394878395; source: results/table_main_v2.csv:3, row=ssim, column=pipal_val.
SSIM & \underline{0.2} & 2.7 & 0.627 & 0.837 & 0.910 & 0.621 & 0.749 & 0.363 \\
% Cell GFLOPs=0.3; exact=0.303561768; source: results/cost.csv:6, metric=ms_ssim, column=gflops_per_pair.
% Cell ms=5.1; exact=5.118806497193873; source: results/cost.csv:6, metric=ms_ssim, column=gpu_ms_per_pair.
% Cell tid2013=0.786; exact=0.7859331139143848; source: results/table_main_v2.csv:4, row=ms_ssim, column=tid2013.
% Cell csiq=0.913; exact=0.9131442004422196; source: results/table_main_v2.csv:4, row=ms_ssim, column=csiq.
% Cell live=0.951; exact=0.9508257413297416; source: results/table_main_v2.csv:4, row=ms_ssim, column=live.
% Cell kadid10k_test=0.826; exact=0.825911654586731; source: results/table_main_v2.csv:4, row=ms_ssim, column=kadid10k_test.
% Cell pipal_val=0.491; exact=0.4906342546342547; source: results/table_main_v2.csv:4, row=ms_ssim, column=pipal_val.
MS-SSIM & 0.3 & 5.1 & 0.786 & 0.913 & 0.951 & 0.826 & 0.869 & 0.491 \\
% Cell GFLOPs=$<$0.1; exact=0.005308416; source: results/cost.csv:7, metric=fsim, column=gflops_per_pair.
% Cell ms=22.8; exact=22.846897481940687; source: results/cost.csv:7, metric=fsim, column=gpu_ms_per_pair.
% Cell tid2013=0.851; exact=0.8509237604556346; source: results/table_main_v2.csv:5, row=fsim, column=tid2013.
% Cell csiq=0.931; exact=0.9308888220293462; source: results/table_main_v2.csv:5, row=fsim, column=csiq.
% Cell live=\textbf{0.965}; exact=0.9646174453827938; source: results/table_main_v2.csv:5, row=fsim, column=live.
% Cell kadid10k_test=0.853; exact=0.8525771834733207; source: results/table_main_v2.csv:5, row=fsim, column=kadid10k_test.
% Cell pipal_val=0.468; exact=0.46784999525354526; source: results/table_main_v2.csv:5, row=fsim, column=pipal_val.
FSIM & \textbf{$<$0.1} & 22.8 & 0.851 & 0.931 & \textbf{0.965} & 0.853 & 0.900 & 0.468 \\
% Cell GFLOPs=$<$0.1; exact=0.015335424; source: results/cost.csv:8, metric=vsi, column=gflops_per_pair.
% Cell ms=9.9; exact=9.872902475763112; source: results/cost.csv:8, metric=vsi, column=gpu_ms_per_pair.
% Cell tid2013=\textbf{0.895}; exact=0.8949480473399499; source: results/table_main_v2.csv:6, row=vsi, column=tid2013.
% Cell csiq=0.940; exact=0.9397328817046876; source: results/table_main_v2.csv:6, row=vsi, column=csiq.
% Cell live=0.949; exact=0.9489654338574829; source: results/table_main_v2.csv:6, row=vsi, column=live.
% Cell kadid10k_test=0.876; exact=0.8763342263535929; source: results/table_main_v2.csv:6, row=vsi, column=kadid10k_test.
% Cell pipal_val=0.450; exact=0.4499776513275843; source: results/table_main_v2.csv:6, row=vsi, column=pipal_val.
VSI & \textbf{$<$0.1} & 9.9 & \textbf{0.895} & 0.940 & 0.949 & 0.876 & \underline{0.915} & 0.450 \\
% Cell GFLOPs=$<$0.1; exact=0.011403264; source: results/cost.csv:9, metric=gmsd, column=gflops_per_pair.
% Cell ms=2.6; exact=2.5983754894696176; source: results/cost.csv:9, metric=gmsd, column=gpu_ms_per_pair.
% Cell tid2013=0.804; exact=0.8043827011850794; source: results/table_main_v2.csv:7, row=gmsd, column=tid2013.
% Cell csiq=0.957; exact=0.95696941737002; source: results/table_main_v2.csv:7, row=gmsd, column=csiq.
% Cell live=0.960; exact=0.9602495815006324; source: results/table_main_v2.csv:7, row=gmsd, column=live.
% Cell kadid10k_test=0.849; exact=0.8486989796577927; source: results/table_main_v2.csv:7, row=gmsd, column=kadid10k_test.
% Cell pipal_val=0.585; exact=0.5847956727956728; source: results/table_main_v2.csv:7, row=gmsd, column=pipal_val.
GMSD & \textbf{$<$0.1} & 2.6 & 0.804 & 0.957 & 0.960 & 0.849 & 0.893 & 0.585 \\
% Cell GFLOPs=$<$0.1; exact=0.049952464; source: results/cost.csv:11, metric=nlpd, column=gflops_per_pair.
% Cell ms=15.6; exact=15.560497005935758; source: results/cost.csv:11, metric=nlpd, column=gpu_ms_per_pair.
% Cell tid2013=0.799; exact=0.7993513266405812; source: results/table_main_v2.csv:9, row=nlpd, column=tid2013.
% Cell csiq=0.937; exact=0.9366452783775855; source: results/table_main_v2.csv:9, row=nlpd, column=csiq.
% Cell live=0.937; exact=0.9366972283652567; source: results/table_main_v2.csv:9, row=nlpd, column=live.
% Cell kadid10k_test=0.811; exact=0.8112403472101343; source: results/table_main_v2.csv:9, row=nlpd, column=kadid10k_test.
% Cell pipal_val=0.370; exact=0.37020814620814624; source: results/table_main_v2.csv:9, row=nlpd, column=pipal_val.
NLPD & \textbf{$<$0.1} & 15.6 & 0.799 & 0.937 & 0.937 & 0.811 & 0.871 & 0.370 \\
% Cell GFLOPs=60.2; exact=60.189179904; source: results/cost.csv:19, metric=deepwsd, column=gflops_per_pair.
% Cell ms=13.5; exact=13.5155440075323; source: results/cost.csv:19, metric=deepwsd, column=gpu_ms_per_pair.
% Cell tid2013=0.855; exact=0.8552530065285766; source: results/table_main_v2.csv:11, row=deepwsd, column=tid2013.
% Cell csiq=0.959; exact=0.9589637541571434; source: results/table_main_v2.csv:11, row=deepwsd, column=csiq.
% Cell live=0.955; exact=0.9550590482305559; source: results/table_main_v2.csv:11, row=deepwsd, column=live.
% Cell kadid10k_test=0.882; exact=0.8817277468423692; source: results/table_main_v2.csv:11, row=deepwsd, column=kadid10k_test.
% Cell pipal_val=0.460; exact=0.46049852849852846; source: results/table_main_v2.csv:11, row=deepwsd, column=pipal_val.
DeepWSD & 60.2 & 13.5 & 0.855 & \underline{0.959} & 0.955 & 0.882 & 0.913 & 0.460 \\
% Cell GFLOPs=326.4; exact=326.356697088; source: results/cost.csv:17, metric=deepdc, column=gflops_per_pair.
% Cell ms=16.1; exact=16.072013997472823; source: results/cost.csv:17, metric=deepdc, column=gpu_ms_per_pair.
% Cell tid2013=0.816; exact=0.816085021125814; source: results/table_main_v2.csv:12, row=deepdc, column=tid2013.
% Cell csiq=0.954; exact=0.9536247776165595; source: results/table_main_v2.csv:12, row=deepdc, column=csiq.
% Cell live=0.950; exact=0.9503388908554229; source: results/table_main_v2.csv:12, row=deepdc, column=live.
% Cell kadid10k_test=\textbf{0.900}; exact=0.8996447545161771; source: results/table_main_v2.csv:12, row=deepdc, column=kadid10k_test.
% Cell pipal_val=0.751; exact=0.7511368991368992; source: results/table_main_v2.csv:12, row=deepdc, column=pipal_val.
DeepDC & 326.4 & 16.1 & 0.816 & 0.954 & 0.950 & \textbf{0.900} & 0.905 & \underline{0.751} \\
% Cell GFLOPs=0.3; exact=0.339738624; source: results/cost.csv:20, metric=lasi, column=gflops_per_pair.
% Cell ms=377.6; exact=377.6092884945683; source: results/cost.csv:20, metric=lasi, column=gpu_ms_per_pair.
% Cell tid2013=0.522; exact=0.5216480838070444; source: results/table_main_v2.csv:18, row=lasi, column=tid2013.
% Cell csiq=0.840; exact=0.8400458639998005; source: results/table_main_v2.csv:18, row=lasi, column=csiq.
% Cell live=0.883; exact=0.8825029952828819; source: results/table_main_v2.csv:18, row=lasi, column=live.
% Cell kadid10k_test=0.723; exact=0.7231234797209626; source: results/table_main_v2.csv:18, row=lasi, column=kadid10k_test.
% Cell pipal_val=0.581; exact=0.5811168411168411; source: results/table_main_v2.csv:18, row=lasi, column=pipal_val.
LASI & 0.3 & 377.6 & 0.522 & 0.840 & 0.883 & 0.723 & 0.742 & 0.581 \\
% Cell GFLOPs=4,999.1; exact=4999.11643136; source: results/cost.csv:21, metric=iem, column=gflops_per_pair.
% Cell ms=725.7; exact=725.7345230318606; source: results/cost.csv:21, metric=iem, column=gpu_ms_per_pair.
% Cell tid2013=0.810; exact=0.8103680917006899; source: results/table_main_v2.csv:19, row=iem, column=tid2013.
% Cell csiq=0.917; exact=0.9174106366457856; source: results/table_main_v2.csv:19, row=iem, column=csiq.
% Cell live=0.922; exact=0.922002995098844; source: results/table_main_v2.csv:19, row=iem, column=live.
% Cell kadid10k_test=0.821$^{\dagger}$; exact=0.8206601110745053; source: results/table_main_v2.csv:19, row=iem, column=kadid10k_test.
% Cell pipal_val=0.182$^{\dagger}$; exact=0.1821652186434983; source: results/table_main_v2.csv:19, row=iem, column=pipal_val.
IEM$\dagger$ & 4,999.1 & 725.7 & 0.810 & 0.917 & 0.922 & 0.821$^{\dagger}$ & 0.868$^{\dagger}$ & 0.182$^{\dagger}$ \\
\midrule
\multicolumn{9}{@{}l}{\textbf{Ours: Label Free}} \\
% Cell GFLOPs=10.7; exact=10.720791552; source: results/cost.csv:2, metric=jld, column=gflops_per_pair.
% Cell ms=2.5; exact=2.525739546399564; source: results/cost.csv:2, metric=jld, column=gpu_ms_per_pair.
% Cell tid2013=0.875; exact=0.8745150738826878; source: results/table_main_v2.csv:21, row=jld_v2, column=tid2013.
% Cell csiq=0.956; exact=0.9558676422362274; source: results/table_main_v2.csv:21, row=jld_v2, column=csiq.
% Cell live=0.945; exact=0.9448069352546935; source: results/table_main_v2.csv:21, row=jld_v2, column=live.
% Cell kadid10k_test=0.870; exact=0.8704005316941635; source: results/table_main_v2.csv:21, row=jld_v2, column=kadid10k_test.
% Cell pipal_val=0.578; exact=0.578056178056178; source: results/table_main_v2.csv:21, row=jld_v2, column=pipal_val.
\rowcolor{ourMetrics}\namefast{} & 10.7 & \underline{2.5} & 0.875 & 0.956 & 0.945 & 0.870 & 0.911 & 0.578 \\
% Cell GFLOPs=118.5; exact=118.467293184; source: results/cost.csv:3, metric=jld_full12, column=gflops_per_pair.
% Cell ms=12.2; exact=12.249739957042038; source: results/cost.csv:3, metric=jld_full12, column=gpu_ms_per_pair.
% Cell tid2013=0.877; exact=0.8767106009025009; source: results/table_main_v2.csv:22, row=jld_v2p, column=tid2013.
% Cell csiq=\textbf{0.971}; exact=0.9708250295632062; source: results/table_main_v2.csv:22, row=jld_v2p, column=csiq.
% Cell live=0.964; exact=0.9638385267329496; source: results/table_main_v2.csv:22, row=jld_v2p, column=live.
% Cell kadid10k_test=0.892; exact=0.8916927527531087; source: results/table_main_v2.csv:22, row=jld_v2p, column=kadid10k_test.
% Cell pipal_val=0.624; exact=0.6242270522270523; source: results/table_main_v2.csv:22, row=jld_v2p, column=pipal_val.
\rowcolor{ourMetrics}\name{} & 118.5 & 12.2 & \underline{0.877} & \textbf{0.971} & \underline{0.964} & \underline{0.892} & \textbf{0.926} & 0.624 \\
\bottomrule
\end{tabular*}
\end{table}

We report Spearman rank correlation (\srcc{}) between each distance and human scores on TID2013, CSIQ, LIVE, KADID-10k and PIPAL \citep{ponomarenko2015tid2013,larson2010mad,sheikh2006live,lin2019kadid,gu2020pipal}.
LIVE, KADID-10k test references and PIPAL validation are held out; CSIQ and the KADID development references select the settings, while TID2013 informs the pre-filter. All methods are evaluated under one protocol (Appendix~\ref{app:sweep}). Table~\ref{tab:main} averages the first four databases. All costs are measured on an A100 using $512\times384$ inputs, except IEM at $256\times256$.

% v23.7 worker M7: compact main-text version of fig:choice (the label moved here from sections/figs/fig_choice.tex,
% whose input in sections/app_main_moved.tex is commented out). Re-plot only: scripts/v237/M7_choice.py reads
% figures/v23/fig_choice_numbers.json -> figures/v237/M7_choice.pdf; numbers in figures/v237/M7_numbers.json.
% v23.7b: the user asked for the original A14 figure (per dataset) in the main text; the compact wrap is kept below.
% \begin{wrapfigure}{r}{2.36in}
% \vspace{-1.2em}
% \centering
% \includegraphics[width=2.3in]{v237/M7_choice.pdf}
% \vspace{-0.6em}
% \caption{Agreement with 699{,}534 human choices, mean of four datasets. Grey: fitted to human judgements.}
% \label{fig:choice}
% \vspace{-1.2em}
% \end{wrapfigure}
\begin{figure}[t]% arXiv layout: was [b]
\centering
\includegraphics[width=\linewidth]{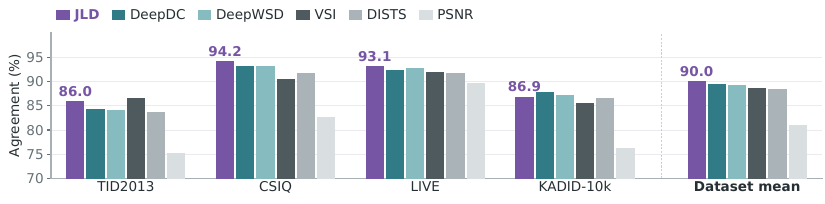}% arXiv layout: was 0.85
\vspace{-0.4em}
\caption{Agreement with 699{,}534 human same-reference choices, per dataset and as the mean over the four datasets.}
\label{fig:choice}
% \vspace{-0.8em}% arXiv layout: dropped, the float now sits above a subsection heading
\end{figure}

\name{} sets a new state of the art: it achieves the highest mean \srcc{} over TID2013, CSIQ, LIVE and KADID-10k test among all distances in Table~\ref{tab:main}, 0.926, versus 0.915 for VSI, 0.913 for DeepWSD, 0.905 for DeepDC and 0.900 for DISTS. Every gap has a 95\% interval over resampled references that excludes zero (Appendix~\ref{app:moved-results}).
\name{} exceeds every learned deep distance on all four datasets, although those distances are fitted to human judgements and \name{} is not, and it is first or second on each dataset. On PIPAL's restoration and GAN outputs, \name{} reaches 0.624 without any human label, above LPIPS-VGG and every classical distance.
On BAPPS CNN distortions, it agrees with 0.821 of human choices, within 0.005 of LPIPS-Alex (0.824) and DreamSim (0.826), both of which are trained on BAPPS (Appendix~\ref{app:results}); it also beats DISTS on all distortion types (Appendix~\ref{app:types}).

Correlations summarise entire databases, so we also evaluate individual pairwise decisions.
Across all 699{,}534 same-reference pairs with different human scores, \name{} places the preferred image closer to the reference in 90.0\% of cases, the highest among the 17 distances, ahead of DeepDC (89.4\%) and DISTS (88.4\%) (Figure~\ref{fig:choice}); it ranks first on CSIQ, second on TID2013 and LIVE, and third on KADID-10k.
Among the 24{,}623 pairs whose PSNR differs by less than 0.25\,dB, where PSNR is at chance (48.7\%), \name{} agrees with people on 76.0\%, compared with 74.4\% for DISTS and 65.3\% for LPIPS-VGG (mean over the four datasets).
Figure~\ref{fig:B-pairs} (Appendix~\ref{app:main-moved}) shows four such pairs where \name{} agrees with people while both DISTS and LPIPS-VGG do not.
In a ten-person study on trials where the methods disagree, participants side with \name{} on 24 of 28 decided trials and with DISTS on 14 of 26 (Appendix~\ref{app:validation}).

\subsection{Seeing the distance}\label{sec:see}
% v23.5 worker M: new reference for (a) (I45, seven images per row, full width) and new ruler for (b) (I32).
% Selection and maps: scripts/v235/M_see_select.py; drawing: scripts/v235/M_see_figure.py; numbers: figures/v235/M_see_numbers.json.
% KADID-10k test images (held out) with their human MOS (1 to 5). The v23.4 version is kept below as comments.
\begin{figure}[t]
\centering
\includegraphics[width=\linewidth]{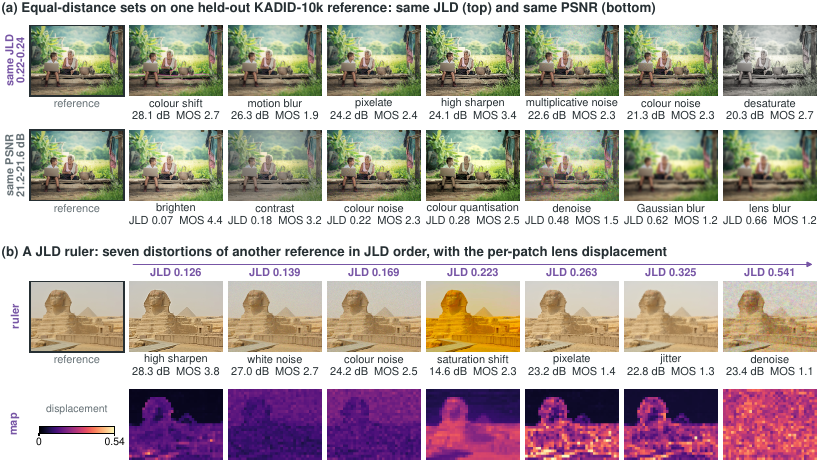}
\caption{Equal \name{} means equal damage. (a) Held-out KADID-10k distortions of one reference at equal \name{} (top) and at equal PSNR (bottom), with human scores (MOS, 1 to 5). (b) Distortions of another reference in \name{} order, with the per-patch \name{} map of each. Selection rules in Section~\ref{sec:see}.}
\label{fig:B2-see}
\end{figure}
% ---- v23.4 version ----
% % v23.4 worker B2. Selection, maps and numbers: scripts/v234/B2_see_select.py; drawing: scripts/v234/B2_see_figure.py;
% % numbers: figures/v234/B2_numbers.json ('see'). KADID-10k test images (held out) with their human MOS (1 to 5).
% \begin{figure}[t]
% \centering
% \includegraphics[width=\linewidth]{v234/B2_see.pdf}
% \caption{Seeing the distance. (a) Top: distortions of one reference at the same \name{} (0.16 to 0.18); their PSNR spans 10\,dB. Bottom: distortions at the same PSNR (25.2 to 25.9\,dB); \name{} varies sevenfold and follows the human score (MOS, 1 to 5), from a barely visible mean shift to a smeared motion blur. (b) Seven distortions of another reference in \name{} order, with \name{}'s map of each: the MOS falls from 4.5 to 2.0 along the ruler while PSNR jumps up and down. Each type is chosen by a fixed rule (Section~\ref{sec:see}).}
% \label{fig:B2-see}
% \end{figure}

A useful perceptual distance should assign similar values to distortions of similar perceptual severity (Figure~\ref{fig:B2-see}, held-out KADID-10k test images). For one reference, we choose from each of the 25 distortion types the level closest to a target \name{} value of 0.233 and retain the 9 types within 7\%; their PSNR ranges from 20.3 to 28.1\,dB, yet their human scores remain close, with a standard deviation of 0.46 MOS points.
Applying the same rule around a target PSNR of 21.5\,dB (within 0.5\,dB) retains 10 types whose \name{} values vary ninefold, from 0.071 to 0.663, and whose human scores spread almost three times as much, with a standard deviation of 1.26.
A brightening and a lens blur with equal PSNR receive human scores of 4.4 and 1.2, and \name{} values of 0.07 and 0.66.
The ruler in panel (b) starts at the 0.35 \name{} quantile of another reference, making every step visible: ordered by \name{}, the human score decreases at every step from 3.8 to 1.1, whereas PSNR is not even monotone, with a 14.6\,dB saturation shift in the middle; across all 121 distortions of that reference, \name{} reaches 0.921 \srcc{} and PSNR 0.566.
Each map gives an exact spatial account of its score because its root mean square equals the lens term: the saturation shift is distributed across the whole sphinx, whereas pixelation and jitter concentrate on its edges.
Across complete datasets, neighbours in a distance's ordering of a reference's distortions should have similar human scores; on KADID-10k test, neighbouring distortions differ by 0.511 MOS points under \name{}, 0.514 under DISTS, 0.797 under PSNR and 0.806 under LPIPS-VGG, and on TID2013 by 0.547 under \name{}, 0.642 under DISTS and 0.972 under PSNR (MOS range 0 to 9).
\IfFileExists{sections/F_field_main.tex}{% v23.6 (worker M6): the main-text glyph figure and its paragraph are commented out. The figure moved to the appendix as
}{}

\needspace{13\baselineskip}% arXiv layout: keep the wrapped figure and its text on one page
\subsection{Ablation}\label{sec:ablation}
% v23.4 worker B2: wrapfigure with the climb only; the four-panel version is fig:B2-ablation-full in
% sections/app_main_moved.tex. Drawing: scripts/v234/B2_plots.py -> figures/v234/B2_ablation_{wrap,full}.
\begin{wrapfigure}{r}{2.36in}
\vspace{-1.4em}
\centering
\includegraphics[width=2.3in]{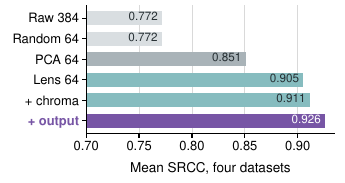}
\vspace{-0.5em}
\caption{Building \name{} step by step on DINOv2-S block 1.}
\label{fig:B-ablation}
\vspace{-1.2em}
\end{wrapfigure}

We ablate the components of \name{} in Figure~\ref{fig:B-ablation}; the gain comes from which directions the lens selects, not from how many are kept (Table~\ref{tab:climb}).
Using all 384 raw block-1 features gives a four-set mean of 0.772, and 64 random directions give the same value; 64 PCA directions reach 0.851, while the 64 lens directions reach 0.905 and perform better on every dataset.
The chroma pre-filter and output term raise the mean further, to 0.911 and 0.926.
PCA peaks at 16 directions and then declines, while the lens rises to 0.913 at $k=64$ (Figure~\ref{fig:B2-ablation-full}b); treating the selected directions equally also outperforms weighting them by the metric tensor, 0.916 versus 0.897 (Table~\ref{tab:flat-ablation}).
The construction is not specific to DINOv2: fitting a lens to each of 14 frozen encoders, including ViTs and CNNs from DINOv3 \citep{simeoni2025dinov3} to CLIP \citep{radford2021clip} and VGG16 \citep{simonyan2015vgg}, improves over raw features from the same block for all 14 (Figure~\ref{fig:B2-ablation-full}d; Appendix~\ref{app:encoders}).
% arXiv: the quality-transfer appendix is not in this version.
% It also converts the frozen no-reference TOPIQ and DOVER encoders into reference distances \citep{chen2024topiq,wu2023dover} (Appendix~\ref{app:results}).

\subsection{Resolution, cost and video}\label{sec:resolution}\label{sec:encoders}\label{sec:video}\label{sec:video-form}
% v23.4 worker B2. Data: results/scorecard_resolution.csv, figures/v23/fig_motivation_numbers.json, figures/v23/app_cost_numbers.json,
% figures/v23/video_heldout_numbers.json; drawing: scripts/v234/B2_plots.py -> figures/v234/B2_numbers.json (plots, pareto).
\begin{figure}[b]% arXiv layout: was [t]; keeps Figure 6 before Figure 7 on the page
\centering
\includegraphics[width=\linewidth]{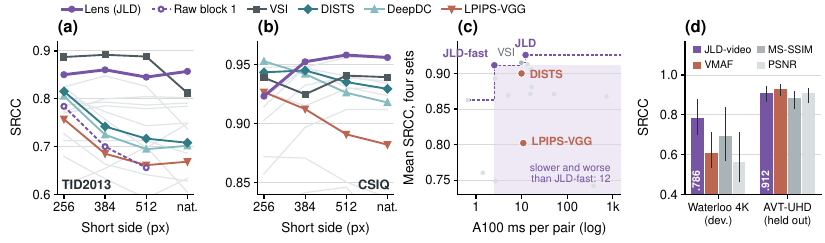}
\caption{Resolution, cost and video. (a, b) \srcc{} at each short side and at native size; grey lines are the other ten methods, dashed is unweighted block-1 features. (c) Mean \srcc{} over the first four datasets of Table~\ref{tab:main} against A100 time per pair, with the accuracy-cost frontier (dashed). (d) Video \srcc{} with 95\% bootstrap intervals.}
\label{fig:B-robust}
\end{figure}

To test stability across display resolutions, we rescore the labelled TID2013 and CSIQ pairs with 15 methods at short-side resolutions of 256, 384 and 512 pixels and at native size, using the lens term (Figure~\ref{fig:B-robust}a,b; Table~\ref{tab:resolution}).
From 256 to 512 pixels, the lens remains stable on TID2013 at 0.850, 0.860 and 0.845, while the unweighted block-1 features from which it is built fall from 0.784 to 0.656, DISTS from 0.815 to 0.717 and LPIPS-VGG from 0.757 to 0.661.
On CSIQ, the lens rises from 0.923 to 0.958, the highest of the 15 methods at 512 pixels.
Its lowest TID2013 value, 0.845, exceeds the lowest value of every other method. \name{} is also efficient to evaluate (Figure~\ref{fig:B-robust}c). \namefast{}, which uses the lens term with the pre-filter, retains a mean of 0.911 at 2.5\,ms per pair on an A100, 4.3 times faster than LPIPS-VGG (10.9\,ms) and 4.0 times faster than DISTS (9.9\,ms). No other distance in Table~\ref{tab:main} is both faster and more accurate: only PSNR is faster, and only VSI and DeepWSD score higher, at four to five times the cost.

The same construction extends to video by changing only the encoder. \namevideo{} uses a 16-direction lens on block 2 of a frozen VideoMAE-Base \citep{tong2022videomae}, fitted on 50 unlabelled clips; its space-time tokens capture temporal effects such as flicker that no single frame reveals (Appendix~\ref{app:video}). On the 240 Waterloo IVC 4K pairs \citep{li2019waterloo4k}, which are used to choose the block and rank, it reaches 0.786 \srcc{}, compared with 0.694 for MS-SSIM, 0.611 for VMAF \citep{li2016vmaf} and 0.562 for PSNR, giving a 0.175 lead over VMAF (Figure~\ref{fig:B-robust}d).
On 120 held-out AVT-VQDB-UHD-1 pairs \citep{rao2019avt}, it reaches 0.912, within 0.020 of VMAF (0.932), which is trained on human video scores, while \namevideo{} uses none. Figure~\ref{fig:qual-landscape} revisits a held-out image pair with nearly equal PSNR (25.18 and 25.16\,dB): viewers prefer A (MOS 3.67) to B (1.47); \name{} agrees, whereas DISTS and LPIPS-VGG reverse the ordering (Appendix~\ref{app:gallery}).
\begin{figure}[H]
\centering
\includegraphics[width=\linewidth]{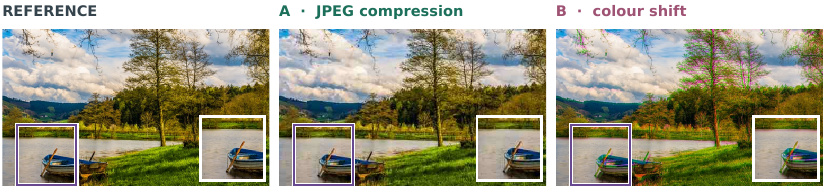}
\caption{Held-out KADID-10k pair. White boxes locate the enlarged boat detail.}
\label{fig:qual-landscape}
\end{figure}

\section{Discussion and conclusion}\label{sec:discussion}\label{sec:conclusion}\label{sec:bapps}
A frozen encoder's output Jacobian tells us how its early features should be weighted. \name{} turns this sensitivity into one fixed metric tensor, fitted in 35 seconds from 100 unlabelled images. Among the distances evaluated here, it achieves the highest mean agreement with human judgements across all image databases, while remaining stable as image resolution doubles. The central principle is simple, measure with early features and weight them using later representations. This combination gives both accuracy and resolution stability (Appendix Figure~\ref{fig:A2-depth}). The lens improves over raw features for all 14 tested encoders, while its fixed geometry makes individual predictions directly inspectable through maps and contours. \namefast{} requires only 2.5\,ms per pair. The seconds-long fit can also transfer the construction to a new encoder or domain without human ratings. More broadly, any differentiable encoder, including those for audio, medical imaging or 3D, can in principle define such a lens; closing its null space with the full metric tensor could further strengthen its use as a training loss (Appendix~\ref{app:limitations}). Perceptual weights need not come from human judgements, they can emerge from the geometry of a network that has simply learned to see.

% \clearpage % arXiv layout: no page limit, statements follow the conclusion
% arXiv: LLM-use statement removed at the author's request.
% \input{sections/ai_use}
\section*{Reproducibility Statement}
We release the full code repository for all variants of \name{}, together with the pre-fitted DINOv2-S lens, so the distance can be used without refitting. Section~\ref{sec:method} and Algorithm~\ref{alg:A2-jld} define the fit, the configuration and the distance, and Appendix~\ref{app:usage} gives a runnable example. Appendices~\ref{app:sweep}, \ref{app:fit}, \ref{app:video} and~\ref{app:validation} give the baseline, fitting, video and human-study protocols. All image and video benchmarks are publicly available, and the full code base, including the lens fitting, evaluation and figure scripts, will be released publicly upon publication.

\bibliography{refs}
\bibliographystyle{jld_preprint}

% \newpage % arXiv layout: the appendix follows the references without a page break
\appendix
% Start appendix float pages at the top instead of vertically centring isolated figures.
\makeatletter\setlength{\@fptop}{0pt}\makeatother
\setcounter{figure}{0}\renewcommand{\thefigure}{A\arabic{figure}}
\setcounter{table}{0}\renewcommand{\thetable}{A\arabic{table}}
\addtocontents{toc}{\protect\setcounter{tocdepth}{1}}
% Compact one-page appendix contents in the LumaFlux style: section entries only, small type,
% tight dotted lines, whole entry linked.
\section*{Appendix}
\begingroup\makeatletter
\setlength{\parskip}{0pt}\hypersetup{linktoc=all}
\renewcommand*\l@section{\@dottedtocline{1}{0em}{1.6em}}
\small\@starttoc{toc}
\makeatother\endgroup
\vspace{2pt}
% Appendix contents use the v5 text throughout.
% The appendix is float heavy: let several floats share a text page and pack float-only
% pages from the top (the same settings as the LumaFlux appendix).
\setcounter{topnumber}{4}\setcounter{bottomnumber}{3}\setcounter{totalnumber}{8}
\renewcommand{\topfraction}{0.92}\renewcommand{\bottomfraction}{0.85}
\renewcommand{\textfraction}{0.06}\renewcommand{\floatpagefraction}{0.80}
\setlength{\columnsep}{12pt}
% v23.4 (C2): method details first, then the proofs, then usage and results; Limitations always last.
\Needspace{4\baselineskip}
\section{Method details}\label{app:method-details}
This section gives the implementation details of \name{}. The encoder and features, lens calibration, rank selection, image-space corrections, variants, fit stability, and computational cost.

% v23.6 worker M6: full-width one-row version of fig:A-lens (replaces the v23.5 wrapfigure in fig_A_lens.tex,
% which is no longer input). Re-plot only: scripts/v236/M6_lens_row.py -> figures/v236/M6_lens_row.pdf
% (calibration drawing of scripts/v22/calibration.py; projection and scatter of scripts/v21/intuition.py,
% 153 perturbations at 40 dB, Spearman 0.643 lens, 0.147 raw, checked against figures/v5/fig2_intuition_numbers.json).
\begin{figure}[!h]
\centering
% \includegraphics[width=\linewidth]{v236/M6_lens_row.pdf}
% v23.7 worker M7: vertically compressed re-plot (scripts/v237/M7_lens_row.py, 5.5 x 1.42 in instead of 1.80 in).
\includegraphics[width=\linewidth]{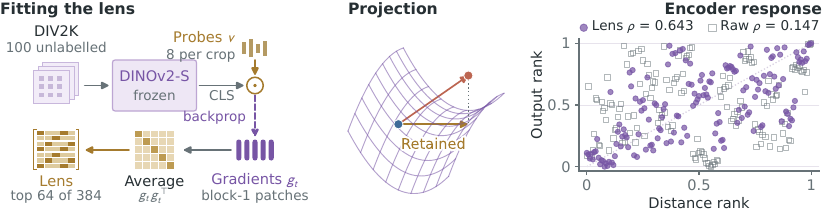}
\vspace{-14pt}
% \caption{The Jacobian lens. Left: output probes are backpropagated to the block-1 patch tokens of 100 unlabelled images, and the top 64 eigenvectors of the averaged gradient outer products form the lens. Middle: the lens retains the component of a feature change that moves the output. Right: over 153 perturbations at equal pixel error (40\,dB), the lens displacement ranks the change in the encoder output far better than the raw block-1 displacement.}
\caption{The Jacobian lens. Left: output probes are backpropagated to the block-1 patch tokens of 100 unlabelled images; the top 64 eigenvectors of the averaged gradient outer products form the lens. Middle: the lens keeps the part of a feature change that moves the output. Right: over 153 perturbations at 40\,dB PSNR, the lens displacement ranks the encoder output change far better than the raw block-1 displacement.}
\label{fig:A-lens}
\vspace{-6pt}
\end{figure}

\subsection{Encoder and patch features}\label{app:md-encoder}
% Depth sweep (worker A2, v23.4): moved from the main-text motivation figure (v23.3 fig:motivation) to the appendix.
% Sources: figures/v233/A_depth_numbers.json (scripts/v233/A_depth.py) and figures/v23/fig_motivation_numbers.json (panel b).
% Plot: scripts/v233/A_motivation_fig.py. The manager inputs this file in the appendix (method details).
% v23.8 worker L8: only panel (a) is kept, as a small wrapfigure (re-plot of the same numbers,
% scripts/v238/L8_depth_panel.py -> figures/v238/L8_depth.pdf from figures/v233/A_depth_numbers.json). Panel (b), the
% resize study, is summarised in the text beside it (figures/v23/fig_motivation_numbers.json). The two-panel float is
% commented out below.
% arXiv layout: the wrapfigure ran off the page foot once the appendix followed the references; it is now a top float.
% \begin{wrapfigure}{r}{0.48\linewidth}
% \vspace{-12pt}
% \centering
% \includegraphics[width=\linewidth]{v238/L8_depth.pdf}
% \vspace{-16pt}
% \caption{TID2013 \srcc{} of patch-token distances at each depth of DINOv2-S, raw or through the lens fitted there.}
% \label{fig:A2-depth}
% \vspace{-10pt}
% \end{wrapfigure}
\begin{figure}[t]
\centering
\includegraphics[width=0.44\linewidth]{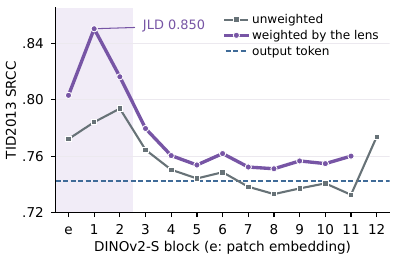}
\vspace{-4pt}
\caption{TID2013 \srcc{} of patch-token distances at each depth of DINOv2-S, raw or through the lens fitted there.}
\label{fig:A2-depth}
\end{figure}
% ---- v23.7 two-panel float (restore by uncommenting and removing the wrapfigure above) ----
% \begin{figure}[t]
% \centering
% \includegraphics[width=\linewidth]{v233/A_motivation.pdf}
% \vspace{-14pt}
% % \caption{Measure early, weigh by late. Agreement with human scores on all 3{,}000 TID2013 pairs for distances between patch tokens of one frozen DINOv2-S.
% % (a)~At each depth, tokens compared in all directions alike (grey) or through the lens fitted at that depth (purple); block 12 tokens no longer reach the output, so it has no lens.
% % (b)~The same distances as both images are resized: only the lens-weighted early tokens stay accurate.}
% \caption{TID2013 \srcc{} of patch-token distances in one frozen DINOv2-S. (a)~By depth, isotropic (grey) or through the lens fitted at that depth (purple). (b)~As both images are resized.}
% \label{fig:A2-depth}
% \vspace{-6pt}
% \end{figure}

\name{} compares two aligned images using one frozen encoder and one fixed linear projection.
We use DINOv2-S (ViT-S/14), pretrained without quality judgements \citep{oquab2024dinov2}.
For an RGB image whose sides are multiples of the patch size $p=14$, the encoder produces $N=HW/196$ patch positions.
At each position $t$, we read the patch vector $\tok{t}\in\mathbb{R}^{384}$ from the residual stream after block 1; block 0 denotes the linear patch embedding. The final layer-normalised class token $\out\in\mathbb{R}^{384}$ defines the sensitivity used to fit the lens and supplies the whole-image output term. The distance has no fixed input size.
Calibration uses $224\times224$ crops, corresponding to a $16\times16$ patch grid.
At evaluation, images are centre-cropped to the largest multiple of 14 pixels in each dimension. For example, a $512\times384$ image becomes $504\times378$, or $36\times27$ patches.
DINOv2 interpolates its positional embeddings to each grid, allowing the same fitted lens to operate across image sizes (Section~\ref{sec:resolution}). %Changing the encoder or output target requires only a refit (Appendices~\ref{app:encoders} and~\ref{app:quality-transfer}).

Figure~\ref{fig:A2-depth} repeats the fit at every encoder depth. The lens-weighted distance peaks at block~1 with correlation 0.850, compared with 0.803 at the patch embedding, 0.816 after block~2, and at most 0.780 from block~3 onward. The same choice also gives the strongest resolution stability, from a 256 to 512 pixel short side, the block-1 lens changes only from 0.850 to 0.845, whereas unweighted block-1 features fall from 0.784 to 0.656 and final patch features from 0.762 to 0.722.

\subsection{Calibrating the Lens}\label{app:fit}

\begin{figure}[t]
\centering
\includegraphics[width=\linewidth]{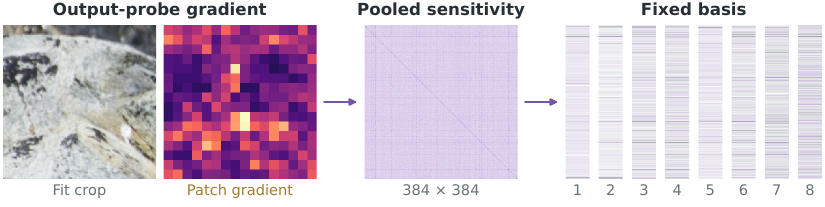}
\caption{One calibration step. Probe gradients at every patch (left) are pooled into the $384\times384$ sensitivity matrix (middle), whose leading eigenvectors form the lens (right, first eight of 64).}
\label{fig:fit}
\end{figure}

The lens should preserve feature directions that affect the encoder output rather than directions that merely vary across images.
For patch $t$, the Jacobian
$J_t(x)=\partial\out(x)/\partial\tok{t}(x)$
maps a small token displacement $\delta$ to the first-order output change $J_t(x)\delta$.
Averaging its squared response over images and positions gives the sensitivity matrix of Equation~\ref{eq:sensitivity},
\begin{equation*}
\M=\mathbb{E}_{x,t}\bigl[J_t(x)^\top J_t(x)\bigr],
\qquad
\delta^\top\M\delta
=
\mathbb{E}_{x,t}\|J_t(x)\delta\|^2.
\end{equation*}
Squaring before averaging prevents responses of opposite sign from cancelling. We omit cross-patch products, so $\M$ describes the sensitivity of a single patch intervention. Explicitly forming $J_t$ would require one backward pass per output coordinate. Instead, we draw a Gaussian probe $v\sim\mathcal N(0,I)$ at the output. One backward pass of $v^\top\out$ gives
$g_t=J_t^\top v$
at every patch, with
$\mathbb E_v[g_tg_t^\top]=J_t^\top J_t$.
For $n$ images, $C$ crops, and $P$ probes, we estimate
\begin{equation}
\widehat{\M}
=
\frac{1}{nCP}\sum_{i,c,p}\frac{1}{N}\sum_t
g_{icpt}g_{icpt}^{\top},
\qquad
g_{icpt}=J_t(x_{ic})^{\top}v_{icp},
\quad
v_{icp}\sim\mathcal{N}(0,I),
\label{eq:app_fit_estimator}
\end{equation}
which is unbiased, with error decreasing as $(nCP)^{-1/2}$ (Proposition~\ref{prop:estimation}).
We make $\widehat{\M}$ symmetrical and retain its 64 leading eigenvectors as the orthonormal lens $\U_k$, without eigenvalue weighting or whitening (Algorithm~\ref{alg:A2-jld}). Calibration uses 100 unlabelled DIV2K validation images \citep{agustsson2017div2k}, disjoint from all quality datasets used for evaluation. 
%We sort the filenames, draw four random $224\times224$ crops per image with seed 0, and use eight probes per crop, for 3{,}200 probe gradients in total.
The calibration crops are not chroma-filtered, and only activation gradients are computed; the encoder remains frozen.
The complete fit takes 34.53\,s on one A100 GPU with 327.8\,MB peak GPU memory. It uses no image pairs, distortions, or human labels. %Human-rated development data are used only to choose the block, rank, and output weight (Appendix~\ref{app:protocols}).
Once fitted, the lens is reused for every comparison with no gradients. Figure~\ref{fig:fit} visualizes one calibration crop, a single output produces gradients at all patches, whose outer products are accumulated into the $384\times384$ sensitivity matrix. Its leading eigenvectors mix many feature coordinates rather than selecting individual channels.

\Needspace{14\baselineskip}
\subsection{Feature Rank}\label{app:md-rank}
% v23.3: the calibration diagram moved to the main text (fig:A-lens); this float keeps only the spectrum.
% v23.5 (P1): very small wrapfigure (was a 0.55\linewidth float).
% \begin{figure}[t]
% \centering
% \includegraphics[width=0.55\linewidth]{v22/app_spectrum.pdf}
\begin{wrapfigure}[11]{l}{0.33\linewidth}
\centering\vspace{-12pt}
\includegraphics[width=\linewidth]{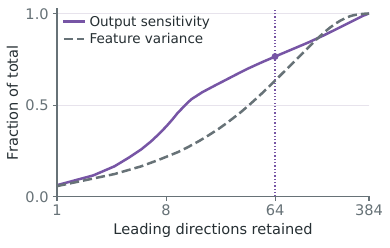}
% Source: scripts/v22/app_diag.py; figures/v18/appB_spectrum_numbers.json, retained_fraction.
\captionsetup{font=footnotesize,skip=2pt}
\caption{Cumulative share of $\operatorname{tr}\M$ (violet) and of feature variance (grey).}
% \caption{Cumulative eigenvalue share of output sensitivity and feature variance, each normalised by its own trace.
% In the final CLS-only block-1 fit, the leading 64 of 384 directions retain 76.5\% of squared sensitivity.}
\label{fig:appB-spectrum}
\vspace{-10pt}
\end{wrapfigure}
% \end{figure}

Figure~\ref{fig:appB-spectrum} compares the cumulative spectrum of $\M$ with that of the feature covariance.
Sensitivity is substantially more concentrated, the leading 64 of 384 directions contain 76.5\% of $\operatorname{tr}\M$, so the lens discards rest of feature space while retaining over three quarters of the encoder's squared sensitivity. The sensitivity and covariance spectra also select different subspaces.
For a diagnostic using the class token and mean patch token as output, the block-1 lens retains 78.7\% of sensitivity but only 33.7\% of feature variance. Its overlap with the 64 leading PCA directions is 0.358 at block~1 and 0.060 at block~12, compared with 0.167 in expectation for a random 64-dimensional subspace.

% The rank follows a simple development rule, choose the smallest $k$ within the noise of the best development correlation (Appendix~\ref{app:sweep}).
\begin{wrapfigure}[17]{r}{2.15in}
\centering\vspace{-12pt}
\includegraphics[width=2.1in]{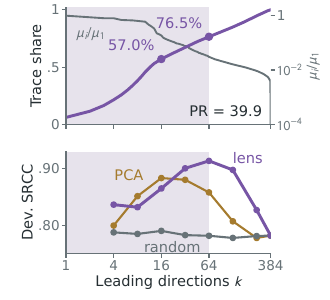}
% Source: figures/v22/fig_geo_spectrum_numbers.json (eigenvalues from lens/jld_dinov2_s14_block1_k64.npz; rank sweep from
% figures/v22/app_sweep_numbers.json display_data.rank_rows); script scripts/v22/geo_spectrum.py.
\captionsetup{font=footnotesize,skip=3pt}
\caption{Spectrum of $\M$ (top; 64 kept directions carry 76.5\% of the trace; PR: participation ratio) and development SRCC against rank $k$.}
\label{fig:geo-spectrum}
\vspace{-10pt}
\end{wrapfigure}

Development correlation peaks at 0.913 for $k=64$, falls to 0.782 when all 384 directions are retained, and reaches at most 0.883 for PCA subspaces at any rank (Figure~\ref{fig:geo-spectrum}).
Thus, the gain comes primarily from which directions are retained rather than dimensionality reduction alone.
The released lens is flat, all retained directions receive equal weight rather than being weighted by $\mu_i$.
Table~\ref{tab:flat-ablation} compares this choice with eigenvalue-weighted and full Mahalanobis alternatives.
Although the flat and eigenvalue-weighted lens terms rank pairs similarly, the flat projection agrees better with human scores on all four datasets. The full Mahalanobis metric $\M$ also performs worse, while raw block-1 features are substantially weaker.

\begin{table}[t]
\centering\small
\caption{Lens form. \srcc{} with MOS for the lens term under each choice of $\mathbf Q$ in $\|\mathbf Q^{1/2}\delta\|$, together with the development mean (CSIQ, KADID dev.) after adding the output term.}
\label{tab:flat-ablation}
\begin{tabular}{@{}lccccc@{}}
\toprule
$\mathbf Q$ in the lens term & CSIQ & KADID dev. & TID2013 & KADID test & Dev. mean, with output \\
\midrule
$\U_k\U_k^\top$ (flat, released) & \textbf{0.956} & \textbf{0.876} & \textbf{0.874} & \textbf{0.870} & \textbf{0.932} \\
$\M$ (Mahalanobis) & 0.946 & 0.848 & 0.859 & 0.845 & 0.926 \\
$\Pi_k\M\Pi_k$ ($\sqrt{\mu}$-weighted lens) & 0.934 & 0.836 & 0.859 & 0.830 & 0.926 \\
$I$ (raw block-1 features) & 0.836 & 0.710 & 0.670 & 0.718 & 0.859 \\
\bottomrule
\end{tabular}
\end{table}

\subsection{Lens Sensitivity}\label{app:visualisation}

Output sensitivity is a substantially better ruler than raw feature energy. Across 153 perturbations of three reference images, RMS lens displacement has Spearman correlation 0.643 with the change of the final class token, compared with 0.147 for raw block-1 displacement. The lens therefore retains the component of an early feature change that propagates to the encoder output. The lens also has a direct pixel interpretation at block~0.
There the patch features are a linear map $W$ of a $3\times14\times14$ RGB patch, so each lens direction $u$ acts as a matched filter $W^\top u$ on pixel changes; \namelin{} uses this construction directly. Biases and positional embeddings cancel when taking differences. Eigenvector signs are arbitrary, and nearly tied directions may rotate between fits, but sign changes and rotations within the retained subspace leave every distance unchanged.

\subsection{Chroma pre-filter}\label{app:md-prefilter}

Human sensitivity falls faster with spatial frequency for colour than for luminance \citep{campbell1968csf,mullen1985}.
The unfiltered lens does not fully reproduce this behaviour, at the pixel Nyquist frequency, its red--green and blue--yellow responses retain 0.208 and 0.274 of their peaks, respectively, where human chromatic sensitivity is close to zero (Appendix~\ref{app:csf}). Following S-CIELAB \citep{zhang1996scielab}, we therefore apply a fixed pre-filter that blurs chroma while preserving luminance,
\begin{equation}
\pre x
=
T^{-1}\bigl(Y,\ G_\sigma*C_b,\ G_\sigma*C_r\bigr),
\qquad
(Y,C_b,C_r)=T(x),
\qquad
\sigma=2\ \text{pixels},
\label{eq:prefilter}
\end{equation}
where $T$ maps RGB to YCbCr and $G_\sigma$ is a Gaussian with reflect padding. The result is clipped to $[0,1]$, and both filtered images are then evaluated on the same patch grid. %The width $\sigma$ was selected on CSIQ and KADID-10k development data after a diagnostic using TID2013 gratings, so TID2013 is not a held-out test for this correction (Appendix~\ref{app:csf}). The lens itself is fitted on unfiltered crops.

\subsection{Pooling and Output term}\label{app:md-pooling}

Per-patch lens lengths are pooled quadratically.
Under independent equal-variance Gaussian noise, an ideal observer combines per-patch $d'$ in this form, corresponding to exponent-two probability summation \citep{quick1974}. Quadratic pooling also makes the lens term a Euclidean distance after a fixed feature map and therefore a pseudometric (Proposition~\ref{prop:metric}). The output term is one half of the cosine distance between the final layer-normalised class tokens.
It complements the local lens term with a whole-image comparison, raising CSIQ correlation from 0.956 to 0.971.
Its weight of $1/2$ was selected on development data and is not sensitive, weights from 0.25 to 1 keep CSIQ between 0.968 and 0.971 and TID2013 between 0.872 and 0.877 (Appendix~\ref{app:output-weight}).

\subsection{Variants: \namefast{}, \namelin{} and \namevideo{}}\label{app:md-variants}

\namefast{} retains the pre-filter and lens term, removes the output term, and stops the encoder after block~1.
This reduces the cost per pair from 118.5 GFLOPs and 12.2\,ms to 10.7 GFLOPs and 2.5\,ms on an A100, while making the complete distance a pseudometric.

\namelin{} moves one step earlier and fits the lens on the block-0 patch embedding.
The resulting distance is the RMS response of 64 fixed $14\times14$ RGB filters to the filtered difference image, and its behaviour under grid refinement can be analysed directly (Theorem~\ref{thm:grid-lin}).

\namevideo{} applies the same construction to a frozen VideoMAE-Base fine-tuned on Kinetics-400 \citep{tong2022videomae}.
It uses a 16-dimensional lens at block~2, fitted to the class logits using 50 unlabelled clips and eight probes per clip.
Squared projected differences between aligned space-time tokens are averaged over tokens and eight evenly spaced 16-frame clips before taking the square root.
The video variant uses neither the chroma pre-filter nor an output term (Appendix~\ref{app:video}).

\subsection{Fit stability}\label{app:md-stability}

\begin{table}[t]
\centering\small\setlength{\tabcolsep}{5pt}
\caption{Fit stability. Subspace overlap with the 100-image lens and development mean \srcc{} (CSIQ, KADID dev.); $\Delta$ is relative to the 100-image fit.}
\label{tab:appB_stability}
\begin{tabular}{@{}lccr@{\hspace{18pt}}lccr@{}}
\toprule
Fit corpus & Overlap & \srcc{} & $\Delta\times10^{4}$ & Disjoint half & Overlap & \srcc{} & $\Delta\times10^{4}$ \\
\midrule
10 images & 0.877 & 0.913 & $-2.5$ & Half A, seed 0 & 0.963 & 0.913 & $-2.7$ \\
25 images & 0.916 & 0.914 & $+3.7$ & Half A, seed 1 & 0.935 & 0.914 & $+1.6$ \\
50 images & 0.963 & 0.913 & $-2.7$ & Half B, seed 0 & 0.943 & 0.914 & $+4.0$ \\
100 images & 1.000 & 0.913 & $0.0$ & Half B, seed 1 & 0.951 & 0.913 & $-0.7$ \\
\bottomrule
\end{tabular}
\vspace{-4pt}
\end{table}

We test whether calibration depends strongly on the choice of unlabelled corpus.
For two lenses $\U$ and $\widetilde{\U}$, we measure subspace overlap as
$\|\U^\top\widetilde{\U}\|_F^2/k$, which equals 1 for identical subspaces and has expectation $k/384$ for random ones.
Ten calibration images already give overlap 0.877 with the 100-image lens and change either development correlation by less than 0.002.
Two disjoint 50-image halves, each fitted with two seeds, overlap the full fit by 0.935--0.963 and remain within $4\times10^{-4}$ of its development correlation.
Similarly, two probes per crop give overlap 0.967 and four give 0.990; using 16 probes instead of eight, or crop sizes of 112 or 448 instead of 224, changes either correlation by less than 0.002.
The output target matters more than the calibration sample: targeting the next block rather than the final output gives overlap 0.510, CSIQ correlation 0.919, and KADID development correlation 0.842.
These stability diagnostics used the preliminary target combining the class token and mean patch token; Appendix~\ref{app:sweep} reports the block, rank, and output-weight sweeps for the final configuration.
\phantomsection\label{app:refit}
To refit the lens for a new domain, collect roughly 100 representative unlabelled images that are disjoint from evaluation data, while keeping the encoder, block, output target, and normalisation fixed.
Save the fitted basis together with the corpus manifest and random seed, and compare independent fits through their subspace overlap.
Changing the encoder or output target requires a refit; selecting the rank or output weight from human ratings requires separate development and test data (Appendix~\ref{app:usage}).

\Needspace{10\baselineskip}
\subsection{Implementation}\label{app:md-cost}

% The appendix latency figure is omitted because Figure~\ref{fig:B-robust}c
% already shows the speed--accuracy trade-off; Table~\ref{tab:app_cost}
% provides the complete timing protocol and measurements.

Scoring one pair requires two forward passes, one $384\times64$ projection per patch, and one cosine computation, with no gradients. The implementation center crops each image to the patch grid, applies the chroma pre-filter, and returns the distance. %A differentiable interface is also provided, allowing \name{} to be used in optimisation-based stress tests (Appendix~\ref{app:validation}).
The released lens is a $384\times64$ array stored together with its sensitivity matrix, feature covariance, and calibration metadata; inference requires only the lens and encoder. As shown in Figure~\ref{fig:B-robust}c, \name{} and \namefast{} lie on the measured speed--accuracy Pareto frontier.
\namefast{} is 4.3 times faster than LPIPS-VGG on an A100 and 14 times faster on twelve CPU threads because it stops after the first of twelve encoder blocks.
Full \name{} evaluates the remaining blocks only to obtain the final class token.
\name{} timings exclude the pre-filter and output cosine.% arXiv: dangling "gives the complete timing protocol" clause removed (its table is not in this version)

\FloatBarrier
\section{Proofs}\label{app:proofs}

We assume throughout that the encoder is frozen and that its patch features $\tok{t}$ and output $\out$ are differentiable in the input pixels, twice where required below. Images lie in the compact set $[0,1]^{3HW}$ for a fixed patch grid, and patch positions are sampled uniformly within each image.
These assumptions keep the relevant Jacobians bounded and the Taylor remainders uniform; DINOv2, built from GELU activations, layer normalisation and softmax attention, satisfies them.
The pre-filter $\pre$ is affine away from saturated pixels (Appendix~\ref{app:md-prefilter}), and statements involving $\pre x$ are understood on this region. As in Section~\ref{sec:method},
\[
J_t(x)=\frac{\partial\out(x)}{\partial\tok{t}(x)},
\qquad
\M=\mathbb{E}_{x,t}[J_t^\top J_t]
=\sum_{j=1}^{d}\mu_j u_j u_j^\top,
\]
with $\mu_1\ge\cdots\ge\mu_d\ge0$.
The lens is $\U_k=[u_1\cdots u_k]$ and $\Pi_k=\U_k\U_k^\top$.

\subsection{The sensitivity matrix and the lens}

The quadratic form
\[
\delta^\top\M\delta
=
\mathbb{E}_{x,t}\|J_t\delta\|^2
\]
is the Euclidean metric of output space pulled back through the map $\tok{t}\mapsto\out$, with the other intermediate tokens fixed, and then averaged over images and positions.
If the output is observed through isotropic Gaussian noise with variance $\sigma^2$, then $J_t^\top J_t/\sigma^2$ is the Fisher information for $\tok{t}$.
Thus, to first order, $\sqrt{\delta^\top\M\delta}/\sigma$ is the RMS detectability $d'$ of the displacement $\delta$.

\begin{proposition}[Optimal subspace]\label{prop:optimality}
Among all rank-$k$ orthonormal bases $V$, the lens $\U_k$ maximises
$\operatorname{tr}(V^\top\M V)$.
For any displacement $\delta$ chosen independently of the image and position,
\begin{equation}
\mu_k\|\U_k^\top\delta\|^2
\le
\mathbb{E}_{x,t}\|J_t(x)\delta\|^2
=
\delta^\top\M\delta
\le
\mu_1\|\U_k^\top\delta\|^2
+
\mu_{k+1}\|(I-\Pi_k)\delta\|^2.
\label{eq:app_complement}
\end{equation}
\end{proposition}

\begin{proof}
The trace statement follows from Ky Fan's maximum principle.
Expanding in the eigenbasis gives
\[
\delta^\top\M\delta
=
\sum_j\mu_j\langle u_j,\delta\rangle^2.
\]
For $j\le k$, $\mu_j\in[\mu_k,\mu_1]$, while for $j>k$, $\mu_j\in[0,\mu_{k+1}]$, which gives Equation~\ref{eq:app_complement}.
\end{proof}

For the released lens, $\mu_{65}/\mu_1=0.030$.
Hence an equal-norm displacement lying entirely outside the lens has RMS sensitivity at most
$\sqrt{\mu_{65}/\mu_1}=0.173$
of the leading eigendirection.
The worst-case variation within the retained subspace is much larger, since $\mu_1/\mu_{64}=32.8$, but real distortions occupy a much narrower range.
Across all reference--distortion pairs in CSIQ, KADID-10k and TID2013, the pooled ratio
\[
\gamma(\delta)
=
\frac{(\delta^\top\M\delta)^{1/2}}
{\|\U_k^\top\delta\|}
\]
has a 90th-to-10th percentile ratio of only 1.26--1.50, compared with the worst-case factor $\sqrt{32.8}=5.7$.
This helps explain why the flat lens and the $\M$-weighted form rank image pairs similarly (Table~\ref{tab:flat-ablation}).

\begin{lemma}[Captured sensitivity is stable]\label{lem:capture}
Let $A$ and $B$ be symmetric matrices with top-$k$ eigenbases $\U_A$ and $\U_B$, and let $\|\cdot\|_{(k)}$ denote the sum of the $k$ largest singular values.
Then
\[
0
\le
\operatorname{tr}(\U_A^\top A\U_A)
-
\operatorname{tr}(\U_B^\top A\U_B)
\le
2\|A-B\|_{(k)}.
\]
\end{lemma}

\begin{proof}
The lower bound follows from Ky Fan's principle for $A$.
For the upper bound, write the gap as
\[
\operatorname{tr}(\U_A^\top(A-B)\U_A)
+
\Big[
\operatorname{tr}(\U_A^\top B\U_A)
-
\operatorname{tr}(\U_B^\top B\U_B)
\Big]
+
\operatorname{tr}(\U_B^\top(B-A)\U_B).
\]
The middle term is non-positive by Ky Fan's principle for $B$, and each outer term is at most $\|A-B\|_{(k)}$.
\end{proof}

Importantly, this bound requires no eigengap.
The sensitivity captured by the estimated subspace can therefore remain stable even when individual eigenvectors are not.

\begin{proposition}[Random probes]\label{prop:estimation}
Let $x_1,\ldots,x_n$ denote independent calibration inputs, and draw $P$ probes
$v\sim\mathcal N(0,I_m)$ independently for each input.
Let $\widehat{\M}$ average
\[
\frac{1}{N}\sum_t(J_t^\top v)(J_t^\top v)^\top
\]
over the $nP$ input--probe pairs, and let $\M_n$ be the corresponding average of the exact per-input sensitivity tensors.

Then:
\begin{enumerate}
\item $\widehat{\M}$ is unbiased.
\item For fixed inputs, its relative RMS probe error is at most
\[
\beta\sqrt{\frac{2r}{nP}},
\qquad
r=\frac{(\operatorname{tr}\M_n)^2}{\|\M_n\|_F^2}\le d,
\]
where $\beta\ge1$ measures variation in the per-input traces.
\item Over random inputs,
\[
\mathbb{E}\|\widehat{\M}-\M\|_F^2
=
\frac{V_{\rm content}}{n}
+
\frac{V_{\rm probe}}{nP},
\]
with both terms finite.
\item If $\widehat{\U}$ is the top-$k$ eigenspace of $\widehat{\M}$, then
\[
\operatorname{tr}(\widehat{\U}^\top\M\widehat{\U})
\ge
\sum_{j\le k}\mu_j
-
2\|\widehat{\M}-\M\|_{(k)}.
\]
\end{enumerate}
\end{proposition}

\begin{proof}
For (i),
\[
\mathbb{E}_v[(J_t^\top v)(J_t^\top v)^\top]
=
J_t^\top\mathbb{E}[vv^\top]J_t
=
J_t^\top J_t.
\]
Moreover, one backward pass of $v^\top\out$ returns $J_t^\top v$ simultaneously at every patch.

For (ii), independent probes make the centred errors uncorrelated.
Isserlis's theorem expresses their second moment through products $J_t^\top J_s$, and Cauchy--Schwarz bounds the resulting variance by $2(\operatorname{tr}\M_i)^2$.
Averaging over $nP$ samples and normalising by $\|\M_n\|_F^2$ gives the stated bound.

Part (iii) follows from the law of total variance.
Part (iv) follows directly from Lemma~\ref{lem:capture} with $A=\M$ and $B=\widehat{\M}$.
\end{proof}

For the released calibration, the $100$ source images provide $n=400$ random crops, with $P=8$ probes per crop.
Here $r=39.9$ and $\beta^2=1.58$, giving a probe-error bound of approximately $0.20\|\M\|_F$; image sampling contributes approximately $0.14\|\M\|_F$.
The measured variation is smaller: two independent refits with fresh crops and probes differ from the released tensor by 8.4\% of $\|\M\|_F$, their lenses have subspace overlap 0.962, and each loses at most 0.35\% of the sensitivity captured by the other's fitted lens.

\subsection{The lens term as a distance}

\begin{proposition}[Early-feature distance]\label{prop:metric}
The lens term $\dist_{\rm lens}$ of Equation~\ref{eq:lens-distance} is a pseudometric on images evaluated on a common patch grid.
\end{proposition}

\begin{proof}
With the pre-filter, lens, and patch grid fixed,
\[
\dist_{\rm lens}(x,y)
=
\|\Phi(x)-\Phi(y)\|_2,
\]
where
\[
\Phi(x)
=
N^{-1/2}
\operatorname{vec}
\bigl(
\U_k^\top\tok{1}(\pre x),
\ldots,
\U_k^\top\tok{N}(\pre x)
\bigr).
\]
Euclidean distance after a fixed map is non-negative, symmetric, satisfies the triangle inequality, and vanishes whenever the mapped representations agree.
Hence $\dist_{\rm lens}$ is a pseudometric.
\end{proof}

Zero distance means that every projected patch feature agrees, giving a pair of lens metamers.
The result also applies to \namefast{} and \namelin{}.
A crop chosen independently for each pair would not define a single map $\Phi$ across a triangle and would therefore not inherit this guarantee.

\begin{proposition}[Local quadratic form]\label{prop:local}
Let
\[
A_t(x)
=
\frac{\partial\tok{t}(\pre x)}{\partial x}.
\]
For a small pixel perturbation $\varepsilon$,
\[
\dist_{\rm lens}(x,x+\varepsilon)^2
=
\varepsilon^\top G(x)\varepsilon
+
o(\|\varepsilon\|^2),
\]
where
\[
G(x)
=
\frac1N\sum_t
A_t(x)^\top
\U_k\U_k^\top
A_t(x).
\]
The matrix $G(x)$ is positive semidefinite and has rank at most $kN$.
If $\out(\pre x)\neq0$, the output term is $O(\|\varepsilon\|^2)$.
\end{proposition}

\begin{proof}
Each patch feature has the expansion
\[
\tok{t}(\pre(x+\varepsilon))
-
\tok{t}(\pre x)
=
A_t(x)\varepsilon
+
o(\|\varepsilon\|).
\]
Projecting, squaring, and averaging over patches gives the stated quadratic form.
Each summand has rank at most $k$, hence $\operatorname{rank}G(x)\le kN$.

Let $\hat z=\out/\|\out\|$.
The output term can be written as
\[
\frac14
\|
\hat z(\pre(x+\varepsilon))
-
\hat z(\pre x)
\|^2,
\]
which is the square of a first-order change and is therefore $O(\|\varepsilon\|^2)$.
\end{proof}

Thus, $G(x)$ is the pullback of the fixed feature-space form $\U_k\U_k^\top$ into pixel space.
The lens is fixed globally, but its local pixel geometry depends on the image.
The lens term is the chordal distance induced by this feature map, not a geodesic length.

Moreover,
\[
\varepsilon^\top G(x)\varepsilon
=
\frac1N\sum_t
\|\U_k^\top A_t(x)\varepsilon\|^2,
\]
so the first-order metamers are
\[
\ker G(x)
=
\{
\varepsilon:
\U_k^\top A_t(x)\varepsilon=0
\ \forall t
\}.
\]
Their dimension is at least $3HW-kN$.
Likewise, in token space, any $\delta\perp\operatorname{span}\U_k$ is invisible to the lens and satisfies
\[
\delta^\top\M\delta
\le
\mu_{k+1}\|\delta\|^2
\]
by Equation~\ref{eq:app_complement}.

\subsection{Refining the pixel grid}

For \namelin{}, define the per-patch lens energy
\[
\rho_N(t)
=
\|
\U_k^\top
(
\tok{t}(\pre x)-\tok{t}(\pre y)
)
\|^2,
\]
so that $\dist_{\rm lens}^2$ is its average over the patch grid.
At block~0, the patch token is the affine map
\[
E\operatorname{vec}(P_t x)+e_t,
\]
and the offset cancels between two images.
Hence
\[
\rho_N(t)
=
\|
W\operatorname{vec}(P_t(\pre x-\pre y))
\|^2,
\qquad
W=\U_k^\top E,
\]
which is the squared response of $k$ fixed $14\times14$ RGB filters.

\begin{theorem}[Grid consistency of \namelin{}]\label{thm:grid-lin}
Let two images on the unit square have an $L$-Lipschitz difference $\delta=x-y$, and assume that the pre-filter does not saturate.
Sample them at $n=14g$ pixels per unit length, giving $N=g^2$ patches.
Let $W_{\rm dc}\in\mathbb{R}^{k\times3}$ contain the DC gains of the filters and define
\[
\rho(u)=\|W_{\rm dc}\delta(u)\|^2.
\]
Then
\begin{equation}
\left|
\dist_{{\rm lens},N}^2
-
\int\rho(u)\,du
\right|
\le
\frac{c_1L}{n}
+
\frac{c_2L^2}{n^2}
=
O(N^{-1/2}),
\label{eq:app_grid_lin}
\end{equation}
where $c_1$ and $c_2$ depend only on
$\|\delta\|_\infty$,
$\|W\|_2$,
$\|W_{\rm dc}\|_2$,
and the colour transform.
For band-limited images with angular frequency at most $B$, $L\le\sqrt6B$.
\end{theorem}

\begin{proof}
Every pixel contributing to patch $t$, including those reached through the chroma filter, lies within $r_p=13\sqrt2$ pixels of the patch centre $c_t$.
Lipschitz continuity therefore gives
\[
\rho_N(t)
=
\|W_{\rm dc}\delta(c_t)+b_t\|^2,
\qquad
\|b_t\|
\le
\frac{14\kappa Lr_p\|W\|_2}{n},
\]
where
$\kappa=\sqrt3\|T\|_2\|T^{-1}\|_2$.
Consequently,
\[
|\rho_N(t)-\rho(c_t)|
\le
2\|W_{\rm dc}\|_2\|\delta\|_\infty\|b_t\|
+
\|b_t\|^2.
\]
The function $\rho$ is
$2\|W_{\rm dc}\|_2^2\|\delta\|_\infty L$-Lipschitz, so replacing its value at each patch centre by the corresponding cell average contributes another $O(n^{-1})$ term.
Summing the two errors yields Equation~\ref{eq:app_grid_lin}.
The band-limited statement follows from Bernstein's inequality applied channel-wise.
\end{proof}

The limiting form depends only on the filters' DC gains, which contain just 0.28\% of their total energy.
Thus, the learned filters are predominantly band-pass edge and texture detectors.
The experimentally relevant resolutions lie well before this asymptotic DC limit; the resolution stability measured in Section~\ref{sec:resolution} is therefore a property of this finite-resolution band-pass regime.
For block-1 features, one attention step intervenes between pixels and tokens, so resolution stability is measured empirically rather than covered by Theorem~\ref{thm:grid-lin}.

\subsection{One lens for every image}

For one image, define
\[
\M_x
=
\frac1N\sum_t J_t(x)^\top J_t(x),
\qquad
\M=\mathbb{E}_x\M_x,
\]
and let $\U_x$ denote the top-$k$ eigenspace of $\M_x$.

\begin{proposition}[One lens for every image]\label{prop:global}
\begin{enumerate}
\item The global lens $\U_k$ maximises
\[
\mathbb{E}_x
\operatorname{tr}(V^\top\M_xV)
\]
over all rank-$k$ orthonormal bases $V$.
\item For every image,
\[
\operatorname{tr}(\U_x^\top\M_x\U_x)
-
\operatorname{tr}(\U_k^\top\M_x\U_k)
\le
2\|\M_x-\M\|_{(k)}.
\]
\end{enumerate}
\end{proposition}

\begin{proof}
For (i), linearity gives
\[
\mathbb{E}_x
\operatorname{tr}(V^\top\M_xV)
=
\operatorname{tr}(V^\top\M V),
\]
which is maximised by $\U_k$ by Proposition~\ref{prop:optimality}.
Part (ii) follows from Lemma~\ref{lem:capture} with $A=\M_x$ and $B=\M$.
\end{proof}

Thus, the global lens is the rank-$k$ subspace that maximises sensitivity on average over images, while its loss relative to an image-specific optimum is controlled directly by the difference $\M_x-\M$.
Empirically, the global lens gives a median image-wise sensitivity loss of only 5.0\%, consistent with the strong fit stability reported in Appendix~\ref{app:md-stability}.

\subsection{The output term and the triangle inequality}

For unit vectors $a$ and $b$,
\[
1-\cos(a,b)
=
\frac12\|a-b\|^2.
\]
This is a squared Euclidean distance and does not generally satisfy the triangle inequality.
For example, three unit vectors separated by angles $0$, $\theta$, and $2\theta$, with $0<\theta<\pi/2$, satisfy
\[
1-\cos(2\theta)
>
2(1-\cos\theta).
\]
Thus, the lens term is a pseudometric, but the full \name{} distance is not guaranteed to be one.
Section~\ref{sec:geometry} measures how rarely the full distance violates the triangle inequality in practice.
A single Euclidean norm over concatenated lens features and a suitably scaled unit output would restore the triangle inequality and gives similar rankings (Appendix~\ref{app:output-weight}).

\FloatBarrier
\section{Usage and reproducibility}\label{app:usage}
The released package scores a pair of image paths, arrays, or tensors in three lines.
It centre-crops each image to a multiple of 14 pixels, applies the chroma pre-filter, and evaluates Equation~\ref{eq:lens-distance}.
A differentiable interface is also provided.

\begin{lstlisting}[style=jldcode]
from jld import JLD

jld = JLD.pretrained("full")    # use "fast" for JLD-fast
d = jld(ref, dist)
\end{lstlisting}

The returned scalar is a distance, so smaller values indicate greater similarity.
The supplied lens can be used directly and requires no refitting.
% arXiv: listings below follow the public release API; the earlier research-code listings are kept commented.
% For a new domain, it can instead be recalibrated from unlabelled images without human ratings (Appendix~\ref{app:refit}). The following is the example code.
% Here we evaluate $224\times224$ centre crops; replacing \texttt{crop224} with \texttt{native} evaluates the full image after cropping it to the patch grid.
% 
% \begin{lstlisting}[style=jldcode]
% from jld.images import load_pair
% from jld.metric import JLD, JLDConfig, WorkspaceStats
% from jld.observers import make_observer
% 
% r = "/path/to/jlens-artifacts"
% s = (
%     "cache/workspace/dinov2_s14/"
%     "iso_report_cls_seed0_n100_p8/L1.npz"
% )
% 
% cfg = JLDConfig(
%     layers=(1,),
%     k=64,
%     kind="jlens",
%     form="euclid",
%     pooling="mean_sq",
%     chroma_sigma=2.0,
%     report_weight=0.5,
% )
% 
% enc = make_observer(
%     "dinov2_s14",
%     layers=[1],
%     report="cls",
%     device="cpu",
% )
% 
% metric = JLD(
%     enc,
%     {1: WorkspaceStats.load(f"{r}/{s}")},
%     cfg,
% )
% 
% data = f"{r}/data/tid2013/tid2013"
% x, y = load_pair(
%     f"{data}/reference_images/I04.BMP",
%     f"{data}/distorted_images/i04_10_4.bmp",
%     14,
%     "crop224",
% )
% 
% print(f"{metric(x, y):.9f}")
% \end{lstlisting}
For a new domain, it can instead be recalibrated from unlabelled images without human ratings (Appendix~\ref{app:refit}). Inputs may be file paths, arrays, or tensors in $[0,1]$, and are evaluated at native resolution. The same object returns the per-patch lens map and a differentiable distance for use as a loss.

\begin{lstlisting}[style=jldcode]
from jld import JLD

jld = JLD.pretrained("full")
ref, dist = "reference.png", "distorted.png"

d = jld(ref, dist)                    # distance, lower is more similar
local = jld.map(ref, dist)            # per-patch lens response
loss = jld.distance_tensor(ref, dist) # differentiable scalar
\end{lstlisting}

\Needspace{4\baselineskip}
\subsection*{Refitting the lens}

To fit a replacement lens, use representative unlabelled images that are disjoint from evaluation data.
The released configuration uses 100 images, four $224\times224$ crops per image, eight Gaussian probes per crop, and seed $0$.

% \begin{lstlisting}[style=jldcode]
% from pathlib import Path
% import numpy as np
% 
% from jld.lens import fit_broadcast, basis
% 
% paths = sorted(Path("calibration_images").glob("*.png"))[:100]
% assert len(paths) == 100
% 
% enc = make_observer(
%     "dinov2_s14",
%     layers=[1],
%     report="cls",
%     device="cpu",
% )
% 
% fit = fit_broadcast(
%     enc,
%     paths,
%     layers=[1],
%     n_crops=4,
%     crop=224,
%     n_probes=8,
%     probe="iso",
%     seed=0,
%     batch=8,
% )
% 
% fit[1].metadata["image_paths"] = [str(p) for p in paths]
% fit[1].save("calibration_L1.npz")
% 
% U = basis(fit[1], k=64, kind="jlens")
% np.save("lens64.npy", U)
% 
% metric = JLD(enc, {1: fit[1]}, cfg)
% \end{lstlisting}
\begin{lstlisting}[style=jldcode]
# fit a 64-direction lens on block 1 from a folder of unlabelled images
python scripts/fit_lens.py --images calib_images --out mine.npz

# score with the fitted lens
jld reference.png distorted.png --lens mine.npz
\end{lstlisting}

Fitting keeps all encoder parameters frozen but requires gradients with respect to intermediate activations, so it must not be run in inference mode. An accelerator reduces fitting time.
For a new domain, only the unlabelled calibration corpus should normally change. The block, rank, and output weight should be changed only with separate development evidence. The corpus manifest, random seed, encoder checkpoint, configuration, and fitted statistics should be stored with the lens.

\Needspace{4\baselineskip}
\section{Geometry of the distance}\label{app:geo-fingerprint}

This section reports additional measurements supporting Section~\ref{sec:geometry}, the feature energy retained by the lens, an additional brightness--noise plane, the stability of a single global lens across images, and the local geometry within an image.

\Needspace{4\baselineskip}
\subsection{Iso-distance contours: brightness against noise}\label{app:geo-contours2}

The third plane of Figure~\ref{fig:geo-contours} compares a uniform brightness offset with Gaussian noise on TID2013 image I23, with both directions normalised to equal RMS pixel energy. MSE is nearly isotropic on this plane (anisotropy 1.02), whereas \name{} stretches its contours by a factor of 2.78 along the brightness direction.
LPIPS-VGG and DISTS give similar anisotropies of 2.74 and 2.84, while the local pullback of \name{} predicts 3.14.
At fixed \name{} distance $0.08$, PSNR ranges from 19.6\,dB for a brightness change to 32.6\,dB for noise.
Conversely, on a ring of equal MSE at 29.7\,dB, \name{} varies 4.4-fold, from 0.026 for darkening to 0.114 for noise.
This provides another example in which equal pixel error corresponds to substantially different perceptual severity.

% \Needspace{14\baselineskip}% arXiv layout: replaced by the 18-line guard below
\Needspace{18\baselineskip}% arXiv layout: keep the wrapped figure on the page of its heading
\subsection{One lens serves every image}\label{app:geo-local-lens}

\begin{wrapfigure}[16]{r}{0.29\linewidth}
\centering\vspace{-12pt}
\includegraphics[width=\linewidth,trim=0 0 197 0,clip]{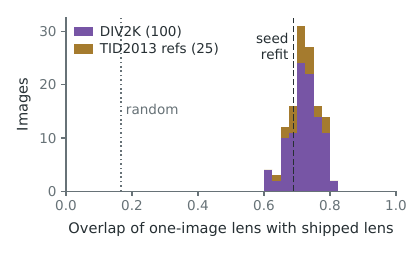}\\[1pt]
\includegraphics[width=\linewidth,trim=199 0 0 0,clip]{v22/app_geo_local_lens.pdf}
\captionsetup{font=footnotesize,skip=2pt}
\caption{Top: subspace overlap with the released lens. Bottom: sensitivity captured on an independent refit.}
\label{fig:geo-local-lens}
\vspace{-5pt}
\end{wrapfigure}

The released lens is one fixed matrix shared by every image.
To test whether this global choice loses important image-specific structure, we fit a separate lens to each of the 100 DIV2K calibration images and 25 TID2013 references, using the calibration procedure of Appendix~\ref{app:fit} restricted to that image.
Each image is fitted twice with independent randomness.
We compare two lenses $\U$ and $\widetilde{\U}$ through their subspace overlap
$\|\U^\top\widetilde{\U}\|_F^2/k$
(Appendix~\ref{app:md-stability}).
A one-image lens has median overlap 0.72 with the released lens, including 0.710 on the held-out TID2013 references.
For comparison, independently refitting the same image gives median overlap 0.69, while two random 64-dimensional subspaces have expected overlap 0.167 (Figure~\ref{fig:geo-local-lens}, top).

We also evaluate sensitivity on the independent refit for each image.
On this held-out estimate, the released lens captures more sensitivity than the first one-image fit for 103 of the 125 images (Figure~\ref{fig:geo-local-lens}, bottom).
Pooling the 100 DIV2K one-image estimates also recovers the released lens with subspace overlap 0.954.
These results indicate that much of the discrepancy between a one-image lens and the global lens comes from estimation noise rather than stable image-specific structure, supporting the use of one lens across images.

% Worker A (v23.3): moved from sections/geometry.tex (W12, v23.1) for the appendix geometry section.
% Integration: \input{sections/A_geo_field_app} at the end of sections/app_geometry.tex.
\subsection{The metric tensor field}\label{sec:geo-field}\label{app:geo-field}
% W12 (v23.1, new way to see the distance). Source: figures/v23/field_numbers.json (masking.*, orientation.*, colour.*, linearity_max_rel_dev_half_eps); scripts/v23/field_measure.py and field_figure.py.
% Owner: W12 (v23.1). Numbers: figures/v23/field_numbers.json; scripts scripts/v23/field_measure.py (metrics) and scripts/v23/field_figure.py (figure).
\begin{figure}[t]
\centering
\includegraphics[width=\linewidth]{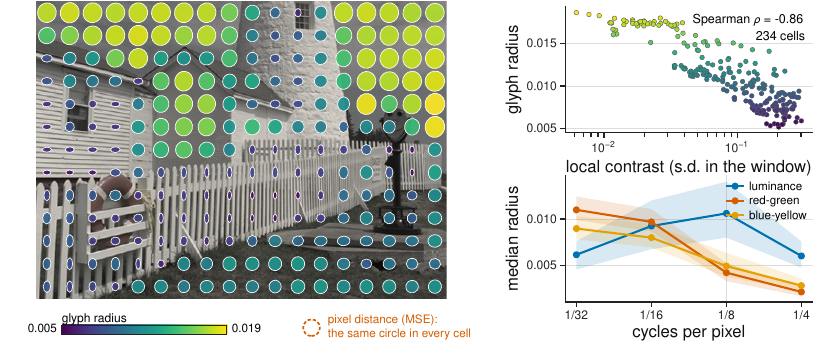}
% \caption{The metric tensor field of \name{} over TID2013 I19: one glyph per $28\times28$ pixel cell, the local form $G$ pulled back onto windowed gratings of 1/8 cycle per pixel.
% The glyph radius in a direction is the lens displacement per unit pixel energy of a grating whose wave vector points that way; a pixel distance draws the dashed circle in every cell.
% Right: glyph radius against local contrast, and the median radius per colour axis against spatial frequency (bands: interquartile range over the 234 cells).}

\caption{Local form $G$ of \name{} over TID2013 I19, one glyph per $28\times28$ cell (gratings of 1/8 cycle per pixel); the dashed circle is a pixel distance.
Right: glyph radius against local contrast, and median radius per colour axis against spatial frequency (bands: interquartile range).}
\label{fig:geo-field}
\end{figure}

% Figure~\ref{fig:geo-field} draws the local form $G$ itself over a photograph: in each of 234 cells of I19 we measure the lens displacement of windowed gratings along $x$, along $y$ and their sum, in luminance and in the two opponent colour directions, by finite differences of the released lensed tokens (halving the step changes a radius by at most 0.65\%).
% v23.6 P6: plain definition of the glyph (no reading tutorial).
In this section, we draw the local form $G$ over an image (Figure~\ref{fig:geo-field}).
Each glyph is $G$ in one $28\times28$ pixel cell: its radius along a direction is the lens displacement per unit pixel energy of a faint grating whose wave vector points in that direction.
A large glyph marks a cell where a small change already moves the distance, and an elongated glyph points along the grating wave direction to which it responds most.
We measure $G$ in each of the 234 cells of TID2013 I19 with windowed gratings along $x$, along $y$ and their sum, in luminance and in the two opponent colour directions, by finite differences of the released lensed tokens (halving the step changes a radius by at most 0.65\%).
Where a pixel distance draws one circle everywhere, \name{} draws a masking map: the glyph radius falls with the local contrast of the cell (Spearman $-0.86$; $-0.95$ against gradient energy), and the smoothest quarter of the image is 2.3 times as sensitive as the most textured quarter, the behaviour the grating study of Appendix~\ref{app:csf} finds on flat backgrounds.
The glyphs stretch along the local structure: in 91\% of the 86 cells with a clear gradient direction the stiff axis of $G$ lies more than $60^\circ$ from the gradient, so a grating that cuts across the fence is seen and one aligned with it is masked.
Colour depends on scale: colour sensitivity is 1.65 times luminance at 1/32 cycle per pixel and 0.44 times at 1/8, and red-green sensitivity falls 5.1-fold from 1/32 to 1/4 while luminance peaks at 1/8.

\section{Reading the distance on real images}\label{app:showcase}

This section shows \name{} on images that can be judged directly by the reader.
Following prior perceptual-distance work, we use two-alternative forced-choice (2AFC) examples as in LPIPS \citep{zhang2018lpips} and DreamSim \citep{fu2023dreamsim}, and group maximum differentiation (gMAD) \citep{ma2016gmad} as in DISTS \citep{ding2022dists}.
All examples come from the held-out KADID-10k test split \citep{lin2019kadid}, containing 65 references, 25 distortion families, and five levels per family.
The scores are produced by the released configuration used in Table~\ref{tab:main}, and each figure states its selection rule.

\subsection{2-AFC}

\begin{figure}[t]
\centering
\includegraphics[width=\linewidth]{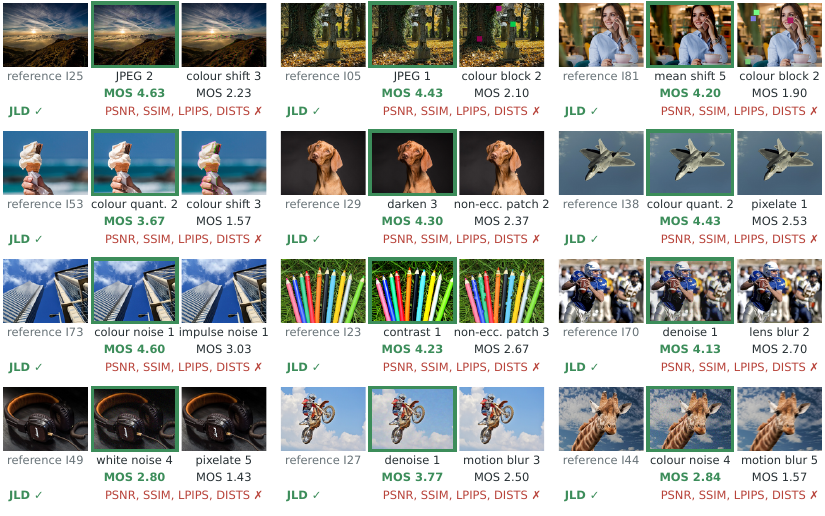}
\caption{Selected 2AFC cases in which only \name{} agrees with human scores. The green-framed image has higher MOS and receives the smaller \name{} distance, while PSNR, SSIM, LPIPS-VGG, and DISTS prefer the other image.}
\label{fig:showcase-2afc}
\end{figure}

For every same-reference pair, we define a 2AFC trial when the MOS difference is statistically significant,
\[
|\Delta\mathrm{MOS}|
>
1.96
\sqrt{\frac{v_A+v_B}{30}},
\]
where $v_A$ and $v_B$ are the reported rating variances.
This yields 387{,}603 trials out of 503{,}750 possible pairs.

\name{} agrees with the human preference on 93.3\% of these trials, compared with 93.0\% for DISTS, 82.9\% for LPIPS-VGG, 81.2\% for PSNR, and 78.4\% for SSIM.
There are 1{,}954 pairs on which \name{} alone agrees with humans while all four alternatives disagree, compared with 1{,}418 pairs showing the reverse pattern.
Figure~\ref{fig:showcase-2afc} shows examples from the former set, selected by decreasing MOS gap with at most one pair per reference.
Many failures of the other distances involve a local defect or colour shift competing with a milder global change such as brightness, contrast, compression, or denoising.
Figure~\ref{fig:choice} reports the corresponding aggregate pairwise analysis across all four image databases without this significance filter.

\subsection{Consistency across distortion families}

\begin{figure}[t]
\centering
\includegraphics[width=0.85\linewidth]{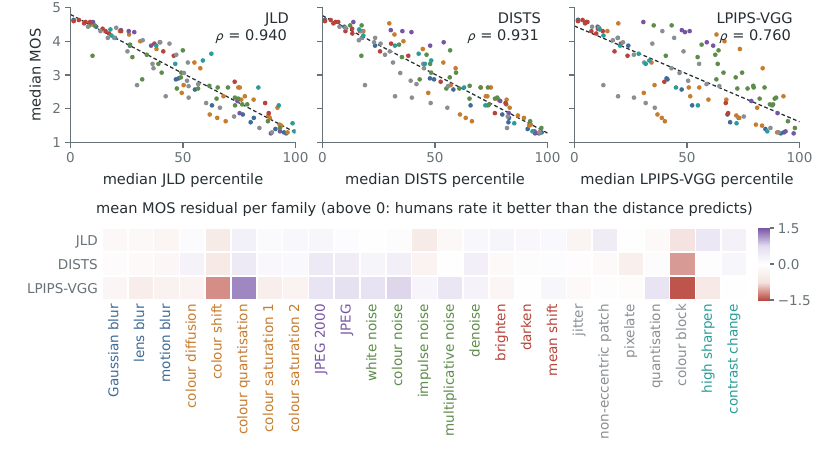}
\caption{Distance against MOS across distortion families and levels on the KADID-10k test split. Each point is the median over 65 references; dashed lines are linear fits. Bottom: mean MOS residual for each family. Red denotes under-penalisation by the distance and violet over-penalisation.}
\label{fig:showcase-families}
\end{figure}

A perceptual distance used as a loss or quality threshold should assign similar values to distortions of similar perceptual severity.
If one distance value represented one amount of perceptual damage, the 125 family--level cells in Figure~\ref{fig:showcase-families} would lie near a common decreasing curve.
This behaviour is strongest for \name{} and DISTS.
Across the 125 cells, \name{} reaches \srcc{} $0.940$ with a linear-fit error of $0.38$ MOS, while DISTS reaches $0.931$ and $0.40$.
LPIPS-VGG reaches only $0.760$ with an error of $0.70$ MOS, and its residual exceeds one MOS point for three colour distortion families.
The remaining systematic errors of \name{} are also visible: it under-penalises colour block distortion by $-0.72$ MOS and colour shift by $-0.53$, while over-penalising high sharpening by $+0.51$.

\subsection{gMAD stress tests}

\begin{figure}[t]
\centering
\includegraphics[width=0.85\linewidth]{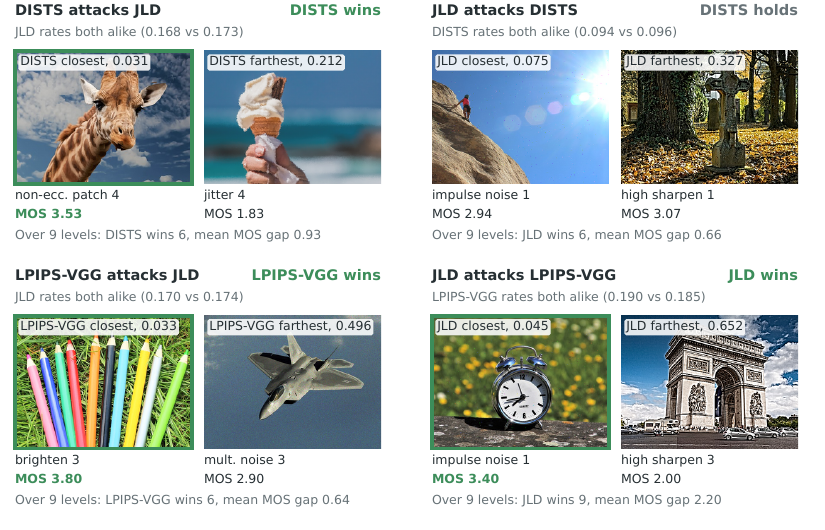}
\caption{gMAD \citep{ma2016gmad} comparisons on the KADID-10k test split. The defender is restricted to images with nearly equal distance, while the attacker selects the pair it considers most different. The attacker wins when human scores significantly prefer its closer image (green frame).}
\label{fig:showcase-gmad}
\end{figure}

gMAD \citep{ma2016gmad} searches specifically for cases where two distances disagree.
For each defender level, we restrict its score to within one percentile point, giving about 162 candidate images, and let the attacker choose its closest and farthest images within that band.
The attacker wins when humans significantly prefer its closer image ($z>1.96$).
We repeat this test at nine operating points, from the 10th to the 90th percentile.
\name{} and DISTS expose different weaknesses in one another: each wins as attacker at 6 of the 9 operating points (Figure~\ref{fig:showcase-gmad}).
DISTS identifies cases where \name{} assigns too much cost to a spatially displaced patch relative to jitter.
Conversely, \name{} identifies cases where DISTS treats high sharpening as similarly mild to light noise or colour diffusion.
When attacking LPIPS-VGG, \name{} wins at all 9 operating points with a mean human-score gap of 2.20 MOS.
LPIPS-VGG wins in the reverse direction at 6 operating points, but with a substantially smaller mean gap of 0.64 MOS.

% The previous appendix distance-ruler figure is omitted because
% Figure~\ref{fig:B2-see} in the main paper already provides the corresponding
% ordered visualisation of distance against human perceptual severity.

\FloatBarrier

\Needspace{4\baselineskip}
\section{Full image results}\label{app:results}

Tables~\ref{tab:app_fr_wide_p} and~\ref{tab:app_fr_wide_k} report correlations on all six image sets under each method's native preprocessing (P) and identical cropping (C).
Because \plcc{} uses a separate nonlinear mapping for each dataset, it measures within-dataset prediction rather than calibration transfer.
% Generated by scripts/v23/app_fr_merge.py from tables/app_fr_<dataset>.tex; do not edit by hand.
% Cell sources: the per-cell comments in tables/app_fr_<dataset>.tex (scripts/70_tables_v5.py).
\begin{table}[t]
\centering\footnotesize% v23.5 P2: was \small
\setlength{\tabcolsep}{0pt}
\caption{\srcc{} and \plcc{} per dataset. Violet: \name{} family; grey: fitted on human judgements. Higher is better; \textbf{bold} and \underline{underlining} mark the best and second distinct displayed values in each column.}
\label{tab:app_fr_wide_p}
\colorlet{humanJudgements}{gray!7}
\definecolor{ourMetrics}{HTML}{F3EFF8}
\begin{tabular*}{\linewidth}{@{\extracolsep{\fill}}lrrrrrrrrrrrr@{}}
\toprule
 & \multicolumn{2}{c}{TID2013} & \multicolumn{2}{c}{CSIQ} & \multicolumn{2}{c}{LIVE} & \multicolumn{2}{c}{KADID test} & \multicolumn{2}{c}{PIPAL pub.} & \multicolumn{2}{c}{PIPAL val.} \\
\cmidrule(lr){2-3}\cmidrule(lr){4-5}\cmidrule(lr){6-7}\cmidrule(lr){8-9}\cmidrule(lr){10-11}\cmidrule(lr){12-13}
Metric & \srcc & \plcc & \srcc & \plcc & \srcc & \plcc & \srcc & \plcc & \srcc & \plcc & \srcc & \plcc \\
\midrule
PSNR & 0.687 & 0.679 & 0.809 & 0.828 & 0.873 & 0.868 & 0.673 & 0.678 & 0.407 & 0.415 & 0.255 & 0.292 \\
SSIM & 0.627 & 0.686 & 0.837 & 0.815 & 0.910 & 0.915 & 0.621 & 0.630 & 0.504 & 0.509 & 0.363 & 0.426 \\
MS-SSIM & 0.786 & 0.834 & 0.913 & 0.899 & 0.951 & 0.948 & 0.826 & 0.823 & 0.562 & 0.593 & 0.491 & 0.566 \\
FSIM & 0.851 & 0.877 & 0.931 & 0.919 & \textbf{0.965} & \textbf{0.961} & 0.853 & 0.850 & 0.589 & 0.615 & 0.468 & 0.570 \\
VSI & \textbf{0.895} & \textbf{0.898} & 0.940 & 0.925 & 0.949 & 0.945 & 0.876 & 0.876 & 0.539 & 0.560 & 0.450 & 0.524 \\
GMSD & 0.804 & 0.860 & 0.957 & 0.954 & 0.960 & \underline{0.960} & 0.849 & 0.848 & 0.583 & 0.629 & 0.585 & 0.658 \\
MAD & 0.743 & 0.759 & 0.930 & 0.929 & 0.951 & 0.950 & 0.731 & 0.733 & -- & -- & 0.665 & 0.687 \\
NLPD & 0.799 & 0.844 & 0.937 & 0.927 & 0.937 & 0.934 & 0.811 & 0.809 & 0.482 & 0.509 & 0.370 & 0.452 \\
VIF & 0.677 & 0.774 & 0.919 & 0.928 & \underline{0.964} & \underline{0.960} & 0.685 & 0.695 & -- & -- & 0.503 & 0.581 \\
\midrule
DeepWSD & 0.855 & 0.877 & \underline{0.959} & 0.951 & 0.955 & 0.951 & 0.882 & 0.881 & 0.561 & 0.567 & 0.460 & 0.491 \\
DeepDC & 0.816 & 0.841 & 0.954 & \underline{0.959} & 0.950 & 0.947 & \textbf{0.900} & \textbf{0.898} & \textbf{0.654} & \textbf{0.669} & \underline{0.751} & \underline{0.754} \\
LASI & 0.522 & 0.582 & 0.840 & 0.854 & 0.883 & 0.881 & 0.723 & 0.724 & 0.514 & 0.504 & 0.581 & 0.571 \\
IEM & 0.810 & 0.818 & 0.917 & 0.906 & 0.922 & 0.920 & 0.821 & 0.819 & -- & -- & 0.182 & 0.243 \\
\midrule
\rowcolor{humanJudgements}LPIPS-VGG & 0.670 & 0.759 & 0.883 & 0.906 & 0.932 & 0.934 & 0.724 & 0.732 & 0.573 & 0.611 & 0.612 & 0.670 \\
\rowcolor{humanJudgements}LPIPS-Alex & 0.744 & 0.771 & 0.923 & 0.920 & 0.923 & 0.916 & 0.824 & 0.824 & -- & -- & 0.620 & 0.633 \\
\rowcolor{humanJudgements}DISTS & 0.818 & 0.847 & 0.943 & 0.946 & 0.954 & 0.953 & 0.885 & 0.885 & \underline{0.615} & 0.632 & 0.704 & 0.722 \\
\rowcolor{humanJudgements}PieAPP & 0.844 & 0.833 & 0.897 & 0.881 & 0.918 & 0.910 & 0.864 & 0.866 & -- & -- & 0.706 & 0.698 \\
\rowcolor{humanJudgements}DreamSim & 0.812 & 0.840 & 0.911 & 0.928 & 0.910 & 0.918 & 0.851 & 0.848 & -- & -- & \textbf{0.759} & \textbf{0.781} \\
\midrule
\rowcolor{ourMetrics}Lens, no pre-filter & 0.857 & 0.871 & 0.956 & 0.944 & 0.942 & 0.936 & 0.866 & 0.865 & 0.587 & 0.602 & 0.575 & 0.629 \\
\rowcolor{ourMetrics}\namelin{} & 0.826 & 0.828 & 0.905 & 0.885 & 0.910 & 0.907 & 0.809 & 0.807 & 0.468 & 0.473 & 0.344 & 0.402 \\
\rowcolor{ourMetrics}\namefast{} & 0.875 & 0.886 & 0.956 & 0.944 & 0.945 & 0.939 & 0.870 & 0.869 & 0.592 & 0.605 & 0.578 & 0.634 \\
\rowcolor{ourMetrics}\name{} & \underline{0.877} & \underline{0.897} & \textbf{0.971} & \textbf{0.974} & \underline{0.964} & \textbf{0.961} & \underline{0.892} & \underline{0.889} & 0.603 & \underline{0.640} & 0.624 & 0.683 \\
\bottomrule
\end{tabular*}
\end{table}

\name{} reaches \srcc{} 0.971 on CSIQ development and 0.964 on LIVE.
On TID2013, which informed the chroma pre-filter diagnosis, it reaches 0.877.
On the held-out sets, \name{} reaches 0.892 on KADID-10k test, 0.603 on PIPAL public, and 0.624 on PIPAL validation.
DeepDC remains stronger on KADID-10k test and PIPAL, while VSI remains stronger on TID2013.
% Generated by scripts/v23/app_fr_merge.py from tables/app_fr_<dataset>.tex; do not edit by hand.
% Cell sources: the per-cell comments in tables/app_fr_<dataset>.tex (scripts/70_tables_v5.py).
\begin{table}[t]
\centering\footnotesize% v23.5 P2: was \small
\setlength{\tabcolsep}{0pt}
\caption{Comparing KRCC and SRCC}
\label{tab:app_fr_wide_k}
\colorlet{humanJudgements}{gray!7}
\definecolor{ourMetrics}{HTML}{F3EFF8}
\begin{tabular*}{\linewidth}{@{\extracolsep{\fill}}lrrrrrrrrrrrr@{}}
\toprule
 & \multicolumn{2}{c}{TID2013} & \multicolumn{2}{c}{CSIQ} & \multicolumn{2}{c}{LIVE} & \multicolumn{2}{c}{KADID test} & \multicolumn{2}{c}{PIPAL pub.} & \multicolumn{2}{c}{PIPAL val.} \\
\cmidrule(lr){2-3}\cmidrule(lr){4-5}\cmidrule(lr){6-7}\cmidrule(lr){8-9}\cmidrule(lr){10-11}\cmidrule(lr){12-13}
Metric & \krcc & \srcc & \krcc & \srcc & \krcc & \srcc & \krcc & \srcc & \krcc & \srcc & \krcc & \srcc \\
\midrule
PSNR & 0.496 & 0.686 & 0.599 & 0.806 & 0.680 & 0.870 & 0.485 & 0.675 & 0.276 & -- & 0.174 & -- \\
SSIM & 0.455 & 0.628 & 0.632 & 0.837 & 0.731 & 0.910 & 0.449 & 0.623 & 0.349 & -- & 0.249 & -- \\
MS-SSIM & 0.605 & 0.784 & 0.738 & 0.912 & 0.803 & 0.949 & 0.635 & 0.826 & 0.397 & -- & 0.341 & -- \\
FSIM & 0.666 & 0.737 & 0.768 & 0.931 & \textbf{0.837} & \textbf{0.964} & 0.664 & 0.768 & 0.416 & -- & 0.323 & -- \\
VSI & \textbf{0.716} & 0.812 & 0.780 & 0.940 & 0.799 & 0.949 & 0.688 & 0.837 & 0.375 & -- & 0.309 & -- \\
GMSD & 0.634 & 0.802 & 0.812 & 0.957 & 0.827 & \underline{0.960} & 0.665 & 0.848 & 0.414 & -- & 0.415 & -- \\
MAD & 0.545 & 0.743 & 0.759 & 0.930 & 0.800 & 0.954 & 0.557 & 0.732 & -- & -- & 0.471 & -- \\
NLPD & 0.624 & 0.791 & 0.772 & 0.936 & 0.777 & 0.930 & 0.621 & 0.812 & 0.335 & -- & 0.254 & -- \\
VIF & 0.515 & 0.678 & 0.753 & 0.920 & 0.828 & \textbf{0.964} & 0.512 & 0.686 & -- & -- & 0.350 & -- \\
\midrule
DeepWSD & 0.667 & 0.874 & \underline{0.815} & \underline{0.963} & 0.810 & 0.893 & 0.700 & \underline{0.887} & 0.391 & -- & 0.319 & -- \\
DeepDC & 0.622 & 0.702 & 0.810 & 0.918 & 0.801 & 0.947 & \textbf{0.727} & 0.827 & \textbf{0.469} & -- & \underline{0.550} & -- \\
LASI & 0.368 & -- & 0.641 & -- & 0.696 & -- & 0.534 & -- & 0.356 & -- & 0.405 & -- \\
IEM & 0.624 & -- & 0.735 & -- & 0.747 & -- & 0.629 & -- & -- & -- & 0.123 & -- \\
\midrule
\rowcolor{humanJudgements}LPIPS-VGG & 0.497 & 0.668 & 0.697 & 0.882 & 0.765 & 0.932 & 0.535 & 0.726 & 0.404 & -- & 0.438 & -- \\
\rowcolor{humanJudgements}LPIPS-Alex & 0.548 & 0.748 & 0.750 & 0.924 & 0.751 & 0.924 & 0.632 & 0.827 & -- & -- & 0.435 & -- \\
\rowcolor{humanJudgements}DISTS & 0.626 & 0.708 & 0.788 & 0.930 & 0.809 & 0.947 & 0.708 & 0.818 & \underline{0.437} & -- & 0.510 & -- \\
\rowcolor{humanJudgements}PieAPP & 0.657 & 0.846 & 0.711 & 0.896 & 0.749 & 0.918 & 0.674 & 0.864 & -- & -- & 0.512 & -- \\
\rowcolor{humanJudgements}DreamSim & 0.614 & 0.810 & 0.738 & 0.915 & 0.745 & 0.904 & 0.657 & 0.853 & -- & -- & \textbf{0.563} & -- \\
\midrule
\rowcolor{ourMetrics}Lens, no pre-filter & 0.672 & 0.857 & 0.811 & 0.956 & 0.785 & 0.942 & 0.677 & 0.866 & 0.413 & 0.587 & 0.406 & 0.575 \\
\rowcolor{ourMetrics}\namelin{} & 0.648 & 0.826 & 0.723 & 0.905 & 0.734 & 0.910 & 0.617 & 0.809 & 0.322 & 0.468 & 0.237 & 0.344 \\
\rowcolor{ourMetrics}\namefast{} & 0.691 & \underline{0.875} & 0.811 & 0.956 & 0.790 & 0.945 & 0.683 & 0.870 & 0.417 & \underline{0.592} & 0.408 & \underline{0.578} \\
\rowcolor{ourMetrics}\name{} & \underline{0.693} & \textbf{0.877} & \textbf{0.849} & \textbf{0.971} & \underline{0.830} & \textbf{0.964} & \underline{0.711} & \textbf{0.892} & 0.426 & \textbf{0.603} & 0.443 & \textbf{0.624} \\
\bottomrule
\end{tabular*}
\end{table}

On BAPPS, \name{} performs strongly on CNN-generated distortions but is weaker on traditional distortions, with agreement scores of 0.821 and 0.681, respectively.
For comparison, LPIPS-VGG reaches 0.814 and 0.714, while DISTS reaches 0.813 and 0.749.
Across the remaining BAPPS subsets, \name{} reaches 0.625 on colour changes, 0.595 on deblurring, 0.608 on frame interpolation, and 0.694 on super-resolution.
LPIPS was trained using BAPPS human judgements, while DISTS used KADID-10k during training \citep{zhang2018lpips,ding2022dists}.
The no-reference extension reaches 0.486 on held-out KonIQ images, while CLIP-IQA reaches 0.685 on the full dataset.
Because these evaluations use different subsets, the two numbers should not be directly compared.

\Needspace{7\baselineskip}% arXiv layout: was 16; v23.5 P2: moved before the heading so it does not strand at a page foot
\section{Additional results}\label{app:moved-results}\label{app:main-moved}\label{app:types}

\subsection{Ablations}

% Source: paper/distance_metric/tables/ablation.tex; results/reproduction.csv; results/table_main_v2.csv.
% Generated by scripts/70_tables_v5.py; do not edit by hand.
% Cell order: TID2013, CSIQ, LIVE, KADID-10k test, arithmetic mean of the four unrounded SRCCs.
% Row dimensions: results/reproduction.csv, config=full384_euclid/random64_euclid/pca64_rho10; scripts/03_reproduce.py:configurations.
% Caption norm: results/reproduction.csv, config=pca64_rho10; jld/metric.py:metric_matrix implements regularised covariance inversion.
% Caption subset: results/structural_vs_instantaneous.csv:2, variant=instantaneous, column=n, value=1000.
% Matched expected lens: results/structural_vs_instantaneous.csv:6, variant=structural_lens64, column=SRCC, value=0.8607083785147341.
% Output term alone includes the chroma pre-filter; source rows are identified individually below.
% KADID-10k development evaluations are never substituted into the test column.
% v23.3 (D2): this table also carries tab:climb and fig:climb; the old climb table (sections/figs/fig_climb.tex) repeated its first six rows and is no longer input.
\begin{table}[tbp]
\centering\footnotesize% v23.5 P2: was \small
\setlength{\tabcolsep}{3pt}
\renewcommand{\arraystretch}{1.03}
\caption{Building \name{} step by step on one frozen encoder (\srcc{}), and other distances on the same tokens. Random directions gain nothing; the directions of $\M$ give most of the gain. \textbf{Bold} marks the best and \underline{underlining} the second distinct displayed value. $\dagger$ marks a score on 1{,}000 TID2013 pairs.}
\label{tab:ablation}\label{tab:climb}\label{fig:climb}
\definecolor{ourMetrics}{HTML}{F3EFF8}
\begin{tabular*}{\linewidth}{@{\extracolsep{\fill}}lrrrrr@{}}
\toprule
Variant & \shortstack{TID2013\\(diag.)} & \shortstack{CSIQ\\(dev.)} & LIVE & \shortstack{KADID-10k\\test} & Mean \\
\midrule
% Cell tid2013=0.660; exact=0.6598224438719781; source: results/reproduction.csv:6, dataset=tid2013, config=full384_euclid, column=SRCC.
% Cell csiq=0.841; exact=0.8410747571611759; source: results/reproduction.csv:12, dataset=csiq, config=full384_euclid, column=SRCC.
% Cell live=0.859; exact=0.8587854648162377; source: results/reproduction.csv:18, dataset=live, config=full384_euclid, column=SRCC.
% Cell kadid10k_test=0.727; exact=0.7268390814402951; source: results/reproduction.csv:36, dataset=kadid10k_test, config=full384_euclid, column=SRCC.
% Cell mean=0.772; exact=0.7716304368224217; source: arithmetic mean of unrounded cells: results/reproduction.csv:6, dataset=tid2013, config=full384_euclid, column=SRCC; results/reproduction.csv:12, dataset=csiq, config=full384_euclid, column=SRCC; results/reproduction.csv:18, dataset=live, config=full384_euclid, column=SRCC; results/reproduction.csv:36, dataset=kadid10k_test, config=full384_euclid, column=SRCC.
Full block 1 & 0.660 & 0.841 & 0.859 & 0.727 & 0.772 \\
% Cell tid2013=0.660; exact=0.6603864645906662; source: results/reproduction.csv:7, dataset=tid2013, config=random64_euclid, column=SRCC.
% Cell csiq=0.841; exact=0.8406684417816745; source: results/reproduction.csv:13, dataset=csiq, config=random64_euclid, column=SRCC.
% Cell live=0.859; exact=0.8594470879854668; source: results/reproduction.csv:19, dataset=live, config=random64_euclid, column=SRCC.
% Cell kadid10k_test=0.726; exact=0.7263314179233384; source: results/reproduction.csv:37, dataset=kadid10k_test, config=random64_euclid, column=SRCC.
% Cell mean=0.772; exact=0.771708353070286475; source: arithmetic mean of unrounded cells: results/reproduction.csv:7, dataset=tid2013, config=random64_euclid, column=SRCC; results/reproduction.csv:13, dataset=csiq, config=random64_euclid, column=SRCC; results/reproduction.csv:19, dataset=live, config=random64_euclid, column=SRCC; results/reproduction.csv:37, dataset=kadid10k_test, config=random64_euclid, column=SRCC.
Random 64 & 0.660 & 0.841 & 0.859 & 0.726 & 0.772 \\
% Cell tid2013=0.766; exact=0.7660917928383716; source: results/reproduction.csv:5, dataset=tid2013, config=pca64_rho10, column=SRCC.
% Cell csiq=0.912; exact=0.9121502562606459; source: results/reproduction.csv:11, dataset=csiq, config=pca64_rho10, column=SRCC.
% Cell live=0.925; exact=0.9246891384583238; source: results/reproduction.csv:17, dataset=live, config=pca64_rho10, column=SRCC.
% Cell kadid10k_test=0.800; exact=0.7997062488042886; source: results/reproduction.csv:35, dataset=kadid10k_test, config=pca64_rho10, column=SRCC.
% Cell mean=0.851; exact=0.850659359090407475; source: arithmetic mean of unrounded cells: results/reproduction.csv:5, dataset=tid2013, config=pca64_rho10, column=SRCC; results/reproduction.csv:11, dataset=csiq, config=pca64_rho10, column=SRCC; results/reproduction.csv:17, dataset=live, config=pca64_rho10, column=SRCC; results/reproduction.csv:35, dataset=kadid10k_test, config=pca64_rho10, column=SRCC.
PCA 64 & 0.766 & 0.912 & 0.925 & 0.800 & 0.851 \\
\midrule
% Cell tid2013=0.857; exact=0.8571842109104165; source: results/table_main_v2.csv:20, row=jld, column=tid2013.
% Cell csiq=0.956; exact=0.9562379646933636; source: results/table_main_v2.csv:20, row=jld, column=csiq.
% Cell live=0.942; exact=0.9421814863779342; source: results/table_main_v2.csv:20, row=jld, column=live.
% Cell kadid10k_test=0.866; exact=0.8661029743022443; source: results/table_main_v2.csv:20, row=jld, column=kadid10k_test.
% Cell mean=0.905; exact=0.90542665907098965; source: arithmetic mean of unrounded cells: results/table_main_v2.csv:20, row=jld, column=tid2013; results/table_main_v2.csv:20, row=jld, column=csiq; results/table_main_v2.csv:20, row=jld, column=live; results/table_main_v2.csv:20, row=jld, column=kadid10k_test.
\rowcolor{ourMetrics}Lens 64 & 0.857 & \underline{0.956} & 0.942 & 0.866 & 0.905 \\
% Cell tid2013=0.875; exact=0.8745150738826878; source: results/table_main_v2.csv:21, row=jld_v2, column=tid2013.
% Cell csiq=0.956; exact=0.9558676422362274; source: results/table_main_v2.csv:21, row=jld_v2, column=csiq.
% Cell live=0.945; exact=0.9448069352546935; source: results/table_main_v2.csv:21, row=jld_v2, column=live.
% Cell kadid10k_test=0.870; exact=0.8704005316941635; source: results/table_main_v2.csv:21, row=jld_v2, column=kadid10k_test.
% Cell mean=0.911; exact=0.91139754576694305; source: arithmetic mean of unrounded cells: results/table_main_v2.csv:21, row=jld_v2, column=tid2013; results/table_main_v2.csv:21, row=jld_v2, column=csiq; results/table_main_v2.csv:21, row=jld_v2, column=live; results/table_main_v2.csv:21, row=jld_v2, column=kadid10k_test.
\rowcolor{ourMetrics}\quad + pre-filter & \underline{0.875} & \underline{0.956} & \underline{0.945} & \underline{0.870} & \underline{0.911} \\
% Cell tid2013=0.877; exact=0.8767106009025009; source: results/table_main_v2.csv:22, row=jld_v2p, column=tid2013.
% Cell csiq=0.971; exact=0.9708250295632062; source: results/table_main_v2.csv:22, row=jld_v2p, column=csiq.
% Cell live=0.964; exact=0.9638385267329496; source: results/table_main_v2.csv:22, row=jld_v2p, column=live.
% Cell kadid10k_test=0.892; exact=0.8916927527531087; source: results/table_main_v2.csv:22, row=jld_v2p, column=kadid10k_test.
% Cell mean=0.926; exact=0.92576672748794135; source: arithmetic mean of unrounded cells: results/table_main_v2.csv:22, row=jld_v2p, column=tid2013; results/table_main_v2.csv:22, row=jld_v2p, column=csiq; results/table_main_v2.csv:22, row=jld_v2p, column=live; results/table_main_v2.csv:22, row=jld_v2p, column=kadid10k_test.
\rowcolor{ourMetrics}\quad + output term = \name{} & \textbf{0.877} & \textbf{0.971} & \textbf{0.964} & \textbf{0.892} & \textbf{0.926} \\
\midrule
% Cell tid2013=0.748; exact=0.7477516405239852; source: results/explore/e06_frozen_summary.csv:50, dataset=tid2013, column=SRCC.
% Cell csiq=0.840; exact=0.8397815272078942; source: results/explore/e06_frozen_summary.csv:8, dataset=csiq, column=SRCC.
% Cell live=0.910; exact=0.9095445830992883; source: results/explore/e06_frozen_summary.csv:29, dataset=live, column=SRCC.
% Cell kadid10k_test=0.808; exact=0.8079140035219272; source: results/explore/e06_frozen_summary.csv:22, dataset=kadid10k_test, column=SRCC.
% Cell mean=0.826; exact=0.826247938588273725; source: arithmetic mean of unrounded cells: results/explore/e06_frozen_summary.csv:50, dataset=tid2013, column=SRCC; results/explore/e06_frozen_summary.csv:8, dataset=csiq, column=SRCC; results/explore/e06_frozen_summary.csv:29, dataset=live, column=SRCC; results/explore/e06_frozen_summary.csv:22, dataset=kadid10k_test, column=SRCC.
Output only & 0.748 & 0.840 & 0.910 & 0.808 & 0.826 \\
% Cell tid2013=0.679$^{\dagger}$; exact=0.6785149243483536; source: results/structural_vs_instantaneous.csv:2, variant=instantaneous, column=SRCC.
% Cell csiq=--; exact=unavailable; source: results/structural_vs_instantaneous.csv, variant=instantaneous; no csiq evaluation.
% Cell live=--; exact=unavailable; source: results/structural_vs_instantaneous.csv, variant=instantaneous; no live evaluation.
% Cell kadid10k_test=--; exact=unavailable; source: results/structural_vs_instantaneous.csv, variant=instantaneous; no kadid10k_test evaluation.
% Cell mean=--; exact=unavailable; source: results/structural_vs_instantaneous.csv, variant=instantaneous; missing three dataset evaluations.
Pointwise Jacobian & 0.679$^{\dagger}$ & -- & -- & -- & -- \\
\midrule
\multicolumn{6}{@{}l}{\textit{Other distances on the same tokens}} \\
% Cell tid2013=0.661; exact=0.6608679934264635; source: results/naive_feature_baselines.csv:11, dataset=tid2013, distance=uniform_lpips, column=SRCC.
% Cell csiq=0.867; exact=0.8674622808539372; source: results/naive_feature_baselines.csv:2, dataset=csiq, distance=uniform_lpips, column=SRCC.
% Cell live=0.912; exact=0.9119819831563096; source: results/naive_feature_baselines.csv:20, dataset=live, distance=uniform_lpips, column=SRCC.
% Cell kadid10k_test=0.740; exact=0.7404339635576933; source: results/naive_feature_baselines.csv:38, dataset=kadid10k_test, distance=uniform_lpips, column=SRCC.
% Cell mean=0.795; exact=0.7951865552486009; source: arithmetic mean of unrounded cells: results/naive_feature_baselines.csv:11, dataset=tid2013, distance=uniform_lpips, column=SRCC; results/naive_feature_baselines.csv:2, dataset=csiq, distance=uniform_lpips, column=SRCC; results/naive_feature_baselines.csv:20, dataset=live, distance=uniform_lpips, column=SRCC; results/naive_feature_baselines.csv:38, dataset=kadid10k_test, distance=uniform_lpips, column=SRCC.
Uniform LPIPS & 0.661 & 0.867 & 0.912 & 0.740 & 0.795 \\
% Cell tid2013=0.611; exact=0.611069459373261; source: results/naive_feature_baselines.csv:12, dataset=tid2013, distance=standardised_euclidean, column=SRCC.
% Cell csiq=0.802; exact=0.8022385232673042; source: results/naive_feature_baselines.csv:3, dataset=csiq, distance=standardised_euclidean, column=SRCC.
% Cell live=0.832; exact=0.8315536451079419; source: results/naive_feature_baselines.csv:21, dataset=live, distance=standardised_euclidean, column=SRCC.
% Cell kadid10k_test=0.684; exact=0.6842608533466347; source: results/naive_feature_baselines.csv:39, dataset=kadid10k_test, distance=standardised_euclidean, column=SRCC.
% Cell mean=0.732; exact=0.73228062027378545; source: arithmetic mean of unrounded cells: results/naive_feature_baselines.csv:12, dataset=tid2013, distance=standardised_euclidean, column=SRCC; results/naive_feature_baselines.csv:3, dataset=csiq, distance=standardised_euclidean, column=SRCC; results/naive_feature_baselines.csv:21, dataset=live, distance=standardised_euclidean, column=SRCC; results/naive_feature_baselines.csv:39, dataset=kadid10k_test, distance=standardised_euclidean, column=SRCC.
Standardised Euclidean & 0.611 & 0.802 & 0.832 & 0.684 & 0.732 \\
% Cell tid2013=0.662; exact=0.6615099155777487; source: results/naive_feature_baselines.csv:13, dataset=tid2013, distance=deepwsd_tokens, column=SRCC.
% Cell csiq=0.843; exact=0.8428242201289439; source: results/naive_feature_baselines.csv:4, dataset=csiq, distance=deepwsd_tokens, column=SRCC.
% Cell live=0.865; exact=0.8646956757440344; source: results/naive_feature_baselines.csv:22, dataset=live, distance=deepwsd_tokens, column=SRCC.
% Cell kadid10k_test=0.728; exact=0.7276824844097848; source: results/naive_feature_baselines.csv:40, dataset=kadid10k_test, distance=deepwsd_tokens, column=SRCC.
% Cell mean=0.774; exact=0.77417807396512795; source: arithmetic mean of unrounded cells: results/naive_feature_baselines.csv:13, dataset=tid2013, distance=deepwsd_tokens, column=SRCC; results/naive_feature_baselines.csv:4, dataset=csiq, distance=deepwsd_tokens, column=SRCC; results/naive_feature_baselines.csv:22, dataset=live, distance=deepwsd_tokens, column=SRCC; results/naive_feature_baselines.csv:40, dataset=kadid10k_test, distance=deepwsd_tokens, column=SRCC.
DeepWSD-style & 0.662 & 0.843 & 0.865 & 0.728 & 0.774 \\
% Cell tid2013=0.615; exact=0.6149319342097733; source: results/naive_feature_baselines.csv:14, dataset=tid2013, distance=deepdc_tokens, column=SRCC.
% Cell csiq=0.827; exact=0.8269010473543156; source: results/naive_feature_baselines.csv:5, dataset=csiq, distance=deepdc_tokens, column=SRCC.
% Cell live=0.819; exact=0.8186553007977405; source: results/naive_feature_baselines.csv:23, dataset=live, distance=deepdc_tokens, column=SRCC.
% Cell kadid10k_test=0.763; exact=0.7628998606768999; source: results/naive_feature_baselines.csv:41, dataset=kadid10k_test, distance=deepdc_tokens, column=SRCC.
% Cell mean=0.756; exact=0.755847035759682325; source: arithmetic mean of unrounded cells: results/naive_feature_baselines.csv:14, dataset=tid2013, distance=deepdc_tokens, column=SRCC; results/naive_feature_baselines.csv:5, dataset=csiq, distance=deepdc_tokens, column=SRCC; results/naive_feature_baselines.csv:23, dataset=live, distance=deepdc_tokens, column=SRCC; results/naive_feature_baselines.csv:41, dataset=kadid10k_test, distance=deepdc_tokens, column=SRCC.
DeepDC-style & 0.615 & 0.827 & 0.819 & 0.763 & 0.756 \\
\bottomrule
\end{tabular*}
% \par\vspace{3pt}\begin{minipage}{\linewidth}\footnotesize Higher \srcc{} is better. \textbf{Bold}: best; \underline{underlined}: second distinct displayed value; ties share marks. PCA 64 uses regularised Mahalanobis distance ($\rho=10$); lens 64 uses Euclidean distance. The four-dataset mean combines diagnostic, development and held-out data. $\dagger$: $1{,}000$ TID2013 pairs; dashes: unavailable. Other distances use the same unfiltered block 1 tokens.\end{minipage}
% v23.6 P6: shorter note.
\end{table}

Table~\ref{tab:ablation} builds \name{} one component at a time on DINOv2-S block~1 and compares alternative distances on the same token features.
PCA~64 uses regularised Mahalanobis distance ($\rho=10$), while lens~64 uses Euclidean distance; all other token baselines use the same unfiltered block-1 features.
Using all 384 raw features gives a four-dataset mean \srcc{} of 0.772, and projecting onto 64 random directions leaves it unchanged.
The 64 leading PCA directions improve the mean to 0.851, while the 64 leading eigenvectors of $\M$ reach 0.905 and improve every dataset.
Thus, the main gain comes from ranking feature directions by the encoder output's sensitivity rather than by feature variance or dimensionality reduction alone.
The output term alone reaches 0.826, showing that most of the signal comes from the lens term.
The chroma pre-filter raises TID2013 from 0.857 to 0.875, while the output term raises LIVE from 0.945 to 0.964 and KADID-10k test from 0.870 to 0.892.
On the same lensed tokens, a DeepWSD-style comparison matches the Euclidean lens at 0.905, whereas a DeepDC-style comparison falls to 0.744.
The simple Euclidean distance therefore gives up no accuracy once the feature directions have been selected by the Jacobian lens.

% arXiv layout: fig:B-pairs is input here (was after the Resolution paragraph) so it fills the page foot before G.2.
% v23.3 worker B. Replaces fig:visual in the main text. Selection and maps: scripts/v233/B_pairs_select.py (released package,
% CPU); drawing: scripts/v233/B_pairs_figure.py; numbers: figures/v233/B_numbers.json (pairs).
\begin{figure}[ht]
\centering
\includegraphics[width=\linewidth]{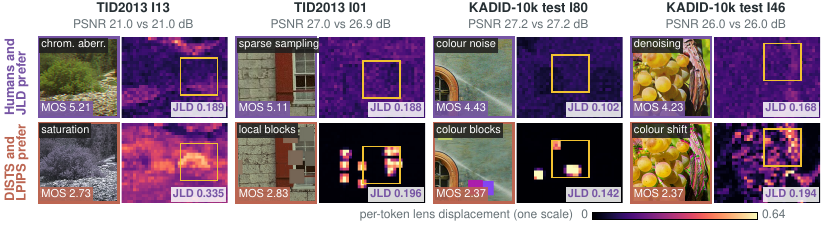}
% \caption{At equal PSNR, \name{} ranks the pair as people do and DISTS and LPIPS-VGG both reverse it. Each column is one reference with two distortions of the same PSNR: a detail crop (left) and \name{}'s per-token lens displacement over the whole image (right, one colour scale for the figure; the box marks the crop). The maps find the damage people object to: the desaturated foliage, the pasted colour blocks and the shifted colour channels light up, while the fringing, fine noise and resampling that viewers tolerate stay darker.}
\caption{At equal PSNR, \name{} ranks the pair as people do while DISTS and LPIPS-VGG reverse it. Each column shows a detail crop (left) and \name{}'s per-token lens displacement over the whole image (right; one colour scale; the box marks the crop).}
\label{fig:B-pairs}
\end{figure}

\Needspace{20\baselineskip}% arXiv layout: before the heading, sized to the wrapped table
\subsection{Resolution results}

% \Needspace{15\baselineskip}
\begin{wraptable}{r}{0.5\linewidth}% arXiv layout: was 0.42; the wrapped text is now as tall as the table
\vspace{-12pt}
\centering\small
% Moved from Section 5.4 (v23); body generated by scripts/v22/resolution.py from results/scorecard_resolution.csv.
\caption{TID2013 \srcc{} at 256 and 512\,px short side. The lens remains stable while several deep feature distances degrade.}
\label{tab:resolution}
% Generated by scripts/v22/resolution.py from results/scorecard_resolution.csv; do not edit by hand.
% TID2013 (3,000 pairs) SRCC at 256 and 512 px short side; change = exact 512 minus exact 256, then rounded.
% Lens = historical block-1 lens without chroma pre-filter or output term (scorecard metric "JLD").
% Rows sorted by absolute change within each group.
\definecolor{lensRow}{HTML}{F3EFF8}
\begin{tabular}{lccr}
\toprule
Method & 256 & 512 & Change \\
\midrule
\multicolumn{4}{l}{\textit{Deep features}} \\
\rowcolor{lensRow}\textbf{Lens} & 0.850 & 0.845 & \textbf{$-$0.005} \\
PieAPP & 0.618 & 0.700 & $+$0.081 \\
LPIPS-VGG & 0.757 & 0.661 & $-$0.096 \\
DISTS & 0.815 & 0.717 & $-$0.099 \\
DeepDC & 0.806 & 0.695 & $-$0.111 \\
LPIPS-Alex & 0.824 & 0.697 & $-$0.127 \\
\midrule
\multicolumn{4}{l}{\textit{Classical}} \\
NLPD & 0.787 & 0.787 & 0.000 \\
PSNR & 0.712 & 0.713 & 0.000 \\
VSI & 0.887 & 0.888 & $+$0.001 \\
MS-SSIM & 0.770 & 0.779 & $+$0.009 \\
SSIM & 0.727 & 0.641 & $-$0.086 \\
\bottomrule
\end{tabular}

\vspace{-8pt}
\end{wraptable}

% Sources: results/scorecard_resolution.csv; figures/v23/fig_visual_numbers.json
% (resolution, per_pair); scripts/v23/visual_compute.py.
Table~\ref{tab:resolution} gives the TID2013 results underlying Section~\ref{sec:resolution} for 11 of the 15 evaluated methods.
The full database is rescored after resizing each image to short sides of 256 and 512 pixels, while the original human scores are kept fixed.
This tests whether a distance preserves its agreement with human judgements when the same content is evaluated at a different image resolution.
The lens remains nearly unchanged, from 0.850 at 256 pixels to 0.845 at 512 pixels.
By comparison, the unweighted block-1 features fall from 0.784 to 0.656, DISTS from 0.815 to 0.717, and LPIPS-VGG from 0.757 to 0.661.
These results support the central observation of Section~\ref{sec:resolution}: weighting early features by output sensitivity is substantially more stable to resolution changes than comparing the same early features directly.

% Random directions add nothing, the lens peaks at $k=64$, and the flat lens outperforms the full metric tensor $\M$, its eigenvalue-weighted restriction and raw features. Figure~\ref{fig:B-pairs} shows four same-reference pairs of equal PSNR on which \name{} agrees with people and DISTS and LPIPS-VGG both disagree, chosen by the largest human-score gap: at 21.0\,dB a TID2013 saturation loss scores 0.335 against 0.189 for chromatic aberration, which people rate 2.5 points better.

% \input{sections/figs/fig_choice}  % v23.7 M7: moved to the main text (sections/figs/fig_M7_choice.tex)
% \input{sections/figs/fig_B_pairs} % arXiv layout: moved up, before the Resolution results heading

\begin{figure}[ht]
\centering
\includegraphics[width=\linewidth]{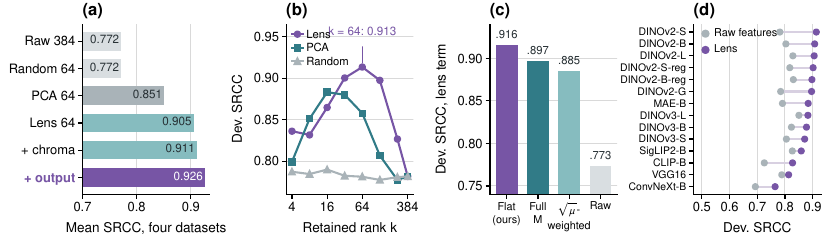}
\caption{Which parts of \name{} carry the agreement. (a)~Building the distance on DINOv2-S block 1 (mean \srcc{}, four datasets). (b)~Development \srcc{} against rank $k$. (c)~Weightings of the retained directions. (d)~Lens (violet) against raw features (grey) for 14 frozen encoders.}
\label{fig:B2-ablation-full}
\end{figure}

\IfFileExists{sections/app_main_moved.tex}{\input{sections/app_main_moved}}{}

\Needspace{4\baselineskip}
\section{Development choices}\label{app:sweep}

Labelled data are used only to select distance settings, never to fit the lens.
We evaluate 245 configurations on CSIQ and KADID-10k development and choose the simplest configuration within 0.005 of the best equally weighted mean \srcc{}.
The selected CLS-target configuration reaches 0.915, compared with 0.916 for the best tested configuration.

\begin{figure}[t]
\centering
\includegraphics[width=\linewidth]{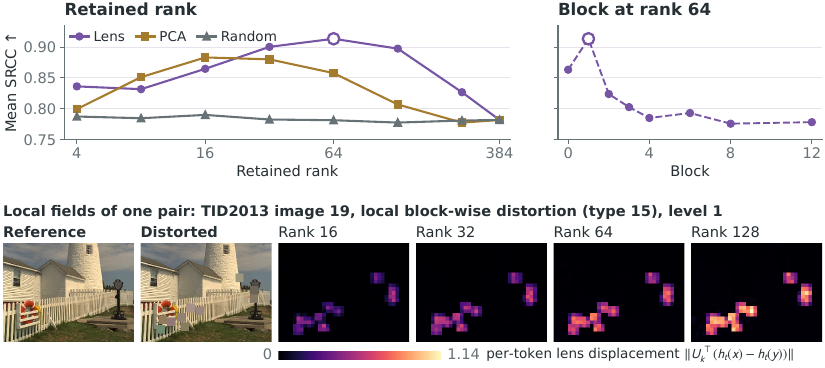}
\caption{Mean development \srcc{} over CSIQ and KADID-10k development as a function of retained rank and encoder block. Bottom: block-1 lens maps at several ranks. Diagnostic fits use the preliminary combined output.}
\label{fig:appG_sweep}
\end{figure}

The sweep considers $k\in\{4,8,16,32,64,128,256,384\}$.
Figure~\ref{fig:appG_sweep} shows that the Jacobian basis consistently outperforms PCA and random projections, with block~1 and $k=64$ selected for the final distance.
\phantomsection\label{app:same-tokens}
As an additional control, a lens pooled over natural images reaches 0.861 \srcc{} on 1{,}000 TID2013 pairs, compared with 0.679 for sensitivity estimated separately from each reference image.
\phantomsection\label{app:output-weight}
The output-term weight is fixed from CSIQ and KADID-10k development.
Matching the score spreads gives weights 0.312 and 0.696 on the two sets, whose mean 0.504 is rounded to $1/2$.
Performance is insensitive to this choice: weights from 0.25 to 1 give mean \srcc{} between 0.923 and 0.927 across the four standard image sets.
A single Euclidean norm over the concatenated lens and output features reaches 0.923, compared with 0.926 for the released sum; unlike the released sum, it satisfies the triangle inequality (Appendix~\ref{app:proofs}).

% Generated by scripts/70_tables_v5.py; do not edit by hand.
% Original generator: scripts/78_output_term_v5.py; numerical source: results/output_term.csv.
% All weights and the selection grid are fixed in scripts/78_output_term_v5.py.
% Mean: unweighted TID2013, CSIQ, LIVE and KADID test; CSIQ also participates in selection.
% Independent held-out and development means are additionally stored in results/output_term.csv.
\begin{table}[tbp]
\centering\footnotesize% v23.5 P2: was \small
\setlength{\tabcolsep}{3pt}
\renewcommand{\arraystretch}{1.03}
\caption{Moderate output weights give similar four-dataset means. Violet: selected settings; white: controls. \textbf{Bold} and \underline{underlining} mark the best and second displayed values.}
\label{tab:output-term}
\colorlet{humanJudgements}{gray!7}
\definecolor{ourMetrics}{HTML}{F3EFF8}
\begin{tabular*}{\linewidth}{@{\extracolsep{\fill}}lrrrrrr@{}}
\toprule
Form & \shortstack{TID2013\\(diag.)} & \shortstack{CSIQ\\(dev.)} & LIVE & \shortstack{KADID\\test} & \shortstack{KADID\\dev} & Mean \\
\midrule
% tid2013=0.87451394521458303; source: results/output_term.csv:8, form=lens_only, weight=0, stage=anchor, column=SRCC.
% csiq=0.95586725422217722; source: results/output_term.csv:17, form=lens_only, weight=0, stage=anchor, column=SRCC.
% live=0.94480576756493795; source: results/output_term.csv:42, form=lens_only, weight=0, stage=anchor, column=SRCC.
% kadid10k_test=0.8703966554350745; source: results/output_term.csv:51, form=lens_only, weight=0, stage=anchor, column=SRCC.
% kadid10k_dev=0.8760750331254662; source: results/output_term.csv:60, form=lens_only, weight=0, stage=anchor, column=SRCC.
% mean_four=0.91139590560919315; source: results/output_term.csv:85, form=lens_only, weight=0, stage=anchor, column=SRCC.
Lens alone ($\lambda=0$) & \underline{0.875} & 0.956 & 0.945 & 0.870 & 0.876 & 0.911 \\
% tid2013=0.87730243032329014; source: results/output_term.csv:3, form=sum, weight=0.25, stage=weight_sweep, column=SRCC.
% csiq=0.96908368717520166; source: results/output_term.csv:12, form=sum, weight=0.25, stage=weight_sweep, column=SRCC.
% live=0.9590649348592275; source: results/output_term.csv:37, form=sum, weight=0.25, stage=weight_sweep, column=SRCC.
% kadid10k_test=0.88456466742604578; source: results/output_term.csv:46, form=sum, weight=0.25, stage=weight_sweep, column=SRCC.
% kadid10k_dev=0.88705822582217131; source: results/output_term.csv:55, form=sum, weight=0.25, stage=weight_sweep, column=SRCC.
% mean_four=0.92250392994594121; source: results/output_term.csv:80, form=sum, weight=0.25, stage=weight_sweep, column=SRCC.
Sum ($\lambda=0.25$) & \textbf{0.877} & 0.969 & 0.959 & 0.885 & 0.887 & 0.923 \\
% tid2013=0.87670302622618235; source: results/output_term.csv:4, form=sum, weight=0.5, stage=weight_sweep, column=SRCC.
% csiq=0.97083565375743608; source: results/output_term.csv:13, form=sum, weight=0.5, stage=weight_sweep, column=SRCC.
% live=0.96384007519110348; source: results/output_term.csv:38, form=sum, weight=0.5, stage=weight_sweep, column=SRCC.
% kadid10k_test=0.89172659163503165; source: results/output_term.csv:47, form=sum, weight=0.5, stage=weight_sweep, column=SRCC.
% kadid10k_dev=0.89246723257839056; source: results/output_term.csv:56, form=sum, weight=0.5, stage=weight_sweep, column=SRCC.
% mean_four=0.92577633670243831; source: results/output_term.csv:81, form=sum, weight=0.5, stage=weight_sweep, column=SRCC.
\rowcolor{ourMetrics}\name{} ($\lambda=0.5$) & \textbf{0.877} & \textbf{0.971} & 0.964 & 0.892 & 0.892 & \underline{0.926} \\
% tid2013=0.87480097624699127; source: results/output_term.csv:5, form=sum, weight=0.75, stage=weight_sweep, column=SRCC.
% csiq=0.9698491742271127; source: results/output_term.csv:14, form=sum, weight=0.75, stage=weight_sweep, column=SRCC.
% live=0.96591033835829365; source: results/output_term.csv:39, form=sum, weight=0.75, stage=weight_sweep, column=SRCC.
% kadid10k_test=0.89544680645617825; source: results/output_term.csv:48, form=sum, weight=0.75, stage=weight_sweep, column=SRCC.
% kadid10k_dev=0.89521558353243846; source: results/output_term.csv:57, form=sum, weight=0.75, stage=weight_sweep, column=SRCC.
% mean_four=0.92650182382214386; source: results/output_term.csv:82, form=sum, weight=0.75, stage=weight_sweep, column=SRCC.
Sum ($\lambda=0.75$) & \underline{0.875} & \underline{0.970} & \underline{0.966} & \underline{0.895} & \underline{0.895} & \textbf{0.927} \\
% tid2013=0.87232432220229295; source: results/output_term.csv:6, form=sum, weight=1, stage=weight_sweep, column=SRCC.
% csiq=0.96810910674241579; source: results/output_term.csv:15, form=sum, weight=1, stage=weight_sweep, column=SRCC.
% live=0.96677678954135438; source: results/output_term.csv:40, form=sum, weight=1, stage=weight_sweep, column=SRCC.
% kadid10k_test=0.89723007967750412; source: results/output_term.csv:49, form=sum, weight=1, stage=weight_sweep, column=SRCC.
% kadid10k_dev=0.89617539996062379; source: results/output_term.csv:58, form=sum, weight=1, stage=weight_sweep, column=SRCC.
% mean_four=0.92611007454089178; source: results/output_term.csv:83, form=sum, weight=1, stage=weight_sweep, column=SRCC.
Sum ($\lambda=1$) & 0.872 & 0.968 & \textbf{0.967} & \textbf{0.897} & \textbf{0.896} & \underline{0.926} \\
% tid2013=0.86102889069828448; source: results/output_term.csv:7, form=sum, weight=2, stage=weight_sweep, column=SRCC.
% csiq=0.95934397097255808; source: results/output_term.csv:16, form=sum, weight=2, stage=weight_sweep, column=SRCC.
% live=0.96504911639457225; source: results/output_term.csv:41, form=sum, weight=2, stage=weight_sweep, column=SRCC.
% kadid10k_test=0.89656021001527553; source: results/output_term.csv:50, form=sum, weight=2, stage=weight_sweep, column=SRCC.
% kadid10k_dev=0.89379811738243031; source: results/output_term.csv:59, form=sum, weight=2, stage=weight_sweep, column=SRCC.
% mean_four=0.92049554702017256; source: results/output_term.csv:84, form=sum, weight=2, stage=weight_sweep, column=SRCC.
Sum ($\lambda=2$) & 0.861 & 0.959 & 0.965 & \textbf{0.897} & 0.894 & 0.920 \\
% tid2013=0.74776017592425004; source: results/output_term.csv:9, form=output_only, weight=1, stage=anchor, column=SRCC.
% csiq=0.83977329336800932; source: results/output_term.csv:18, form=output_only, weight=1, stage=anchor, column=SRCC.
% live=0.90955405154013091; source: results/output_term.csv:43, form=output_only, weight=1, stage=anchor, column=SRCC.
% kadid10k_test=0.80793411628156675; source: results/output_term.csv:52, form=output_only, weight=1, stage=anchor, column=SRCC.
% kadid10k_dev=0.81187253936099912; source: results/output_term.csv:61, form=output_only, weight=1, stage=anchor, column=SRCC.
% mean_four=0.82625540927848928; source: results/output_term.csv:86, form=output_only, weight=1, stage=anchor, column=SRCC.
Output alone & 0.748 & 0.840 & 0.910 & 0.808 & 0.812 & 0.826 \\
% tid2013=0.8713716412105873; source: results/output_term.csv:10, form=single_norm, weight=0.375, stage=selected, column=SRCC.
% csiq=0.96649800773724048; source: results/output_term.csv:19, form=single_norm, weight=0.375, stage=selected, column=SRCC.
% live=0.96268002618824133; source: results/output_term.csv:44, form=single_norm, weight=0.375, stage=selected, column=SRCC.
% kadid10k_test=0.89339452403089614; source: results/output_term.csv:53, form=single_norm, weight=0.375, stage=selected, column=SRCC.
% kadid10k_dev=0.89336624168257794; source: results/output_term.csv:62, form=single_norm, weight=0.375, stage=selected, column=SRCC.
% mean_four=0.92348604979174131; source: results/output_term.csv:87, form=single_norm, weight=0.375, stage=selected, column=SRCC.
\rowcolor{ourMetrics}Single norm ($w=0.375$) & 0.871 & 0.966 & 0.963 & 0.893 & 0.893 & 0.923 \\
\bottomrule
\end{tabular*}
% \par\vspace{3pt}\begin{minipage}{\linewidth}\footnotesize Higher \srcc{} is better. \textbf{Bold}: best; \underline{underlined}: second distinct displayed value; ties share marks. The sum is $D_{\mathrm{lens}}+\lambda\Dout$. The single norm is $\sqrt{D_{\mathrm{lens}}^2+w^2\lVert\hat{\out}(\pre x)-\hat{\out}(\pre y)\rVert_2^2}$, with unit-normalised CLS embeddings. Only CSIQ and KADID development select $w$; the mean uses four standard columns, excluding KADID development. Violet marks the frozen sum and selected single norm.\end{minipage}
% v23.6 P6: shorter note.
\end{table}

\Needspace{4\baselineskip}
\section{What the lens sees in gratings}\label{app:csf}

We probe \name{} with classical psychophysical stimuli to ask whether a lens fitted without human judgements recovers known properties of visual sensitivity.
Following contrast-sensitivity experiments \citep{campbell1968csf,mullen1985}, we add sinusoidal gratings that vary in spatial frequency, colour direction, orientation, and background texture, and compare the lens with raw block-1 features.
\begin{figure}[t]
\centering
\includegraphics[width=\linewidth]{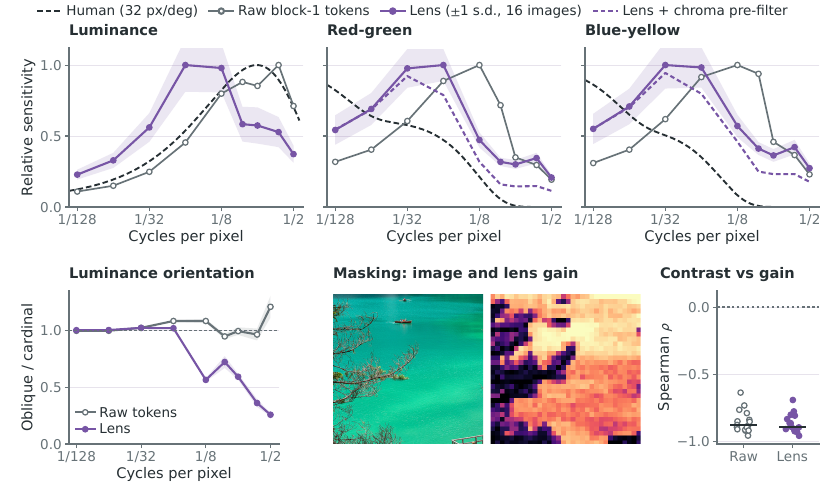}
\caption{Grating probes of the lens on 16 DIV2K crops. Top: peak-normalised sensitivity across spatial frequency for luminance, red--green, and blue--yellow gratings, compared with human references. Bottom left: sensitivity to oblique relative to cardinal gratings. Bottom right: local lens gain for a luminance grating (bright denotes higher sensitivity).}
\label{fig:csf}
\end{figure}
Figure~\ref{fig:csf} shows three effects.
First, the lens suppresses high-frequency chromatic changes substantially more than raw block-1 features; the fixed chroma pre-filter moves this response further toward human chromatic sensitivity \citep{mullen1985,zhang1996scielab}.
Second, the lens is less sensitive to fine oblique gratings than to horizontal and vertical ones, reproducing the qualitative oblique effect that is absent from the raw features.
Third, its local gain decreases in textured regions: for a $1/8$ cycle-per-pixel luminance grating, gain has median Spearman correlation $-0.89$ with local luminance contrast across the 16 images, consistent with visual masking.
The agreement is not complete.
In particular, the lens does not reproduce the human luminance contrast-sensitivity function: its response peaks at a lower spatial frequency.
The unfiltered lens also remains too sensitive to the finest chromatic patterns, motivating the fixed $\sigma=2$ chroma pre-filter described in Appendix~\ref{app:md-prefilter}.
Thus, the frozen encoder already contains several human-like sensitivity patterns, while the pre-filter corrects its clearest systematic mismatch.

\Needspace{4\baselineskip}
\section{Validation and human study}\label{app:validation}

We provide additional checks of the geometry of \name{}, a small human study on cases where perceptual metrics disagree, and targeted perturbations for directions retained or discarded by the lens.

\subsection{Human study}\label{app:human-study}

\begin{figure}[t]
\centering
\includegraphics[width=\linewidth]{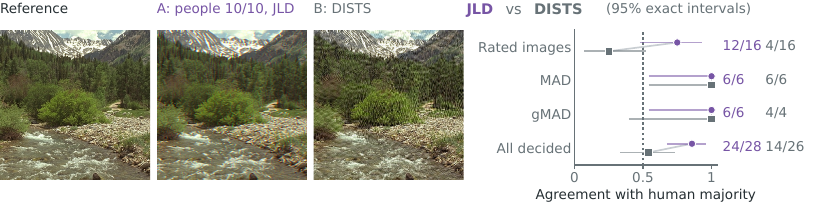}
% Source: scripts/v23/app_study.py (from scripts/v22/fig_study.py). Example: trial real_tid2013_00, pixels figures/v18/fig8_study.inputs.npz.
% Rates and exact intervals: figures/v23/app_study_numbers.json; totals match results/study_summary.csv (jld 24/28, dists 14/26).
% \caption{Human-study agreement.
% Left: one rated-image trial in which all ten participants and \name{} chose A, while DISTS chose B.
% Right: fraction of decided trials on which each metric agrees with the human majority, with exact 95\% binomial intervals; the dotted line is chance.
% On the 16 rated-image trials \name{} agrees on 12 and DISTS on 4, and their intervals overlap only between 0.48 and 0.52.
% Figure~\ref{fig:study-coverage} shows all 40 trials.}
\caption{Human-study agreement. Left: one rated-image trial in which all ten participants and \name{} chose A and DISTS chose B. Right: share of decided trials on which each metric agrees with the human majority (exact 95\% binomial intervals; dotted: chance).}
\label{fig:study}
\end{figure}

% The full 40-trial grid is omitted from the final appendix.
% \IfFileExists{sections/figs/app_fig_study_trials.tex}{\input{sections/figs/app_fig_study_trials}}{}

Ten participants completed 40 trials selected where perceptual metrics disagree, choosing which distorted image was closer to its reference, with a ``cannot tell'' option and randomised sides.
Across trials with an untied majority, the lens agrees with the majority on $24/28$ decisions, compared with $14/26$ for DISTS.
On the 16 trials drawn from rated image pairs, the majority agrees with MOS on all 16; \name{} agrees on $12/16$ and DISTS on $4/16$.
% Because these trials were intentionally selected for disagreement, these numbers are descriptive rather than estimates of overall accuracy.

\subsection{Metamers and targeted perturbations}

A displacement orthogonal to the lens has zero projected distance, providing a direct way to probe what the metric ignores.
For eight references, we construct approximate metamers, lens-maximising perturbations, and random perturbations at the same nominal 40\,dB PSNR.
\begin{figure}[t]
\centering
\includegraphics[width=\linewidth]{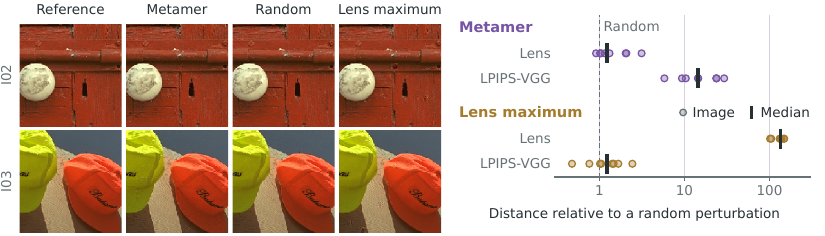}
\caption{Approximate metamers, random perturbations, and lens-maximising perturbations at nominal 40\,dB PSNR. Right: distances relative to the same-reference random perturbation; the dashed line denotes 1.}
\label{fig:appI-metamers}
\end{figure}
As shown in Figure~\ref{fig:appI-metamers}, approximate metamers remain close under the unfiltered lens (median ratio $1.23$ relative to random perturbations) while LPIPS-VGG rises to $14.5$.
Conversely, lens-maximising perturbations increase the lens distance by a factor of $134$, while LPIPS-VGG changes by only $1.23$. These results confirm that the lens is strongly directional, it suppresses feature changes outside the output-sensitive subspace while strongly responding to changes aligned with it.

\Needspace{4\baselineskip}
\section{Encoders}\label{app:encoders}

We apply the same calibration procedure to 14 frozen encoders spanning DINOv2, DINOv3, MAE, CLIP, SigLIP~2, VGG16, and ConvNeXt
\citep{oquab2024dinov2,simeoni2025dinov3,he2022mae,radford2021clip,tschannen2025siglip2,simonyan2015vgg,liu2022convnext}.
Each encoder receives its own lens, fitted on the same unlabelled DIV2K corpus, with block and rank selected on CSIQ and KADID-10k development.
As shown in Table~\ref{tab:app_encoder_gain} and Figure~\ref{fig:B2-ablation-full}d, the lens improves over raw features for all 14 encoders and over PCA~64 for 13 of 14.
DINOv2-S gives the strongest development result and is therefore used as the default encoder.
% Generated by scripts/70_tables_v5.py; do not edit by hand.
% Selection: scripts/09_observers.py:500-503; validated against results/observers_choice.csv.
% Mean: arithmetic mean of csiq_srcc and kadid10k_dev_srcc in results/observers_dev.csv.
\begin{table}[tbp]
\centering\footnotesize% v23.5 P2: was \small
\setlength{\tabcolsep}{3pt}
\renewcommand{\arraystretch}{1.03}
\caption{Development-selected projections across encoder families. Violet: lens; white: controls. \textbf{Bold} and \underline{underlining} mark the best and second displayed values in each comparison.}
\label{tab:app_encoder_gain}
\colorlet{humanJudgements}{gray!7}
\definecolor{ourMetrics}{HTML}{F3EFF8}
\begin{tabular*}{\linewidth}{@{\extracolsep{\fill}}lrr>{\columncolor{ourMetrics}}rrrrrrr@{}}
\toprule
& & & \multicolumn{5}{c}{Selected block} & \multicolumn{2}{c}{Best block} \\
\cmidrule(lr){4-8}\cmidrule(l){9-10}
Encoder & Block & Rank & Lens & Raw & PCA 64 & $\Delta$ raw & $\Delta$ PCA & Raw & PCA 64 \\
\midrule
% Cell Block=1; exact=1; source: results/observers_dev.csv:11, encoder=dinov2_s14, column=layer.
% Cell Rank=64; exact=64; source: results/observers_dev.csv:11, column=subspace, value=lens-64.
% Cell Lens=0.913; exact=0.9134610462195684; source: results/observers_dev.csv:11, encoder=dinov2_s14, block=1, subspace=lens-64, column=mean_srcc.
% Cell Raw=0.782; exact=0.7816913383510065; source: results/observers_dev.csv:15, encoder=dinov2_s14, block=1, subspace=full, column=mean_srcc.
% Cell PCA 64=0.858; exact=0.8577655268640465; source: results/observers_dev.csv:13, encoder=dinov2_s14, block=1, subspace=pca-64, column=mean_srcc.
% Cell $\Delta$ raw=+0.132; exact=0.1317697078685619; source: difference of results/observers_dev.csv:11, encoder=dinov2_s14, block=1, subspace=lens-64, column=mean_srcc minus results/observers_dev.csv:15, encoder=dinov2_s14, block=1, subspace=full, column=mean_srcc.
% Cell $\Delta$ PCA=+0.056; exact=0.0556955193555219; source: difference of results/observers_dev.csv:11, encoder=dinov2_s14, block=1, subspace=lens-64, column=mean_srcc minus results/observers_dev.csv:13, encoder=dinov2_s14, block=1, subspace=pca-64, column=mean_srcc.
% Cell Raw=0.807; exact=0.8069192918338622; source: results/observers_dev.csv:22, encoder=dinov2_s14, block=2, subspace=full, column=mean_srcc.
% Cell PCA 64=0.858; exact=0.8577655268640465; source: results/observers_dev.csv:13, encoder=dinov2_s14, block=1, subspace=pca-64, column=mean_srcc.
DINOv2-S & 1 & 64 & \textbf{0.913} & 0.782 & 0.858 & \textbf{+0.132} & \underline{+0.056} & 0.807 & 0.858 \\
% Cell Block=1; exact=1; source: results/observers_dev.csv:81, encoder=dinov2_b14, column=layer.
% Cell Rank=64; exact=64; source: results/observers_dev.csv:81, column=subspace, value=lens-64.
% Cell Lens=0.908; exact=0.90767212553855; source: results/observers_dev.csv:81, encoder=dinov2_b14, block=1, subspace=lens-64, column=mean_srcc.
% Cell Raw=0.804; exact=0.8035245905274508; source: results/observers_dev.csv:85, encoder=dinov2_b14, block=1, subspace=full, column=mean_srcc.
% Cell PCA 64=0.859; exact=0.8591266505350661; source: results/observers_dev.csv:83, encoder=dinov2_b14, block=1, subspace=pca-64, column=mean_srcc.
% Cell $\Delta$ raw=+0.104; exact=0.1041475350110992; source: difference of results/observers_dev.csv:81, encoder=dinov2_b14, block=1, subspace=lens-64, column=mean_srcc minus results/observers_dev.csv:85, encoder=dinov2_b14, block=1, subspace=full, column=mean_srcc.
% Cell $\Delta$ PCA=+0.049; exact=0.0485454750034839; source: difference of results/observers_dev.csv:81, encoder=dinov2_b14, block=1, subspace=lens-64, column=mean_srcc minus results/observers_dev.csv:83, encoder=dinov2_b14, block=1, subspace=pca-64, column=mean_srcc.
% Cell Raw=0.804; exact=0.8035245905274508; source: results/observers_dev.csv:85, encoder=dinov2_b14, block=1, subspace=full, column=mean_srcc.
% Cell PCA 64=0.859; exact=0.8591266505350661; source: results/observers_dev.csv:83, encoder=dinov2_b14, block=1, subspace=pca-64, column=mean_srcc.
DINOv2-B & 1 & 64 & \underline{0.908} & 0.804 & 0.859 & +0.104 & +0.049 & 0.804 & 0.859 \\
% Cell Block=4; exact=4; source: results/observers_dev.csv:408, encoder=dinov2_l14, column=layer.
% Cell Rank=32; exact=32; source: results/observers_dev.csv:408, column=subspace, value=lens-32.
% Cell Lens=0.906; exact=0.9058260909628923; source: results/observers_dev.csv:408, encoder=dinov2_l14, block=4, subspace=lens-32, column=mean_srcc.
% Cell Raw=0.829; exact=0.8288388514392658; source: results/observers_dev.csv:413, encoder=dinov2_l14, block=4, subspace=full, column=mean_srcc.
% Cell PCA 64=0.891; exact=0.8911082968572999; source: results/observers_dev.csv:411, encoder=dinov2_l14, block=4, subspace=pca-64, column=mean_srcc.
% Cell $\Delta$ raw=+0.077; exact=0.0769872395236265; source: difference of results/observers_dev.csv:408, encoder=dinov2_l14, block=4, subspace=lens-32, column=mean_srcc minus results/observers_dev.csv:413, encoder=dinov2_l14, block=4, subspace=full, column=mean_srcc.
% Cell $\Delta$ PCA=+0.015; exact=0.0147177941055924; source: difference of results/observers_dev.csv:408, encoder=dinov2_l14, block=4, subspace=lens-32, column=mean_srcc minus results/observers_dev.csv:411, encoder=dinov2_l14, block=4, subspace=pca-64, column=mean_srcc.
% Cell Raw=0.829; exact=0.8288388514392658; source: results/observers_dev.csv:413, encoder=dinov2_l14, block=4, subspace=full, column=mean_srcc.
% Cell PCA 64=0.895; exact=0.8945244594152086; source: results/observers_dev.csv:404, encoder=dinov2_l14, block=3, subspace=pca-64, column=mean_srcc.
DINOv2-L & 4 & 32 & 0.906 & 0.829 & \textbf{0.891} & +0.077 & +0.015 & 0.829 & \textbf{0.895} \\
% Cell Block=4; exact=4; source: results/observers_dev.csv:477, encoder=dinov2_g14, column=layer.
% Cell Rank=16; exact=16; source: results/observers_dev.csv:477, column=subspace, value=lens-16.
% Cell Lens=0.896; exact=0.8964000836515483; source: results/observers_dev.csv:477, encoder=dinov2_g14, block=4, subspace=lens-16, column=mean_srcc.
% Cell Raw=0.785; exact=0.784804574797078; source: results/observers_dev.csv:483, encoder=dinov2_g14, block=4, subspace=full, column=mean_srcc.
% Cell PCA 64=0.875; exact=0.8748609713221498; source: results/observers_dev.csv:481, encoder=dinov2_g14, block=4, subspace=pca-64, column=mean_srcc.
% Cell $\Delta$ raw=+0.112; exact=0.1115955088544703; source: difference of results/observers_dev.csv:477, encoder=dinov2_g14, block=4, subspace=lens-16, column=mean_srcc minus results/observers_dev.csv:483, encoder=dinov2_g14, block=4, subspace=full, column=mean_srcc.
% Cell $\Delta$ PCA=+0.022; exact=0.0215391123293985; source: difference of results/observers_dev.csv:477, encoder=dinov2_g14, block=4, subspace=lens-16, column=mean_srcc minus results/observers_dev.csv:481, encoder=dinov2_g14, block=4, subspace=pca-64, column=mean_srcc.
% Cell Raw=0.795; exact=0.794885043959134; source: results/observers_dev.csv:455, encoder=dinov2_g14, block=0, subspace=full, column=mean_srcc.
% Cell PCA 64=0.875; exact=0.8748609713221498; source: results/observers_dev.csv:481, encoder=dinov2_g14, block=4, subspace=pca-64, column=mean_srcc.
DINOv2-G & 4 & 16 & 0.896 & 0.785 & \underline{0.875} & \underline{+0.112} & +0.022 & 0.795 & 0.875 \\
% Cell Block=1; exact=1; source: results/observers_dev.csv:116, encoder=dinov2_s14_reg, column=layer.
% Cell Rank=64; exact=64; source: results/observers_dev.csv:116, column=subspace, value=lens-64.
% Cell Lens=0.902; exact=0.9019430022853868; source: results/observers_dev.csv:116, encoder=dinov2_s14_reg, block=1, subspace=lens-64, column=mean_srcc.
% Cell Raw=0.818; exact=0.8176381038880336; source: results/observers_dev.csv:120, encoder=dinov2_s14_reg, block=1, subspace=full, column=mean_srcc.
% Cell PCA 64=0.871; exact=0.870532352709299; source: results/observers_dev.csv:118, encoder=dinov2_s14_reg, block=1, subspace=pca-64, column=mean_srcc.
% Cell $\Delta$ raw=+0.084; exact=0.0843048983973532; source: difference of results/observers_dev.csv:116, encoder=dinov2_s14_reg, block=1, subspace=lens-64, column=mean_srcc minus results/observers_dev.csv:120, encoder=dinov2_s14_reg, block=1, subspace=full, column=mean_srcc.
% Cell $\Delta$ PCA=+0.031; exact=0.0314106495760878; source: difference of results/observers_dev.csv:116, encoder=dinov2_s14_reg, block=1, subspace=lens-64, column=mean_srcc minus results/observers_dev.csv:118, encoder=dinov2_s14_reg, block=1, subspace=pca-64, column=mean_srcc.
% Cell Raw=0.820; exact=0.8197038648401114; source: results/observers_dev.csv:134, encoder=dinov2_s14_reg, block=3, subspace=full, column=mean_srcc.
% Cell PCA 64=0.871; exact=0.870532352709299; source: results/observers_dev.csv:118, encoder=dinov2_s14_reg, block=1, subspace=pca-64, column=mean_srcc.
DINOv2-S-reg & 1 & 64 & 0.902 & 0.818 & 0.871 & +0.084 & +0.031 & 0.820 & 0.871 \\
% Cell Block=1; exact=1; source: results/observers_dev.csv:290, encoder=dinov2_b14_reg, column=layer.
% Cell Rank=32; exact=32; source: results/observers_dev.csv:290, column=subspace, value=lens-32.
% Cell Lens=0.898; exact=0.8976951684373196; source: results/observers_dev.csv:290, encoder=dinov2_b14_reg, block=1, subspace=lens-32, column=mean_srcc.
% Cell Raw=0.831; exact=0.8305035156326993; source: results/observers_dev.csv:295, encoder=dinov2_b14_reg, block=1, subspace=full, column=mean_srcc.
% Cell PCA 64=0.873; exact=0.8734005283234533; source: results/observers_dev.csv:293, encoder=dinov2_b14_reg, block=1, subspace=pca-64, column=mean_srcc.
% Cell $\Delta$ raw=+0.067; exact=0.0671916528046203; source: difference of results/observers_dev.csv:290, encoder=dinov2_b14_reg, block=1, subspace=lens-32, column=mean_srcc minus results/observers_dev.csv:295, encoder=dinov2_b14_reg, block=1, subspace=full, column=mean_srcc.
% Cell $\Delta$ PCA=+0.024; exact=0.0242946401138663; source: difference of results/observers_dev.csv:290, encoder=dinov2_b14_reg, block=1, subspace=lens-32, column=mean_srcc minus results/observers_dev.csv:293, encoder=dinov2_b14_reg, block=1, subspace=pca-64, column=mean_srcc.
% Cell Raw=0.831; exact=0.8305035156326993; source: results/observers_dev.csv:295, encoder=dinov2_b14_reg, block=1, subspace=full, column=mean_srcc.
% Cell PCA 64=0.873; exact=0.8734005283234533; source: results/observers_dev.csv:293, encoder=dinov2_b14_reg, block=1, subspace=pca-64, column=mean_srcc.
DINOv2-B-reg & 1 & 32 & 0.898 & \underline{0.831} & 0.873 & +0.067 & +0.024 & 0.831 & 0.873 \\
% Cell Block=1; exact=1; source: results/observers_dev.csv:46, encoder=dinov3_s16, column=layer.
% Cell Rank=64; exact=64; source: results/observers_dev.csv:46, column=subspace, value=lens-64.
% Cell Lens=0.871; exact=0.8714460690434627; source: results/observers_dev.csv:46, encoder=dinov3_s16, block=1, subspace=lens-64, column=mean_srcc.
% Cell Raw=0.805; exact=0.8053623350800673; source: results/observers_dev.csv:50, encoder=dinov3_s16, block=1, subspace=full, column=mean_srcc.
% Cell PCA 64=0.825; exact=0.8249233689531617; source: results/observers_dev.csv:48, encoder=dinov3_s16, block=1, subspace=pca-64, column=mean_srcc.
% Cell $\Delta$ raw=+0.066; exact=0.0660837339633954; source: difference of results/observers_dev.csv:46, encoder=dinov3_s16, block=1, subspace=lens-64, column=mean_srcc minus results/observers_dev.csv:50, encoder=dinov3_s16, block=1, subspace=full, column=mean_srcc.
% Cell $\Delta$ PCA=+0.047; exact=0.0465227000903010; source: difference of results/observers_dev.csv:46, encoder=dinov3_s16, block=1, subspace=lens-64, column=mean_srcc minus results/observers_dev.csv:48, encoder=dinov3_s16, block=1, subspace=pca-64, column=mean_srcc.
% Cell Raw=0.816; exact=0.8155137365107519; source: results/observers_dev.csv:57, encoder=dinov3_s16, block=2, subspace=full, column=mean_srcc.
% Cell PCA 64=0.828; exact=0.8276593790554754; source: results/observers_dev.csv:41, encoder=dinov3_s16, block=0, subspace=pca-64, column=mean_srcc.
DINOv3-S & 1 & 64 & 0.871 & 0.805 & 0.825 & +0.066 & +0.047 & 0.816 & 0.828 \\
% Cell Block=2; exact=2; source: results/observers_dev.csv:156, encoder=dinov3_b16, column=layer.
% Cell Rank=16; exact=16; source: results/observers_dev.csv:156, column=subspace, value=lens-16.
% Cell Lens=0.878; exact=0.8779988532822542; source: results/observers_dev.csv:156, encoder=dinov3_b16, block=2, subspace=lens-16, column=mean_srcc.
% Cell Raw=0.824; exact=0.8237201188925876; source: results/observers_dev.csv:162, encoder=dinov3_b16, block=2, subspace=full, column=mean_srcc.
% Cell PCA 64=0.825; exact=0.8252958261325654; source: results/observers_dev.csv:160, encoder=dinov3_b16, block=2, subspace=pca-64, column=mean_srcc.
% Cell $\Delta$ raw=+0.054; exact=0.0542787343896666; source: difference of results/observers_dev.csv:156, encoder=dinov3_b16, block=2, subspace=lens-16, column=mean_srcc minus results/observers_dev.csv:162, encoder=dinov3_b16, block=2, subspace=full, column=mean_srcc.
% Cell $\Delta$ PCA=+0.053; exact=0.0527030271496888; source: difference of results/observers_dev.csv:156, encoder=dinov3_b16, block=2, subspace=lens-16, column=mean_srcc minus results/observers_dev.csv:160, encoder=dinov3_b16, block=2, subspace=pca-64, column=mean_srcc.
% Cell Raw=0.842; exact=0.8417102137703626; source: results/observers_dev.csv:155, encoder=dinov3_b16, block=1, subspace=full, column=mean_srcc.
% Cell PCA 64=0.846; exact=0.8463090113461259; source: results/observers_dev.csv:146, encoder=dinov3_b16, block=0, subspace=pca-64, column=mean_srcc.
DINOv3-B & 2 & 16 & 0.878 & 0.824 & 0.825 & +0.054 & +0.053 & \underline{0.842} & 0.846 \\
% Cell Block=3; exact=3; source: results/observers_dev.csv:436, encoder=dinov3_l16, column=layer.
% Cell Rank=32; exact=32; source: results/observers_dev.csv:436, column=subspace, value=lens-32.
% Cell Lens=0.883; exact=0.8831888800992773; source: results/observers_dev.csv:436, encoder=dinov3_l16, block=3, subspace=lens-32, column=mean_srcc.
% Cell Raw=0.851; exact=0.8506501477115216; source: results/observers_dev.csv:441, encoder=dinov3_l16, block=3, subspace=full, column=mean_srcc.
% Cell PCA 64=0.845; exact=0.8451871539336602; source: results/observers_dev.csv:439, encoder=dinov3_l16, block=3, subspace=pca-64, column=mean_srcc.
% Cell $\Delta$ raw=+0.033; exact=0.0325387323877557; source: difference of results/observers_dev.csv:436, encoder=dinov3_l16, block=3, subspace=lens-32, column=mean_srcc minus results/observers_dev.csv:441, encoder=dinov3_l16, block=3, subspace=full, column=mean_srcc.
% Cell $\Delta$ PCA=+0.038; exact=0.0380017261656171; source: difference of results/observers_dev.csv:436, encoder=dinov3_l16, block=3, subspace=lens-32, column=mean_srcc minus results/observers_dev.csv:439, encoder=dinov3_l16, block=3, subspace=pca-64, column=mean_srcc.
% Cell Raw=0.857; exact=0.8566889325072955; source: results/observers_dev.csv:448, encoder=dinov3_l16, block=4, subspace=full, column=mean_srcc.
% Cell PCA 64=0.859; exact=0.8585607604442788; source: results/observers_dev.csv:446, encoder=dinov3_l16, block=4, subspace=pca-64, column=mean_srcc.
DINOv3-L & 3 & 32 & 0.883 & \textbf{0.851} & 0.845 & +0.033 & +0.038 & \textbf{0.857} & 0.859 \\
% Cell Block=2; exact=2; source: results/observers_dev.csv:193, encoder=mae_b16, column=layer.
% Cell Rank=64; exact=64; source: results/observers_dev.csv:193, column=subspace, value=lens-64.
% Cell Lens=0.884; exact=0.8841655747091626; source: results/observers_dev.csv:193, encoder=mae_b16, block=2, subspace=lens-64, column=mean_srcc.
% Cell Raw=0.791; exact=0.7911994131210978; source: results/observers_dev.csv:197, encoder=mae_b16, block=2, subspace=full, column=mean_srcc.
% Cell PCA 64=0.850; exact=0.8500426586383777; source: results/observers_dev.csv:195, encoder=mae_b16, block=2, subspace=pca-64, column=mean_srcc.
% Cell $\Delta$ raw=+0.093; exact=0.0929661615880648; source: difference of results/observers_dev.csv:193, encoder=mae_b16, block=2, subspace=lens-64, column=mean_srcc minus results/observers_dev.csv:197, encoder=mae_b16, block=2, subspace=full, column=mean_srcc.
% Cell $\Delta$ PCA=+0.034; exact=0.0341229160707849; source: difference of results/observers_dev.csv:193, encoder=mae_b16, block=2, subspace=lens-64, column=mean_srcc minus results/observers_dev.csv:195, encoder=mae_b16, block=2, subspace=pca-64, column=mean_srcc.
% Cell Raw=0.833; exact=0.8333629262230017; source: results/observers_dev.csv:183, encoder=mae_b16, block=0, subspace=full, column=mean_srcc.
% Cell PCA 64=0.868; exact=0.8681869500794418; source: results/observers_dev.csv:188, encoder=mae_b16, block=1, subspace=pca-64, column=mean_srcc.
MAE-B & 2 & 64 & 0.884 & 0.791 & 0.850 & +0.093 & +0.034 & 0.833 & 0.868 \\
% Cell Block=1; exact=1; source: results/observers_dev.csv:219, encoder=clip_b16, column=layer.
% Cell Rank=16; exact=16; source: results/observers_dev.csv:219, column=subspace, value=lens-16.
% Cell Lens=0.828; exact=0.8276169700049023; source: results/observers_dev.csv:219, encoder=clip_b16, block=1, subspace=lens-16, column=mean_srcc.
% Cell Raw=0.726; exact=0.7255877946196408; source: results/observers_dev.csv:225, encoder=clip_b16, block=1, subspace=full, column=mean_srcc.
% Cell PCA 64=0.850; exact=0.8499599133219813; source: results/observers_dev.csv:223, encoder=clip_b16, block=1, subspace=pca-64, column=mean_srcc.
% Cell $\Delta$ raw=+0.102; exact=0.1020291753852615; source: difference of results/observers_dev.csv:219, encoder=clip_b16, block=1, subspace=lens-16, column=mean_srcc minus results/observers_dev.csv:225, encoder=clip_b16, block=1, subspace=full, column=mean_srcc.
% Cell $\Delta$ PCA=-0.022; exact=-0.0223429433170790; source: difference of results/observers_dev.csv:219, encoder=clip_b16, block=1, subspace=lens-16, column=mean_srcc minus results/observers_dev.csv:223, encoder=clip_b16, block=1, subspace=pca-64, column=mean_srcc.
% Cell Raw=0.787; exact=0.7871418536363521; source: results/observers_dev.csv:218, encoder=clip_b16, block=0, subspace=full, column=mean_srcc.
% Cell PCA 64=0.883; exact=0.8831038567055787; source: results/observers_dev.csv:216, encoder=clip_b16, block=0, subspace=pca-64, column=mean_srcc.
CLIP-B & 1 & 16 & 0.828 & 0.726 & 0.850 & +0.102 & -0.022 & 0.787 & \underline{0.883} \\
% Cell Block=4; exact=4; source: results/observers_dev.csv:277, encoder=siglip2_b16, column=layer.
% Cell Rank=64; exact=64; source: results/observers_dev.csv:277, column=subspace, value=lens-64.
% Cell Lens=0.859; exact=0.8590409886874242; source: results/observers_dev.csv:277, encoder=siglip2_b16, block=4, subspace=lens-64, column=mean_srcc.
% Cell Raw=0.827; exact=0.8267439866187568; source: results/observers_dev.csv:281, encoder=siglip2_b16, block=4, subspace=full, column=mean_srcc.
% Cell PCA 64=0.803; exact=0.802768865941238; source: results/observers_dev.csv:279, encoder=siglip2_b16, block=4, subspace=pca-64, column=mean_srcc.
% Cell $\Delta$ raw=+0.032; exact=0.0322970020686674; source: difference of results/observers_dev.csv:277, encoder=siglip2_b16, block=4, subspace=lens-64, column=mean_srcc minus results/observers_dev.csv:281, encoder=siglip2_b16, block=4, subspace=full, column=mean_srcc.
% Cell $\Delta$ PCA=+0.056; exact=0.0562721227461862; source: difference of results/observers_dev.csv:277, encoder=siglip2_b16, block=4, subspace=lens-64, column=mean_srcc minus results/observers_dev.csv:279, encoder=siglip2_b16, block=4, subspace=pca-64, column=mean_srcc.
% Cell Raw=0.827; exact=0.8267439866187568; source: results/observers_dev.csv:281, encoder=siglip2_b16, block=4, subspace=full, column=mean_srcc.
% Cell PCA 64=0.803; exact=0.802768865941238; source: results/observers_dev.csv:279, encoder=siglip2_b16, block=4, subspace=pca-64, column=mean_srcc.
SigLIP2-B & 4 & 64 & 0.859 & 0.827 & 0.803 & +0.032 & \underline{+0.056} & 0.827 & 0.803 \\
% Cell Block=3; exact=3; source: results/observers_dev.csv:330, encoder=vgg16, column=layer.
% Cell Rank=16; exact=16; source: results/observers_dev.csv:330, column=subspace, value=lens-16.
% Cell Lens=0.813; exact=0.8133746746639131; source: results/observers_dev.csv:330, encoder=vgg16, block=3, subspace=lens-16, column=mean_srcc.
% Cell Raw=0.789; exact=0.7886826663372051; source: results/observers_dev.csv:336, encoder=vgg16, block=3, subspace=full, column=mean_srcc.
% Cell PCA 64=0.806; exact=0.8062867471835784; source: results/observers_dev.csv:334, encoder=vgg16, block=3, subspace=pca-64, column=mean_srcc.
% Cell $\Delta$ raw=+0.025; exact=0.0246920083267080; source: difference of results/observers_dev.csv:330, encoder=vgg16, block=3, subspace=lens-16, column=mean_srcc minus results/observers_dev.csv:336, encoder=vgg16, block=3, subspace=full, column=mean_srcc.
% Cell $\Delta$ PCA=+0.007; exact=0.0070879274803347; source: difference of results/observers_dev.csv:330, encoder=vgg16, block=3, subspace=lens-16, column=mean_srcc minus results/observers_dev.csv:334, encoder=vgg16, block=3, subspace=pca-64, column=mean_srcc.
% Cell Raw=0.802; exact=0.8016980003741871; source: results/observers_dev.csv:322, encoder=vgg16, block=1, subspace=full, column=mean_srcc.
% Cell PCA 64=0.813; exact=0.8130648354418535; source: results/observers_dev.csv:341, encoder=vgg16, block=4, subspace=pca-64, column=mean_srcc.
VGG16 & 3 & 16 & 0.813 & 0.789 & 0.806 & +0.025 & +0.007 & 0.802 & 0.813 \\
% Cell Block=1; exact=1; source: results/observers_dev.csv:351, encoder=convnext_b, column=layer.
% Cell Rank=16; exact=16; source: results/observers_dev.csv:351, column=subspace, value=lens-16.
% Cell Lens=0.765; exact=0.7646297158720263; source: results/observers_dev.csv:351, encoder=convnext_b, block=1, subspace=lens-16, column=mean_srcc.
% Cell Raw=0.693; exact=0.6932643509564664; source: results/observers_dev.csv:357, encoder=convnext_b, block=1, subspace=full, column=mean_srcc.
% Cell PCA 64=0.697; exact=0.6970274011695237; source: results/observers_dev.csv:355, encoder=convnext_b, block=1, subspace=pca-64, column=mean_srcc.
% Cell $\Delta$ raw=+0.071; exact=0.0713653649155599; source: difference of results/observers_dev.csv:351, encoder=convnext_b, block=1, subspace=lens-16, column=mean_srcc minus results/observers_dev.csv:357, encoder=convnext_b, block=1, subspace=full, column=mean_srcc.
% Cell $\Delta$ PCA=+0.068; exact=0.0676023147025026; source: difference of results/observers_dev.csv:351, encoder=convnext_b, block=1, subspace=lens-16, column=mean_srcc minus results/observers_dev.csv:355, encoder=convnext_b, block=1, subspace=pca-64, column=mean_srcc.
% Cell Raw=0.764; exact=0.7639997155284018; source: results/observers_dev.csv:378, encoder=convnext_b, block=4, subspace=full, column=mean_srcc.
% Cell PCA 64=0.760; exact=0.760109083177969; source: results/observers_dev.csv:376, encoder=convnext_b, block=4, subspace=pca-64, column=mean_srcc.
ConvNeXt-B & 1 & 16 & 0.765 & 0.693 & 0.697 & +0.071 & \textbf{+0.068} & 0.764 & 0.760 \\
\bottomrule
\end{tabular*}
% \par\vspace{3pt}\begin{minipage}{\linewidth}\footnotesize Mean development \srcc{} on CSIQ and KADID-10k; higher is better. \textbf{Bold}: best; \underline{underlined}: second distinct displayed value; ties share marks. Choose each lens over blocks and ranks; exact ties favour the earlier block, then smaller rank. Raw and PCA 64 use that block; lens and PCA ranks can differ. Gains subtract their scores from the lens score. The final columns maximise each control across measured blocks. These development comparisons do not isolate equal-rank generalisation.\end{minipage}
% v23.6 P6: shorter note.
\end{table}

The same trend extends to video. 
At block~2, VideoMAE K400 reaches \srcc{} 0.786 with lens~16, compared with 0.771 for PCA~64 and 0.710 for full features.
A new encoder therefore requires only a frozen output and an intermediate feature representation; the estimator of Appendix~\ref{app:fit} can then be used to fit its lens.

\Needspace{4\baselineskip}
\section{Video}\label{app:video}
% arXiv layout: moved up from the end of Appendix L.1 (same figure order).
\begin{figure}[t]
\centering
\includegraphics[width=\linewidth]{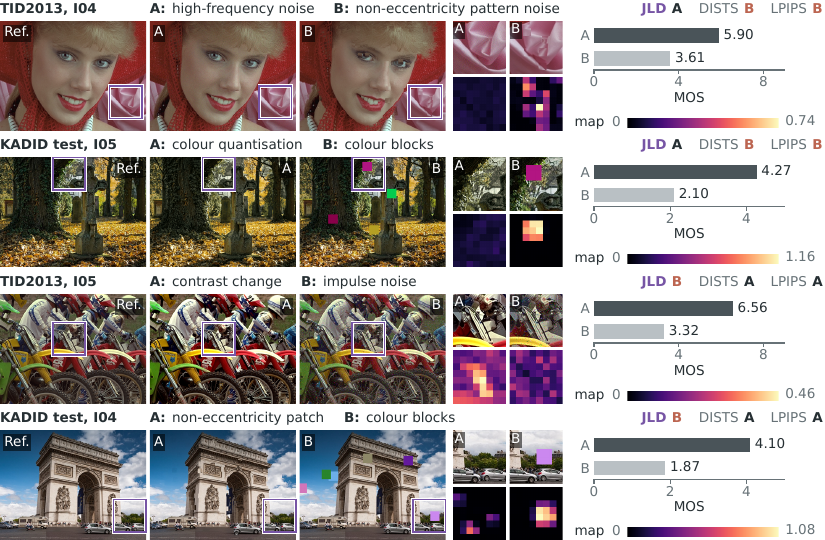}
\caption{Four same-reference pairs in which humans prefer A. \name{} agrees in the top two rows and reverses the preference in the bottom two. Maps share one scale within each row and show the lens term only.}
\label{fig:gallery}
\end{figure}

\namevideo{} applies the same construction to a frozen VideoMAE-Base fine-tuned on Kinetics-400 \citep{tong2022videomae}.
We use block~2 with a 16-dimensional lens, fitted to the class logits of 50 unlabelled clips using eight random probes.
For each video pair, we sample eight aligned 16-frame clips, project the space-time token differences through the lens, average their squared lengths over tokens and clips, and take the square root.
Video uses neither the chroma pre-filter nor an output term.
On Waterloo IVC 4K \citep{li2019waterloo4k}, used to select the block and rank, \namevideo{} reaches \srcc{} 0.786.
For comparison, PCA~64 reaches 0.771, all 768 features 0.710, VMAF 0.611, and PSNR 0.562.
The frame-wise image \name{} reaches 0.710, showing that the space-time representation captures information unavailable from individual frames.

\subsection{Held-out video}\label{app:video-heldout}

We evaluate the frozen Waterloo-selected configuration on all 120 released AVT-VQDB-UHD-1 test pairs \citep{rao2019avt}; no AVT video is used for fitting or model selection.
Decoded videos are scaled to a common display resolution of $1920\times1080$, and all methods are evaluated under the same display protocol.
% v23 (W5): both video datasets side by side. SRCC, within-source SRCC and the lower of the two dataset SRCCs from
% figures/v23/video_heldout_numbers.json (scripts/v23/video_heldout.py; per-pair inputs results/video2_pairs.csv and
% results/video3_pairs_avt_vqdb_uhd1.csv, display protocol). AVT KRCC and PLCC from results/video3_summary.csv,
% protocol heldout-b2-k16-display2160-then-1080-centre256-lens224-v1. Encoded-resolution row: results/video3_summary.csv,
% protocol heldout-b2-k16-eight16-centre256-lens224-fp32-v1.
\begin{table}[t]
\centering\footnotesize% v23.5 P2: was \small
\setlength{\tabcolsep}{3.4pt}
\caption{Video results on both datasets. Within: mean per-source \srcc{}. Lower: the smaller of the two dataset \srcc{}s. Bold marks the best value in each column.}
\label{tab:app_video_heldout}
\definecolor{ourMetrics}{HTML}{F3EFF8}
\begin{tabular}{@{}lccccccccc@{}}
\toprule
 & Fitted to & \multicolumn{2}{c}{Waterloo 4K (dev., 240)} & \multicolumn{4}{c}{AVT-UHD (held out, 120)} & & AVT \\
\cmidrule(lr){3-4}\cmidrule(lr){5-8}
Metric & video MOS & \srcc{} & Within & \srcc{} & \krcc{} & \plcc{} & Within & Lower & encoded \\
\midrule
\rowcolor{ourMetrics}\namevideo{} & no & \textbf{0.786} & 0.967 & 0.912 & 0.761 & 0.918 & 0.944 & \textbf{0.786} & 0.582 \\
VMAF & yes & 0.611 & 0.834 & \textbf{0.932} & \textbf{0.790} & \textbf{0.939} & 0.953 & 0.611 & 0.663 \\
MS-SSIM & no & 0.694 & 0.967 & 0.884 & 0.718 & 0.888 & 0.945 & 0.694 & 0.561 \\
PSNR & no & 0.562 & \textbf{0.968} & 0.908 & 0.758 & 0.915 & \textbf{0.964} & 0.562 & \textbf{0.754} \\
\bottomrule
\end{tabular}
\end{table}

On AVT, \namevideo{} reaches \srcc{} 0.912, compared with 0.932 for VMAF, 0.908 for PSNR, and 0.884 for MS-SSIM.
On Waterloo, \namevideo{} reaches 0.786 compared with 0.611 for VMAF.
Thus, \namevideo{} is the only one of these four methods above 0.78 on both datasets, despite using no subjective video scores during fitting.
Figure~\ref{fig:video-artifacts} shows representative cases where \namevideo{} and VMAF agree or disagree with human scores.

% \subsection{Visible artifacts}\label{app:video-examples}

% arXiv layout: fig:gallery is input after the section heading so it shares the top of this page with the video table.
% \begin{figure}[t]
% \centering
% \includegraphics[width=\linewidth]{v22/app_gallery.pdf}
% \caption{Four same-reference pairs in which humans prefer A. \name{} agrees in the top two rows and reverses the preference in the bottom two. Maps share one scale within each row and show the lens term only.}
% \label{fig:gallery}
% \end{figure}
% arXiv layout: moved here from sections/app_limitations.tex so that it sits next to Appendix L.1.
\begin{figure}[t]
\centering
\includegraphics[width=\linewidth]{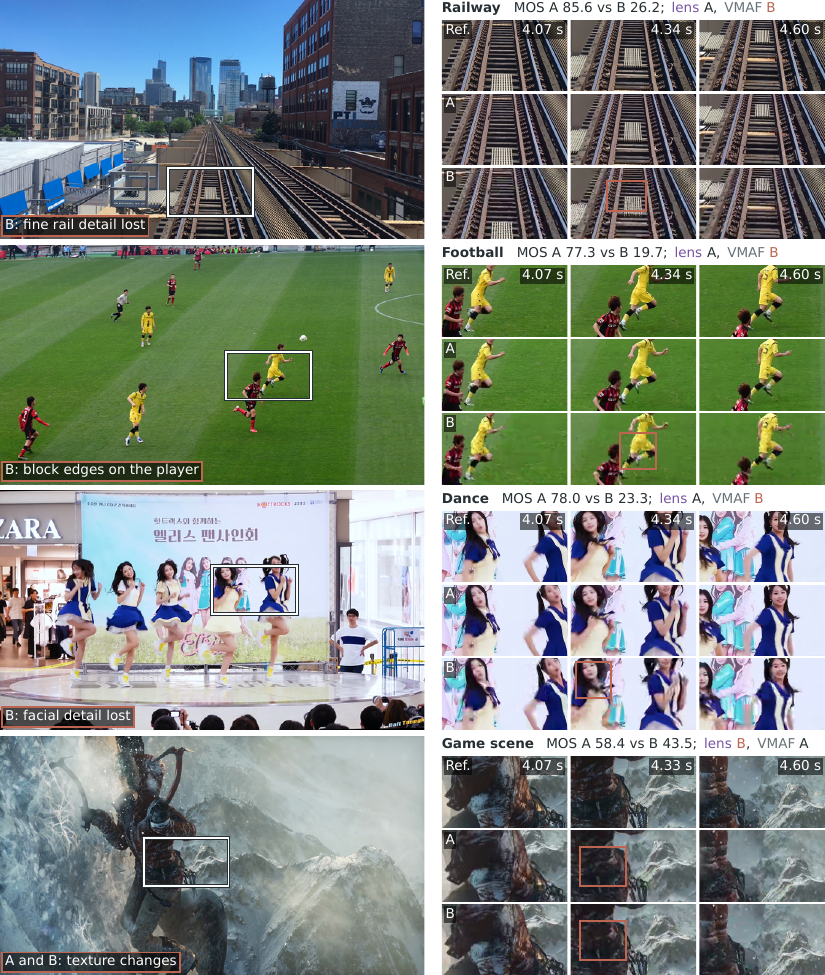}
\caption{Four Waterloo examples. Each row shows the reference and two distorted videos at synchronized frames; row titles report MOS and the \namevideo{} and VMAF choices. Coral marks disagreement with MOS and visible artifacts.}
\label{fig:video-artifacts}\label{fig:video-artifacts-more}
\end{figure}

\section{Matched pixel error}\label{app:gallery}

Figure~\ref{fig:gallery} shows representative same-reference pairs with similar pixel error where human judgements distinguish the distortions.
\name{} agrees with the human preference in the top two examples and fails in the bottom two, providing both representative successes and failure cases.

\section{Limitations}\label{app:limitations}\label{sec:limitations}

\name{} inherits limitations of its encoder. It is weaker on restoration outputs (0.624 \srcc{} on PIPAL) and on global contrast changes (0.378 on TID2013; Appendix~\ref{app:types}).
The lens term is a pseudometric with a non-trivial null space: displacements in the 320 discarded feature directions are invisible to it, so direct optimisation against the lens can exploit these directions \citep{ding2021optimization}.
Although \name{} is differentiable, using it as a training loss therefore requires care.
The output term removes some of this ambiguity but sacrifices the triangle inequality, with violations on at most 107 per million audited triplets; applications requiring a strict pseudometric can use the lens term alone.
The lens itself uses no human labels, but CSIQ and KADID-10k development data select the block, rank, and output weight; LIVE, KADID-10k test, and PIPAL are held out.
On held-out video, \namevideo{} reaches 0.912 versus 0.932 for VMAF (Appendix~\ref{app:video}).
Full \name{} is also slower than DISTS; the approximately four-fold speed advantage applies to \namefast{}.
Finally, the human study contains only ten participants, and \name{} assumes spatially aligned inputs.

\end{document}